\documentclass{article}

\usepackage{PRIMEarxiv}
\usepackage{authblk}
\usepackage[utf8]{inputenc} 
\usepackage[T1]{fontenc}    
\usepackage{hyperref}       
\usepackage{url}            
\usepackage{booktabs}       
\usepackage{amsfonts}       
\usepackage{nicefrac}       
\usepackage{microtype}      
\usepackage{fancyhdr}       
\usepackage{graphicx}       
\usepackage{xcolor}         
\usepackage{array}
\usepackage{tabularx}
\usepackage{subcaption}
\usepackage{caption}
\usepackage{algorithm}
\usepackage{algorithmic}
\usepackage[capitalize,noabbrev]{cleveref}
\usepackage{natbib}
\theoremstyle{definition}
\newtheorem{assumption}{Assumption}
\theoremstyle{plain}
\makeatletter
\newcommand\appfmt@section{Appendix\nobreakspace\thesection:\enspace}
\let\orig@appendix\appendix
\renewcommand\appendix{\orig@appendix
  \renewcommand\@seccntformat[1]{%
    \ifcsname appfmt@##1\endcsname\csname appfmt@##1\endcsname
    \else\csname the##1\endcsname\quad\fi}}
\makeatother

\title{Hopformer: Homogeneity-Pursuit Transformer for Time Series Forecasting} 

\author[1]{Wan Zhang}
\author[3]{Chan Lee}
\author[2]{Qinjie Lin}
\author[2]{Weijian Li}
\author[2, 3]{Han Liu}
\author[4]{Kai Zhang}

\affil[1]{The AMSS Center of Forecasting Science, Chinese Academy of Sciences, Beijing, China}
\affil[2]{Department of Computer Science, Northwestern University, Evanston, IL}
\affil[3]{Department of Statistics and Data Science, Northwestern University, Evanston, IL}
\affil[4]{Department of Statistics and Operations Research, University of North Carolina, Chapel Hill, NC}

\date{} 

\begin{document}
\maketitle
\thispagestyle{fancy}       

\begin{abstract}
We introduce HopFormer (\textbf{Ho}mogeneity-\textbf{P}ursuit Trans\textbf{former}) to forecast multiple time series that share covariate effects but differ in their temporal dynamics. The first stage combines cross-sectional regression experts using Sparsity Pattern Aggregation (SPA) to estimate a covariate-driven component. The second stage adapts a pretrained Transformer through Low-Rank Adaptation (LoRA) to forecast the residuals. For a chronological variant, we establish an end-to-end risk bound under temporal dependence, separating candidate approximation, aggregation error, residual-model input sensitivity, learning and optimization error, and dependence costs. Averaging relative improvements equally across six datasets and ten forecasting baseline models, including those with reported state-of-the-art performance, HopFormer reduces MASE by 7.64\% and MAPE by 4.79\%. These findings demonstrate the potential of our proposed covariate-driven residual modeling as a modular approach to improving pretrained forecasters, particularly in covariate-informative regimes.
\end{abstract}

\section{Introduction}
Forecasting multiple time series requires combining temporal history with external information, such as prices, promotions, or demand forecasts. Transformer architectures provide flexible models of temporal structure~\citep{zhou2021informer,li2019enhancing,zhou2022fedformer,zhang2023crossformer,wu2021autoformer}, while pretrained forecasting models seek to reuse temporal representations across datasets~\citep{gerald2024moirai,ansari2024chronoslearninglanguagetime,liu2024timer,das2024decoder}. These developments motivate a complementary question: how can a forecaster incorporate covariates without requiring its temporal backbone to model the full covariate space directly? The challenge is to exploit shared covariate relationships while preserving temporal information that remains specific to each series.

We propose HopFormer (\textbf{Ho}mogeneity-\textbf{P}ursuit Trans\textbf{former}), a modular framework that decomposes prediction into a covariate-driven component and a temporal residual. Stage I fits cross-sectional regression experts and combines predictions from sparse expert subsets through \emph{Sparsity Pattern Aggregation} (SPA)~\citep{rigollet2011spa}. The resulting covariate-to-target mapping is shared across the panel, but its predictions vary with each series' covariates. Stage II forecasts the remaining residuals using a pretrained Transformer adapted through \emph{Low-Rank Adaptation} (LoRA)~\citep{hu2022lora}. At inference, the covariate-driven prediction and residual forecast are added. This decomposition separates covariate modeling from the backbone's temporal inputs; it does not require identical residual distributions or guarantee reduced residual variance.

We analyze how errors in the covariate component propagate to residual forecasting under temporal dependence. For a precisely specified chronological estimator, our end-to-end risk bound separates candidate approximation, aggregation error, residual-model input sensitivity, learning and optimization error, and dependence costs. The contribution is the joint analysis of the covariate-to-residual forecasting pipeline, rather than a new sparse aggregation or adaptation method. Our experiments use in-sample residualization, which is distinct from the chronological estimator covered by the theory.

Our benchmark uses common rolling-origin tasks, permitted information boundaries, and scoring populations, with configurations selected by validation MASE. It contains 66 completed evaluations of HopFormer and ten task-trained baselines across six datasets, together with 42 matched-control and aggregation-ablation evaluations. Relative to these ten baselines, HopFormer attains the lowest mean MASE on three datasets and the lowest mean MAPE on one. Its clearest matched-backbone advantage is on Synthetic Electricity with oracle future covariates: residualization reduces MASE by 71.2\%, with favorable adjusted evidence across all declared block lengths. EPF improves numerically against direct adaptation, but without stable adjusted evidence. Direct adaptation performs better on Sales1, Sales2, Rossmann-100, and Favorita-198. Adapting the residual backbone lowers mean MASE on five of six datasets, with Rossmann as the exception. These findings support task-dependent benefits rather than universal superiority. Fourteen additional completed external-baseline comparisons are retained in the appendix.

\section{Related Work}
\paragraph{Universal Time Series Forecasting}
Traditional forecasting methods often adopt a one-model-per-dataset paradigm, limiting scalability across heterogeneous time series~\citep{wu2023timesnettemporal2dvariationmodeling, nie2023multiv, liu2024itransformerinvertedtransformerseffective}. To address this, recent advances leverage generative pretraining and prompt-based Transformers~\citep{cao2024tempopromptbasedgenerativepretrained, xue2023promptcastnewpromptbasedlearning, ekambaram2024tinytimemixersttms, jin2024timellmtimeseriesforecasting}, but often require custom modules for each task. In contrast, universal forecasting frameworks aim for broad generalization across tasks. SimpleTS~\citep{yao2023simplets} selects optimal models via time series type classification, while MOIRAI~\citep{gerald2024moirai} and UniTS~\citep{gao2024units} integrate predictive and generative capabilities via enhanced Transformer encoders. FlexTSF~\citep{xiao2024flextsf} further extends universality to irregular time series. These methods underscore a growing interest in general-purpose time series forecasting across diverse real-world domains.

\paragraph{Expert Aggregation and Sparse Ensembles}
Combining diverse regressors is a time-tested strategy to improve robustness in forecasting. Classical approaches include stacking~\citep{Godahewa_2023} and mixture-of-experts (MoE) models~\citep{jordan1991moe}, which dynamically assign weights to base learners. N-BEATS-MOE~\citep{matos2025nbeatsmoe} extends this to temporal settings, improving adaptation to heterogeneous dynamics. On the theoretical side, sparsity pattern aggregation (SPA)~\citep{rigollet2011spa} offers a principled way to combine experts via exponential weighting, balancing empirical risk and model complexity. Follow-up work explores PAC-Bayesian guarantees~\citep{dalalyan2008aggregation} and affine estimator ensembles~\citep{dai2014aggregation}. We use SPA as the trend extractor of a two-stage pipeline and analyze how its error propagates to residual forecasting.

\paragraph{Parameter-Efficient Fine-Tuning for Time Series}
PEFT methods have gained popularity for adapting large pretrained models with minimal computational overhead. Techniques such as Low-Rank Adaptation (LoRA)~\citep{hu2022lora}, adapters~\citep{houlsby2019parameter}, and prompt tuning~\citep{lester2021power} enable efficient fine-tuning by updating a small subset of parameters. These approaches have been successfully applied in time series domains~\citep{gupta2024beyond, gupta2024lowrankadaptationtimeseries, rasul2024lagllamafoundationmodelsprobabilistic, ansari2024chronoslearninglanguagetime}, demonstrating generalization across modalities and domains. In this work, we adopt LoRA as a representative PEFT technique due to its simplicity, efficiency and compatibility. 

\section{Method}
Many universal time series forecasting methods have limited support for high-dimensional covariates. Hopformer separates covariate-driven prediction from residual forecasting, allowing a pretrained temporal backbone to use external covariate information without modeling the full covariate space.

\subsection{Setup and Hopformer Overview}\label{sec:method_overview}

\paragraph{Problem Formulation}
We consider a panel of $N$ scalar time series $\mathcal D=(\mathcal D_1,\ldots,\mathcal D_N)$, where $\mathcal D_i=((Y_{i,t},Z_{i,t}))_{t\in\mathcal T_i}$, $Y_{i,t}\in\mathbb R$ is the target, and $Z_{i,t}\in\mathbb R^d$ contains encoded covariates. A multivariate target is represented by multiple scalar series. The covariates use a common encoded schema, but their values are series-specific. Harmonized slots may contain different market-specific quantities, as documented for EPF in Appendix~\ref{app:exp-data}. Static series identity may be included; categorical encodings are fitted using fitting data only. For the fully observed case, $\mathcal T_i$ indexes the available observations. Given context length $L$ and horizon $H$, a forecast starting at $t$ is issued after observing time $t-1$, with $\{t-L,\ldots,t+H-1\}\subseteq\mathcal T_i$. We learn a forecaster $\mathcal F:\mathbb R^L\times\mathbb R^{(L+H)\times d}\to\mathbb R^H$ that produces $\widehat Y_{i,t:t+H-1}=\mathcal F(Y_{i,t-L:t-1},Z_{i,t-L:t+H-1})$ and minimizes a specified loss $\mathcal L(\widehat Y_{i,t:t+H-1},Y_{i,t:t+H-1})$. Future covariates supplied to the forecaster must be available at issuance; synthetic oracle-exogenous inputs are explicitly distinguished from this operational setting. For incomplete panels, the experiments retain complete calendars, causally fill missing historical inputs, and score only observed targets using model-independent masks, as detailed in Appendix~\ref{app:exp-cleaning}.

\paragraph{Hopformer: Two-Stage Decomposition}
To address this challenge, \textit{Hopformer} decomposes forecasting into sequential trend homogenization and residual modeling stages.

Stage~I fits regression experts $\{g_j\}_{j=1}^{M}$ and constructs a clipped predictor $\widehat m_{\mathbf p}$ for each sparsity pattern $\mathbf p\in\mathcal P:=\{0,1\}^{M}$. SPA combines these predictors as \[ \widehat m(z)=\sum_{\mathbf p\in\mathcal P}\widehat\rho(\mathbf p)\widehat m_{\mathbf p}(z), \qquad \widehat\rho(\mathbf p)\geq0,\qquad \sum_{\mathbf p\in\mathcal P}\widehat\rho(\mathbf p)=1. \] Section~\ref{sent:lse} specifies the within-pattern coefficients, clipping, and pattern probabilities. The aggregate can be expressed as a single linear combination of expert outputs only when clipping is inactive. After fitting Stage~I, we form the residual series $\widehat R_{i,t}=Y_{i,t}-\widehat m(Z_{i,t})$.
Stage~II keeps $\widehat m$ fixed and adapts a pretrained Transformer $f_w:\RR^L\to\RR^H$ parameterized by $w$ to forecast residuals. Its input is the residual history $\widehat R_{i,t-L:t-1}$ and its training target is the future residual window $\widehat R_{i,t:t+H-1}$. Section~\ref{sec:transformer} specifies the training objective and LoRA parameterization.

At inference, both fitted stages remain fixed. Applying $\widehat m$ to each known future covariate vector gives the covariate-driven forecast; for any window, write $\widehat m(Z_{i,a:b}):=(\widehat m(Z_{i,a}),\ldots,\widehat m(Z_{i,b}))$, and let $\widehat w$ denote the parameter vector learned in Stage~II. The composite forecast is
\begin{equation}\label{eq:hopformer_forecast}
\widehat Y_{i,t:t+H-1}=\widehat m(Z_{i,t:t+H-1})+f_{\widehat w}(\widehat R_{i,t-L:t-1}).
\end{equation}
This common covariate-to-residual transformation motivates \emph{homogeneity pursuit}: Stage~I applies the same learned rule to extract covariate-driven structure across the panel before Stage~II models the remaining temporal dynamics. It does not require identical predictions or residual distributions across series. Figure~\ref{fig:hop_workflow} summarizes the training and inference workflow.

\begin{figure}[ht]
    \centering
    \includegraphics[width=1\linewidth]{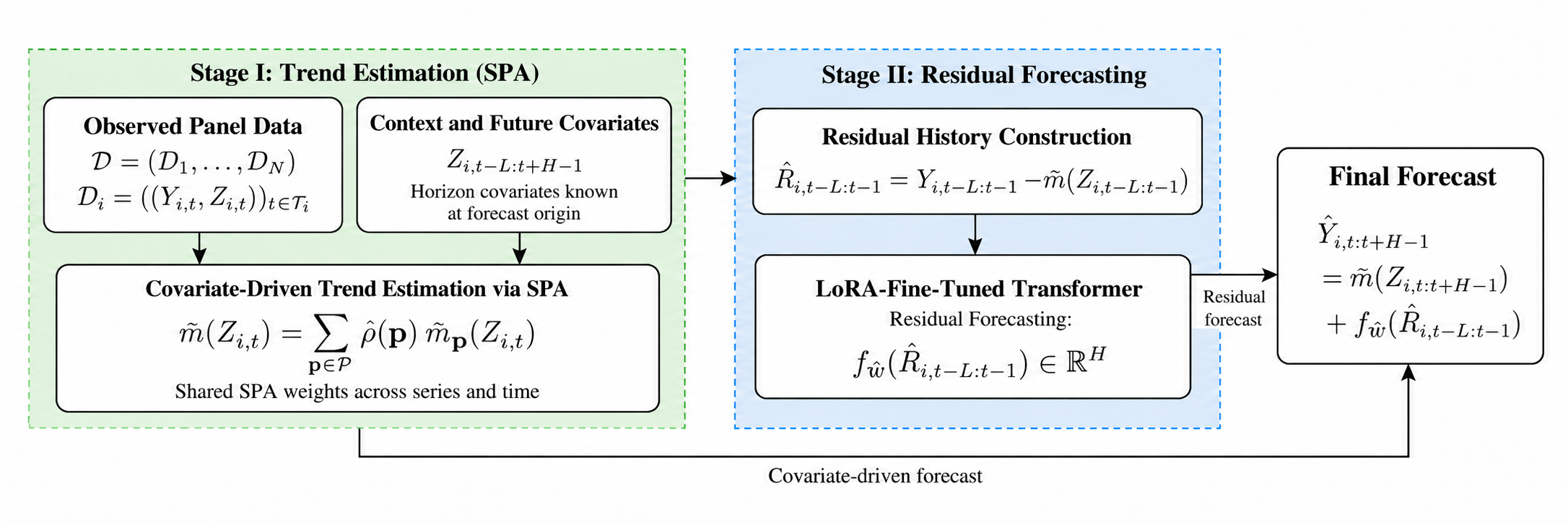}
\caption{\textbf{Hopformer workflow.} Stage~I fits trend predictor $\widehat m$ by aggregating regression experts via shared SPA weights $\widehat\rho(\mathbf{p})$. With $\widehat m$ frozen, Stage~II adapts a Transformer $f_{\widehat w}$ on extracted residuals. At inference, the residual forecast $f_{\widehat w}(\widehat R_{i,t-L:t-1})$ is added to the Stage-I prediction from known horizon covariates $\widehat m(Z_{i,t:t+H-1})$. The diagram suppresses the fixed target normalization and final inverse transformation; the implemented reconstruction is specified in Appendix~\ref{sec:pseudocode}.}
    \label{fig:hop_workflow}
\end{figure}

\subsection{Stage I: Covariate-Driven Trend Estimation via SPA}\label{sent:lse}

Let $g_j:\RR^d\to\RR$, $j=1,\ldots,M$, be $M$ regression experts shared across series and time. Fit them on observations $\mathcal E\subseteq\{(i,t):i\in[N],\ t\in\mathcal T_i\}$. The covariates may include external variables, time features, and lagged targets (lagged by at least $H$ to remain observable at the forecast origin). 

Rather than choosing a single expert or one unrestricted combination, SPA considers all subsets of the expert pool. A sparsity pattern $\bp=(p_1,\ldots,p_M)\in\cP:=\{0,1\}^M$ specifies which experts are included: $p_j=1$ selects $g_j$. Write $S_{\bp}:=\{j:p_j=1\}$ for the selected set and $|\bp|:=|S_{\bp}|$ for its size. Conditional on the fitted experts, estimate the within-pattern coefficients by constrained least squares:
\[
\widehat\omega_{\bp}
\in
\arg\min_{\omega\in\RR^M}
\frac{1}{|\mathcal E|}
\sum_{(i,t)\in\mathcal E}
\left[
Y_{i,t}
-
\sum_{j=1}^M \omega_j g_j(Z_{i,t})
\right]^2
\quad
\text{subject to }
\omega_j=0
\text{ for every }j\notin S_{\bp}.
\]
Thus, \(\widehat\omega_{\bp,j}\) is the fitted coefficient of expert \(j\) when \(j\in S_{\bp}\). The estimator associated with pattern \(\bp\) is
\[
\widehat m_{\bp}(z)=\operatorname{clip}_C\!\left(\sum_{j=1}^M\widehat\omega_{\bp,j}g_j(z)\right),\qquad
\operatorname{clip}_C(u):=\max\{-C,\min\{u,C\}\},
\]
where $C>0$ is a clipping threshold. Each pattern therefore yields one fitted covariate predictor. Its coefficients $\widehat\omega_{\bp,j}$ are zero for excluded experts but otherwise need not be nonnegative or sum to one.

To compare candidates, let $\mathcal{A} \subseteq \{(i,t) : i \in [N], t \in \mathcal{T}_i\}$ denote calibration observations with mean loss
\[
\widehat L_A(\bp)
=
\frac{1}{|\mathcal A|}
\sum_{(i,t)\in\mathcal A}
\left[
Y_{i,t}
-
\widehat m_{\bp}(Z_{i,t})
\right]^2.
\]
The reported implementation uses disjoint chronological regions for expert fitting, $\mathcal E$, aggregation calibration, $\mathcal A$, and aggregation-hyperparameter selection. Stage~II subsequently reuses the fitting interval containing these regions, so residualization is in-sample. The chronological estimator analyzed in Section~\ref{sec:thm} instead uses the separated, buffered regions specified there. To favor sparsity, we assign each pattern the prior
\begin{equation}\label{eq:spa_prior}
\pi_\alpha(\bp)=\frac{e^{-\alpha|\bp|}}{Z_\alpha\binom{M}{|\bp|}},\qquad
Z_\alpha=\sum_{s=0}^M e^{-\alpha s},
\end{equation}
where $\alpha\ge0$ controls the preference for smaller patterns and $Z_\alpha$ normalizes the prior. For the full pattern family $\cP=\{0,1\}^M$, normalization follows from
\[
\sum_{\bp\in\cP}\pi_\alpha(\bp)=\sum_{s=0}^M\binom Ms\frac{e^{-\alpha s}}{Z_\alpha\binom Ms}=1.
\]
Since $\binom Ms\le(eM/s)^s$ for $s\ge1$, the log-prior penalty satisfies
\[
\log\{1/\pi_\alpha(\bp)\}\le |\bp|\log(eM/|\bp|)+\alpha|\bp|+\log Z_\alpha,\qquad |\bp|\ge1.
\]
For the empty pattern, the penalty is $\log Z_\alpha$. The normalizing term remains in the bound even when $\alpha=0$; a restricted pattern family would require normalization over that family.
For inverse temperature $\lambda>0$, we combine the prior and calibration loss through the Gibbs weights, then average the candidate predictions as follows:
\begin{equation}\label{eq:spa}
\widehat\rho(\bp)=\frac{\pi_\alpha(\bp)e^{-\lambda\widehat L_A(\bp)}}{\sum_{\bq\in\cP}\pi_\alpha(\bq)e^{-\lambda\widehat L_A(\bq)}},\qquad
\widehat m(z)=\sum_{\bp\in\cP}\widehat\rho(\bp)\widehat m_{\bp}(z).
\end{equation}
The weights $\widehat\rho(\bp)$ are nonnegative and sum to one. Thus, the within-pattern coefficients $\widehat\omega_{\bp}$ combine selected experts, whereas $\widehat\rho$ determines each candidate's contribution to the final aggregate. Lower calibration loss and greater prior mass increase a candidate's weight.

When clipping is inactive, interchanging the finite sums in~\eqref{eq:spa} gives
\[
\widehat m(z)=\sum_{j=1}^M\omega_j^{\mathrm{SPA}}g_j(z),\qquad
\omega^{\mathrm{SPA}}=\sum_{\bp\in\cP}\widehat\rho(\bp)\widehat\omega_{\bp}.
\]
The effective expert weights vector $\omega^{\mathrm{SPA}}$ is shared across series and time, but is not generally a probability vector: there is no restriction that its entries need to be nonnegative or sum to one. With active clipping, the aggregate must instead be formed from the clipped candidates as in~\eqref{eq:spa}.

We select $(\lambda,\alpha)$ from a finite grid $\Lambda$ specified before calibration. With $M=8$, all $2^M=256$ patterns are enumerated. To relate the mean-loss formulation to a Gaussian-style weighting convention, let $n_A:=|\mathcal A|$ and let $\sigma^2>0$ be a scale parameter. Setting $\lambda=n_A/(4\sigma^2)$ gives
\[
\exp\{-\lambda\widehat L_A(\bp)\}
=\exp\!\left\{-\frac{1}{4\sigma^2}\sum_{(i,t)\in\mathcal A}[Y_{i,t}-\widehat m_{\bp}(Z_{i,t})]^2\right\}.
\]
This is a temperature normalization, not a claim that the clipped, temporally dependent procedure satisfies the assumptions of a Gaussian aggregation theorem.
For block scores, let $m_1$ denote the number of retained calibration blocks and let $V_b(\bp)$ be the mean loss of pattern $\bp$ within block $b$, as specified in Appendix~\ref{app:design}. If the block average $m_1^{-1}\sum_{b=1}^{m_1}V_b(\bp)$ is replaced by the block sum $\sum_{b=1}^{m_1}V_b(\bp)$, use inverse temperature $\lambda/m_1$ to preserve the same Gibbs weights.

\subsection{Stage II: Residual Forecasting with a Pretrained Transformer}\label{sec:transformer}

With Stage~I fixed, Stage~II learns temporal dependencies from residual histories. Let $\mathcal{T}_i^{\mathrm{train}} \subseteq \mathcal{T}_i$ be the time region assigned to residual training for series $i$. Its admissible forecast origins are
\[
\mathcal J_i^{\mathrm{train}}:=\{t\in\mathcal T_i^{\mathrm{train}}:\{t-L,\ldots,t+H-1\}\subseteq\mathcal T_i^{\mathrm{train}}\}.
\]
This restriction keeps both the context and target within the training region. Let $n_{\mathrm{res}}:=\sum_{i=1}^N|\mathcal J_i^{\mathrm{train}}|$ be the number of training examples, and let $\mathcal L_{\mathrm{res}}:\RR^H\times\RR^H\to\RR_+$ be the residual forecasting loss. Write $\Omega_{\mathrm{LoRA}}$ for the admissible parameter class obtained by adapting the backbone with LoRA while keeping its pretrained weights fixed. The training algorithm returns $\widehat w$ by approximately solving
\begin{equation}\label{eq:stage2_objective}
\min_{w\in\Omega_{\mathrm{LoRA}}}\frac{1}{n_{\mathrm{res}}}\sum_{i=1}^N\sum_{t\in\mathcal J_i^{\mathrm{train}}}\mathcal L_{\mathrm{res}}\!\left(f_w(\widehat R_{i,t-L:t-1}),\widehat R_{i,t:t+H-1}\right).
\end{equation}
The general loss is application-specific. Section~\ref{sec:thm} analyzes its squared-loss specialization on a balanced panel with temporally separated training blocks.
For a selected Transformer matrix $W_0^{(\ell)}\in\RR^{d_{\mathrm{in},\ell}\times d_{\mathrm{out},\ell}}$, where $\ell$ indexes the adapted matrix, LoRA uses the update~\citep{hu2022lora}
\begin{equation}\label{eq:lora_update}
\widetilde W^{(\ell)}=W_0^{(\ell)}+s_\ell B_{\mathrm{LoRA}}^{(\ell)}A_{\mathrm{LoRA}}^{(\ell)}.
\end{equation}
Here $B_{\mathrm{LoRA}}^{(\ell)}\in\RR^{d_{\mathrm{in},\ell}\times r_\ell}$ and $A_{\mathrm{LoRA}}^{(\ell)}\in\RR^{r_\ell\times d_{\mathrm{out},\ell}}$ are trainable factors, $r_\ell\ll\min\{d_{\mathrm{in},\ell},d_{\mathrm{out},\ell}\}$ is the adaptation rank, and $s_\ell$ is a scaling coefficient. The original matrix remains fixed; $w$ collects the trainable factors across the adapted matrices. We adapt the attention query, key, value, and output projections. LoRA reduces the trainable parameter count, but other adaptation mechanisms or forecasting backbones can replace it without changing the two-stage prediction rule~\eqref{eq:hopformer_forecast}.

\section{Theoretical Analysis}\label{sec:thm}

We analyze a chronological version of Hopformer in which each stage is fitted on its own time region and frozen before the next stage is trained. The analysis connects the Stage-I trend error to the final forecasting risk. The ingredients are established: Gibbs aggregation of sparse patterns \citep{rigollet2011spa,alquier2024user}, independent-block coupling for $\beta$-mixing processes \citep{Yu1994mixing,mcdonald2017nonparametric}, and information-theoretic generalization bounds \citep{NIPS2017_info}. The primary theoretical challenge is controlling the error propagation across the covariate-to-residual interface under temporal dependence. Our contribution is establishing an end-to-end risk guarantee that bounds the impact of Stage-I trend approximation on the downstream forecasting performance.

\subsection{Setting}\label{sec:setting}

\paragraph{Raw panel and chronological training.}
We specialize Section~\ref{sec:method_overview} to a balanced synchronous panel, where all $N$ series are observed across $t \in \ZZ$. Let $Y_{i,t}^0$ denote the raw target before scaling, and $O_t^0 := \{(Y_{i,t}^0, Z_{i,t})\}_{i=1}^N$ the raw panel snapshot. The training span is partitioned into regions $\cT_E < \cT_A < \cT_R$ separated by temporal buffers: scaling, experts, and pattern coefficients are fitted on $\cT_E$; SPA weights are calibrated on $\cT_A$; and residual learning occurs on $\cT_R$ with Stage~I fixed.

\paragraph{Frozen scaling and population target.}
Let $S$ be the scaling state fitted on $\cT_E$, defining transformed targets $Y_{i,t} = T_{S,i}(Y_{i,t}^0)$ with fixed maps $T_{S,i}$. Stationarity and $\beta$-mixing are assumed for the raw process $O^0$ rather than the data-scaled trajectory (Assumption~1). Population expectations $\EE_\star$ take an independent stationary raw draw with frozen scaling, defining the Stage-I oracle as
\[
m_{\star,S,i}(z):=\EE_\star[T_{S,i}(Y_{i,t}^{0,\star})\mid Z_{i,t}^{\star}=z].
\]
We suppress $S$ in $m_\star$, residuals, and risks below. The oracle conditional mean may differ across series even though Hopformer's fitted predictor is shared. Fitted covariate preprocessing is included in the predictor maps rather than changing the raw conditioning variables.

\paragraph{Panel notation and risks.} Write $Y_t = (Y_{i,t})_{i=1}^N \in \RR^N$, with $i$-th row $Z_{i,t}^\top$ in $Z_t \in \RR^{N\times d}$. For $X_t \in \RR^N$, let $X_{a:b} := [X_a,\dots,X_b]$. For $V \in \RR^{N\times q}$, define $\|V\|_N^2 := N^{-1} \sum_{i=1}^N \sum_{r=1}^q V_{i,r}^2$. A shared predictor $h$ acts as $h(Z_t) = (h(Z_{i,t}))_{i=1}^N$ and $h(Z_{a:b}) = [h(Z_a),\dots,h(Z_b)]$; similarly, $f_w$ acts row-wise on panel histories, with oracle prediction $m_\star(Z_t) = (m_{\star,i}(Z_{i,t}))_{i=1}^N$.

Define the oracle and estimated residuals by $R_t:=Y_t-m_\star(Z_t)$ and $\widehat R_t:=Y_t-\widehat m(Z_t)$. For a fixed shared covariate predictor $h$, the Stage-I trend error is
\[
\delta_1(h):=\EE_\star\|h(Z_t)-m_\star(Z_t)\|_N^2.
\]
By the conditional-mean projection identity, this is also the excess squared-loss risk of $h$ relative to $m_\star$ (Lemma~\ref{lem:proj}). For a residual-model parameter $w$, define
\begin{align*}
R_E(\widehat m,w)&:=\EE_\star\|Y_{t:t+H-1}-\widehat m(Z_{t:t+H-1})-f_w(\widehat R_{t-L:t-1})\|_N^2,\\
R_R(w)&:=\EE_\star\|R_{t:t+H-1}-f_w(R_{t-L:t-1})\|_N^2.
\end{align*}
The first is the end-to-end forecasting risk; the second evaluates the same residual model on oracle residuals. The fitted Hopformer risk is $R_H:=R_E(\widehat m,\widehat w)$.

\paragraph{Blocks and regularity conditions.} Calibration uses $m_1$ retained blocks $A_b \subset \cT_A$ of $a_1$ consecutive snapshots, while residual training uses $m_2$ blocks $\mathcal B_b \subset \cT_R$ of $a_2$ forecast origins (support length $a_2+L+H-1$). Gaps index the difference between adjacent block supports: $d_{EA}$ separates $\cT_E$ from $A_1$, $s_1$ lower-bounds inter-calibration gaps, $d_{AR}$ separates $\cT_A$ from $\mathcal B_1$, and $s_2$ lower-bounds Stage-II support gaps. Appendix~\ref{app:design} gives the complete prespecified layout and block objectives.

Let $F$ be the upstream state frozen before Stage~II, including $S$, experts, candidates, SPA weights, preprocessing, hyperparameters, and backbone. Write $\Omega(F)$ for the admissible residual-model class given this state; $\Omega_{\mathrm{LoRA}}$ is one instance. Assumption~\ref{ass:reg} in Appendix~\ref{app:assumptions} requires common bounds $C$ on scaled targets and candidates and $D$ on residual forecasts, a $K$-Lipschitz residual map under $\|\cdot\|_N$, chronological fitting, and expected squared-loss optimization error at most $\varepsilon_o$ under the independent-block law defined in Section~\ref{sec:e2e}. These bounds hold uniformly over admissible fitted states, including fresh observations. Consequently, Stage-I and Stage-II losses are bounded by $B_1:=4C^2$ and $B_2:=H(2C+D)^2$, respectively. Cross-series dependence within a snapshot is unrestricted, so the sample sizes in the bounds count temporal blocks, not individual panel entries.

\subsection{Stage I: Sparse Aggregation under Dependence}\label{sec:stage1}

\begin{proposition}[Stage-I oracle inequality]\label{prop:stage1}
Under Assumptions~\ref{ass:mix} and~\ref{ass:reg}(a,c), fix a prior $\pi$ on $\cP$ and an inverse temperature $\lambda>0$ before calibration. For every $\eta\in(0,1)$, with probability at least $1-\eta-\tau_1$,
\begin{equation}\label{eq:stage1}
\delta_1(\widehat m)\le b_1:=\min_{\bp:\,\pi(\bp)>0}\left\{\delta_1(\widehat m_{\bp})+\frac{2[\log(1/\pi(\bp))+\log(2/\eta)]}{\lambda}\right\}+\frac{\lambda B_1^2}{4m_1},
\end{equation}
where $\tau_1=\min\{1,\beta(d_{EA})+(m_1-1)\beta(s_1)\}$.
\end{proposition}

This compares SPA with the best \emph{frozen} candidate, whose error already includes expert and coefficient fitting on $\cT_E$. The remaining terms account for prior complexity, concentration over calibration blocks, and temporal dependence. For selection from the finite grid $\Lambda$, replace $\log(2/\eta)$ by $\log(2|\Lambda|/\eta)$ to obtain a simultaneous certificate; the coupling cost $\tau_1$ is paid only once. The proof and the tradeoff between the $1/\lambda$ and $\lambda/m_1$ terms appear in Appendices~\ref{app:proof-stage1} and~\ref{app:rates}.

\subsection{Propagation of Trend Error}\label{sec:prop}

Trend estimation affects both components of residual forecasting: it changes the future residual target and the history supplied to the residual model. To make these effects explicit, write $\Delta_t:=\widehat m(Z_t)-m_\star(Z_t)$, so that $\widehat R_t=R_t-\Delta_t$.

\begin{lemma}[Propagation]\label{lem:prop}
Under stationarity of the raw panel and Assumption~\ref{ass:reg}(a,b), for every fixed $F$ and $w\in\Omega(F)$, with $a_{H,L}:=\sqrt H+K\sqrt L$,
\begin{equation}\label{eq:prop}
\sqrt{R_E(\widehat m,w)}\le\sqrt{R_R(w)}+a_{H,L}\sqrt{\delta_1(\widehat m)}.
\end{equation}
\end{lemma}

The factor $\sqrt H$ captures trend error over the horizon, while $K\sqrt L$ reflects the model's sensitivity to perturbed histories. Squaring \eqref{eq:prop} yields $R_E \le R_R + 2a_{H,L}\sqrt{R_R\,\delta_1} + a_{H,L}^2\delta_1$, showing that when $R_R>0$, the propagated trend error enters the excess risk bound at the rate of $O(\sqrt{\delta_1})$.

\subsection{End-to-End Guarantee}\label{sec:e2e}

Let $U_t:=(\widehat R_{t-L:t-1},\widehat R_{t:t+H-1})$ be a residual forecasting example and $U:=((U_t)_{t\in\mathcal B_b})_{b=1}^{m_2}$ the retained block sample. Write $\widehat w=\mathsf{Train}(U;F,\xi)$ for the fitted parameter, where $\mathsf{Train}$ is the training algorithm and $\xi$ is its independent random seed. To separate training dependence from generalization, define an auxiliary law $Q$: retain the actual distribution of $F$, replace the raw Stage-II supports by independent stationary copies independent of $F$, then apply the frozen scaling and residualization maps. Let $U'$ be the resulting sample and $w':=\mathsf{Train}(U';F,\xi)$. Under $Q$, the transformed blocks are independent conditional on $F$. The quantity $I(U';w'\mid F)$ below is their conditional mutual information with the trained parameter, measured in nats under $Q$.

\begin{theorem}[End-to-end risk of chronological Hopformer]\label{thm:e2e}
Under Assumptions~\ref{ass:mix}--\ref{ass:reg}, let $R_R^\star(F):=\inf_{w\in\Omega(F)}R_R(w)$, $G_2:=B_2\sqrt{I(U';w'\mid F)/(2m_2)}$, and $\tau_2:=\min\{1,\beta(d_{AR})+(m_2-1)\beta(s_2)\}$. Then
\begin{equation}\label{eq:e2e}
\EE R_H\le\EE_F\!\left[\left(\sqrt{R_R^\star(F)}+a_{H,L}\sqrt{\delta_1(\widehat m)}\right)^2\right]+G_2+\varepsilon_o+B_2\tau_2.
\end{equation}
\end{theorem}

The expectation on the left averages over actual training data and algorithmic randomness; $\EE_F$ averages over the frozen upstream state. The benchmark $R_R^\star(F)$ is the best oracle-residual risk in the same admissible class. Above this benchmark, the bound accounts for propagated trend error, data-dependent residual learning through $G_2$, optimization error $\varepsilon_o$, and the coupling cost $B_2\tau_2$. See Appendix~\ref{app:theory} for the full proof, explicit finite-sample bounds, and rate analyses.

\section{Experiments}\label{sec:exp}

We evaluate forecasting accuracy alongside the structural contributions of covariate residualization, residual adaptation, and expert aggregation. Across six diverse panels, all 66 model--dataset evaluations for HopFormer and ten neural baselines, as well as all 42 matched-control and aggregation-ablation experiments, are fully completed under a unified evaluation protocol and metric definitions.

\paragraph{Data and evaluation.}
Table~\ref{tab:dataset-stats} summarizes three real-world and three synthetic panels. EPF supplies market-specific day-ahead forecast fields; their availability is assumed because issue-time vintages are unavailable. Retail future inputs are restricted to calendar features and store/family identities. Synthetic future inputs are explicitly oracle-exogenous; target-derived inventory, price, and construction components are excluded where applicable. Retail subsets are sampled from fitting-period identities, independently of model scores. We evaluate seen series, not transfer to unseen stores. Appendix~\ref{app:exp-data} documents the inputs and data limitations.

\begin{table}[ht]
\centering
\captionsetup{font=small}
\caption{Dataset sizes and frozen schedules. $N$: series; $T$: calendar length; h/d: frequency; $H$: horizon; $\Delta$: origin spacing; $K$: validation/test origins; $L$: candidate contexts; $m$: MASE lag. All lengths are in dataset time steps.}
\label{tab:dataset-stats}
\scriptsize\setlength{\tabcolsep}{2pt}\renewcommand{\arraystretch}{0.95}
\begin{tabular}{lrrcrrrrlr}
\toprule
Dataset & $N$ & $T$ & Freq. & $H$ & $\Delta$ & $K_{\rm val}$ & $K_{\rm test}$ & Candidate $L$ & $m$ \\
\midrule
EPF & 5 & 35,064 & h & 24 & 24 & 28 & 366 & 96, 168, 336, 720 & 24 \\
Sales1 & 200 & 1,096 & d & 28 & 7 & 8 & 49 & 56, 112, 365 & 7 \\
Sales2 & 200 & 1,096 & d & 28 & 7 & 8 & 49 & 56, 112, 365 & 7 \\
Syn. Electricity & 5 & 17,520 & h & 24 & 24 & 28 & 365 & 96, 168, 336, 720 & 24 \\
Rossmann-100 & 100 & 942 & d & 48 & 7 & 8 & 46 & 48, 96, 192 & 7 \\
Favorita-198 & 198 & 1,688 & d & 16 & 7 & 8 & 50 & 32, 64, 128, 365 & 7 \\
\bottomrule
\end{tabular}
\end{table}

All methods strictly share chronological fitting, validation, and test target intervals, identical rolling forecast tasks, and evaluation populations. Candidate context windows and learning rates are tuned jointly on the validation split by validation MASE, then frozen across all test origins; stochastic deep models average across seeds 2026--2028. All trained parameters remain fixed as the rolling history advances, precluding any test-time refitting or data leakage. HopFormer combines eight cross-sectional regression experts through SPA and adapts a Chronos-Bolt Small residual backbone using rank-8 LoRA optimized under native quantile loss, yielding median point forecasts. This implementation realizes the in-sample residual estimator, distinct from the chronological theoretical construction in Section~\ref{sec:thm}. Table~\ref{tab:combined_main_results} benchmarks external models; matched backbone controls and architectural ablations are evaluated separately[cite: 9]. Data splits, hyperparameter search grids, loss functions, input interfaces, and optimization budgets are detailed in Appendices~\ref{app:exp-protocol}--\ref{sec:pseudocode}.

\paragraph{Metrics and uncertainty.}
Lower MASE and MAPE are better. MASE uses fitting-only seasonal scales; MAPE uses observed, exactly nonzero targets without an epsilon. \textbf{MAPE is a decimal ratio: $0.10=10\%$.} We average losses over valid leads, eligible series, origins, and seeds, not forecasts across seeds. Paired circular time-block bootstraps use 10,000 draws, preserving the cross-section and averaging seeds before resampling. We call evidence \emph{stable} only when the effect direction and Holm-adjusted $p<0.05$ hold under all three declared block lengths. The original 252-slot external-comparison and 72-slot ablation families are retained. Separate window/seed SDs, pointwise intervals, adjusted values, and inferential limitations are in Appendix~\ref{app:exp-uncertainty}.

\begin{table}[htbp]
\centering
\caption{Mean forecast errors for HopFormer and neural baselines. MAPE is a decimal ratio. Bold indicates the lowest unrounded mean within these rows, not statistical significance. Complete Sales1/Sales2 comparisons appear in Table~\ref{tab:sales-results}; matched controls appear separately in Table~\ref{tab:ablation-main}.}
\label{tab:combined_main_results}
\small
\setlength{\tabcolsep}{3pt}\renewcommand{\arraystretch}{1.08}
\begin{tabular}{lrrrrrrrr}
\toprule
Model & \multicolumn{2}{c}{EPF} & \multicolumn{2}{c}{Syn. Electricity} & \multicolumn{2}{c}{Rossmann-100} & \multicolumn{2}{c}{Favorita-198} \\
\cmidrule(lr){2-3}\cmidrule(lr){4-5}\cmidrule(lr){6-7}\cmidrule(lr){8-9}
 & MASE & MAPE & MASE & MAPE & MASE & MAPE & MASE & MAPE \\
\midrule
HopFormer & \textbf{0.524} & 0.344 & \textbf{0.175} & \textbf{0.056} & 0.784 & 0.225 & \textbf{1.743} & 0.418 \\
\midrule
TimeXer & 0.612 & 0.426 & 0.991 & 0.307 & 0.799 & 0.226 & 1.822 & 0.453 \\
TimeMixer & 0.638 & 0.419 & 0.379 & 0.108 & 0.800 & 0.230 & 1.886 & 0.462 \\
TFT & 0.737 & \textbf{0.333} & 0.880 & 0.228 & 0.867 & 0.231 & 1.876 & \textbf{0.409} \\
TiDE & 0.761 & 0.440 & 3.025 & 0.838 & 0.941 & 0.293 & 1.928 & 0.469 \\
TSMixerX & 0.596 & 0.436 & 0.393 & 0.109 & 1.754 & 0.538 & 1.832 & 0.418 \\
NBEATSx & 0.621 & 0.408 & 0.719 & 0.185 & 1.137 & 0.333 & 1.846 & 0.428 \\
NHITS & 0.738 & 0.375 & 0.410 & 0.100 & 0.785 & 0.245 & 1.855 & 0.433 \\
PatchTST & 0.624 & 0.447 & 0.490 & 0.159 & 0.786 & 0.218 & 1.892 & 0.477 \\
iTransformer & 0.646 & 0.474 & 0.457 & 0.133 & 0.745 & 0.208 & 1.894 & 0.464 \\
DLinear & 0.743 & 0.581 & 0.489 & 0.136 & \textbf{0.718} & \textbf{0.201} & 1.944 & 0.491 \\
\bottomrule
\end{tabular}
\end{table}

\paragraph{Forecasting results.}
Relative to the ten task-trained baselines in Table~\ref{tab:combined_main_results}, HopFormer attains the lowest mean MASE on three of six datasets and the lowest MAPE on one. Its clearest advantage is Synthetic Electricity: MASE $0.175$ and MAPE $0.056$ improve by $53.7\%$ and $44.0\%$ over the best corresponding task-trained baselines, TimeMixer and NHITS. Both metrics favor HopFormer with stable adjusted evidence against nine of ten baselines; TimeXer does not meet the stability criterion. Against TimeMixer, the MASE difference is $-0.204$, with primary-block 95\% interval $[-0.253,-0.156]$ and adjusted $p=0.0252$ under every block setting. This result concerns oracle covariates and an exported synthetic target without a final independent observation-noise term.

On EPF and Favorita, HopFormer achieves competitive predictive accuracy, lowering mean MASE by $12.1\%$ and $4.3\%$ relative to the strongest task-trained competitors in Table~\ref{tab:combined_main_results} respectively. Statistical bootstrap analyses confirm that HopFormer stably outperforms five of the ten neural baselines on EPF under MASE, while maintaining statistically robust gains over eight baselines on Favorita in MASE and six in MAPE. Although TFT attains lower MAPE on these panels and DLinear remain effective on Rossmann, HopFormer consistently balances cross-sectional covariate strength with residual temporal dynamics across diverse market regimes. Detailed block-bootstrap distributions, multiple-testing corrections, and origin sensitivity bounds appear in Appendix~\ref{app:exp-uncertainty}.

\paragraph{Role of Synthetic Benchmarks.}
Beyond real-world panels, Sales1 and Sales2 serve as controlled synthetic benchmarks designed to isolate specific structural challenges, such as demand shifts and exogenous feature interactions. While dedicated architectures (e.g., TSMixerX) excel under certain narrow linear demand patterns, the synthetic environments provide a transparent testbed to probe HopFormer's modular components. In particular, we exploit these controlled settings for our ablation study (Table~3) and robustness sweeps over context lengths and prediction horizons (Figure~3), verifying the stability of two-stage decomposition under varying operational constraints. Complete baseline metrics on Sales1 and Sales2 and robustness studies are detailed in Appendix~\ref{app:exp-additional}.
\vspace{-0.1in}
\begin{table}[htbp]
\centering
\caption{Ablation MASE on all six datasets (lower is better). Every temporal backbone is Chronos-Bolt Small. Residual backbone frozen retains the fitted SPA trend; direct controls omit SPA. Trend-only rows contain no residual backbone; ElasticNet uses nonnegative aggregation weights. Bold denotes the lowest mean within these rows, not significance. Full metric pairs are in Table~\ref{tab:exp-ablation-results}.}
\label{tab:ablation-main}
\small
\setlength{\tabcolsep}{3pt}\renewcommand{\arraystretch}{1.08}
\begin{tabular}{lrrrrrr}
\toprule
Model / variant & EPF & Sales1 & Sales2 & Syn. Elec. & Rossmann & Favorita \\
\midrule
HopFormer & \textbf{0.524} & 0.474 & 0.313 & \textbf{0.175} & 0.784 & 1.743 \\
Residual backbone frozen & 0.540 & 0.556 & 0.517 & 0.193 & 0.729 & 1.777 \\
Direct backbone + LoRA & 0.528 & \textbf{0.389} & \textbf{0.209} & 0.608 & 0.748 & \textbf{1.682} \\
Direct backbone frozen & 0.563 & 0.603 & 0.567 & 0.753 & 0.846 & 1.742 \\
\midrule
SPA trend only & 1.300 & 0.800 & 0.604 & 0.474 & \textbf{0.686} & 4.344 \\
Equal-weight trend & 1.229 & 1.061 & 0.592 & 1.709 & 0.758 & 4.204 \\
Best-expert trend & 1.347 & 0.778 & 0.607 & 0.586 & 0.736 & 4.156 \\
ElasticNet trend & 1.331 & 0.844 & 0.836 & 0.455 & 0.743 & 4.168 \\
\bottomrule
\end{tabular}
\end{table}
\paragraph{Ablations.}
Table~\ref{tab:ablation-main} separates temporal-backbone controls from trend-only aggregators. Against direct adaptation of the same Small backbone, HopFormer reduces Synthetic Electricity MASE from $0.608$ to $0.175$ ($71.2\%$). The difference is $-0.433$, with primary-block interval $[-0.565,-0.305]$ and adjusted $p=0.0072$; all sensitivities remain below $0.05$ (maximum $0.0329$). EPF improves numerically but lacks stable matched-control evidence. Conversely, direct adaptation has stably lower MASE on Sales1, Sales2, and Favorita, and lower Sales2 MAPE. Residual adaptation improves mean MASE on five of six datasets, with stable gains on Sales1, Sales2, and Synthetic Electricity; Rossmann shifts from $0.729$ to $0.784$ due to noisy residual dynamics.

Across trend-aggregation variants, SPA achieves the lowest MASE among all aggregation schemes on four out of six panels, demonstrating marked resilience over heuristic averaging (e.g., MASE 0.474 vs.\ 1.709 on Synthetic Electricity). Furthermore, isolating the trend explains critical regime transitions: on noise-dominated panels such as Rossmann, the SPA trend alone attains a MASE of 0.686, outperforming all deep temporal baselines and indicating that residual adaptation can overfit when residual dynamics lack coherent temporal structure. Overall, these matched comparisons confirm that explicit trend extraction offers decisive gains when utilizing rich information provided by covariates, while adaptively highlighting domains where residual modeling is most effective.

\vspace{-0.1in}
\section{Discussion and Conclusion}\label{sec:conclusion}
\vspace{-0.1in}
We introduced HopFormer, a two-stage forecaster that combines a shared covariate-driven predictor with a LoRA-adapted residual backbone. For a chronological formulation, we establish an end-to-end risk bound under temporal mixing that rigorously bounds error propagation from trend residualization to downstream risk. Across six datasets, the empirical benefits are task-dependent. The clearest matched-backbone result is on Synthetic Electricity, where residualization reduces MASE by 71.2\% and remains favorable under every declared adjusted block-bootstrap comparison. Direct backbone adaptation performs better on the synthetic sales and retail panels. Future work includes task-dependent pipeline selection and in-sample, cross-fitted, and nonstationary extensions.

\bibliography{hopformer_iclr}
\bibliographystyle{iclr2027_conference}

\appendix
\section{Theory for the Chronological Estimator}\label{app:theory}

This appendix gives the full model, assumptions, chronological block design, and proofs for Section~\ref{sec:thm}. It also records the compact risk bound, rate illustration, and implementation limitations omitted from the main text.

\subsection{Model and Assumptions}\label{app:assumptions}

\paragraph{Raw panel and frozen scaling.}\label{app:scaling}
Write $O_t^0=\{(Y_{i,t}^0,Z_{i,t})\}_{i=1}^N$ for the balanced synchronous raw panel, indexed by $t\in\ZZ$. Its finite observed segment contains the training data; $Z_{i,t}$ denotes the underlying covariate information available at the forecast origin. Fit a scaling state $S$ using only the expert-fitting region $\cT_E$, and set $Y_{i,t}=T_{S,i}(Y_{i,t}^0)$. The maps $(s,y)\mapsto T_{s,i}(y)$ are jointly measurable and time-homogeneous. The state $S$ is part of both the frozen expert state $F_E$ and the later upstream state $F$. Fitted covariate preprocessing is absorbed into the predictor maps, which must remain functions of the corresponding $Z_{i,t}$ only.

\begin{assumption}[Stationary $\beta$-mixing raw panel]\label{ass:mix}
The raw process $(O_t^0)_{t\in\ZZ}$ is strictly stationary, and
\[
\beta(k):=\sup_t\left\|P_{O^0_{-\infty:t},\,O^0_{t+k:\infty}}-P_{O^0_{-\infty:t}}\otimes P_{O^0_{t+k:\infty}}\right\|_{\mathrm{TV}}\longrightarrow0,
\]
where $\|P-Q\|_{\mathrm{TV}}:=\sup_A|P(A)-Q(A)|$.
\end{assumption}

No stationarity or mixing assumption is imposed on the actual transformed trajectory $(Y_t,Z_t)$ after fitting $S$ to that same trajectory. All block comparisons below are made for $O^0$ \emph{before} applying the frozen scaling map. Assumption~\ref{ass:mix} allows arbitrary dependence across series within each snapshot; general nonstationary targets or deterministic time features that violate stationarity remain interesting topics for future work.

\paragraph{Fresh-draw convention and oracle.}
For a fixed fitted state, $\EE_\star$ integrates over an independent stationary raw-panel draw and then applies its frozen scaling map. More explicitly, for a window $\mathcal I$ and a measurable integrand $q$, $\EE_\star q(F,O_{\mathcal I}^{0,\star})=\int q(F,o)\,dP_{O_{\mathcal I}^0}(o)$, with $F$ held fixed. For each realized $s$, define
\[
m_{\star,s,i}(z):=\EE_\star[T_{s,i}(Y_{i,t}^{0,\star})\mid Z_{i,t}^{\star}=z],
\qquad
R_{i,t}^{s}:=T_{s,i}(Y_{i,t}^0)-m_{\star,s,i}(Z_{i,t}).
\]
We suppress $S$ in $Y$, $m_\star$, $R$, $\delta_1$, $R_E$, and $R_R$, but these population quantities depend on the frozen scaling state. The notation is \emph{not} a claim that conditioning the actual trajectory on $S$ preserves mixing. For affine scaling $T_{s,i}(y)=(y-a_i(s))/b_i(s)$, $b_i(s)>0$, the oracle satisfies $m_{\star,s,i}(z)=(m_{\star,i}^0(z)-a_i(s))/b_i(s)$, where $m_{\star,i}^0(z)=\EE[Y_{i,t}^0\mid Z_{i,t}=z]$. The comparison is in the same fitted units; it does not require consistency of $S$ for a population-optimal scale.

\paragraph{Panel notation.}
Write $Y_t=(Y_{i,t})_{i=1}^N\in\RR^N$, let $Z_t\in\RR^{N\times d}$ have $i$th row $Z_{i,t}^\top$, and set $X_{a:b}=[X_a,\ldots,X_b]$ for a sequence $X_t\in\RR^N$. For $V\in\RR^{N\times q}$, $\|V\|_N^2=N^{-1}\sum_{i=1}^N\sum_{r=1}^q V_{i,r}^2$. A panel predictor $\mathbf h=(h_i)_{i=1}^N$ acts as $\mathbf h(Z_t)=(h_i(Z_{i,t}))_{i=1}^N$. Shared experts $g_j$, candidates $\widehat m_{\bp}$, and the aggregate $\widehat m$ are identified with their panel lifts. Likewise, $[f_w(X)]_{i,:}=f_w(X_{i,:})$ for the shared residual map $f_w:\RR^L\to\RR^H$. The oracle $m_\star$ need not be shared. Define $\delta_1(\mathbf h)=\EE_\star\|\mathbf h(Z_t)-m_\star(Z_t)\|_N^2$, with the scale dependence understood as above; the forecasting risks are those in Section~\ref{sec:setting}.

\begin{assumption}[Uniform bounds, stability, chronology, and optimization]\label{ass:reg}
The following conditions hold; the block objective $\widehat L_2$, the law $Q$, and $w'$ are specified in Appendix~\ref{app:design}.
\begin{enumerate}
\item[(a)] There are deterministic finite constants $C,D$ common to all admissible fitted states. For every such scaling state $s$, $|T_{s,i}(Y_{i,t}^0)|\le C$ on the raw support, including fresh observations. Each candidate $\widehat m_{\bp}$ takes values in $[-C,C]$, and each $f_w$, $w\in\Omega(F)$, has coordinates in $[-D,D]$ throughout the residual input domain. The threshold $C$ is the one used in Section~\ref{sent:lse}.
\item[(b)] For every admissible $F$ and $w\in\Omega(F)$, the row-wise lifted residual map, including input-dependent scaling, is $K$-Lipschitz from $([-2C,2C]^{N\times L},\|\cdot\|_N)$ to $(\RR^{N\times H},\|\cdot\|_N)$, with a common deterministic $K$.
\item[(c)] All region boundaries, retained block locations, block lengths, and gaps are deterministic and prespecified. The expert state $F_E$, including $S$, is a measurable function of $O^0_{\cT_E}$ and randomness independent of the panel. The state $F$ is measurable with respect to raw data through the end of $\cT_A$ and such randomness. The backbone is fixed or independent of the panel. The finite grid $\Lambda$ and its priors $\pi_\alpha$ are fixed in advance; grid selection uses only data available through $\cT_A$. For a fixed-pair analysis, $\pi$ and $\lambda$ are fixed before calibration. The Stage-II seed $\xi$ is independent of the panel and upstream randomness.
\item[(d)] The admissible class $\Omega(F)$ is nonempty, the algorithm returns a measurable member of this class, and the losses and infima below are measurable. It satisfies the expected approximate-ERM condition
\[
\EE_Q\left[\widehat L_2(w')-\inf_{w\in\Omega(F)}\widehat L_2(w)\right]\le\varepsilon_o.
\]
\end{enumerate}
\end{assumption}

These are assumptions on the analyzed estimator, not certificates for an arbitrary implementation. In particular, fitting a min--max scaler on $\cT_E$ does not by itself bound fresh scaled observations, and dividing by an arbitrarily small fitted scale need not yield a common bound. Bounded raw support with suitable scale safeguards, or bounded target transformations, can ensure (a), but the relevant implementation must be checked. Similarly, (b) includes the backbone's own input-dependent normalization. A common per-series Euclidean Lipschitz constant implies the same constant for the row-wise panel lift. The bounds yield $B_1=4C^2$ and $B_2=H(2C+D)^2$ for the two block losses. If $K$ depends on $F$, the state-dependent coefficient must remain inside the expectation in~\eqref{eq:e2e}; the compact bound below requires a common $K$.

\subsection{Chronological Block Design}\label{app:design}

Let $T$ be the total training span, including buffers, and let $\cT_E<\cT_A<\cT_R$ be its prespecified regions. Stage~I fits the scaling state, experts, and pattern coefficients on $\mathcal E=[N]\times\cT_E$, using the construction in Section~\ref{sent:lse}. For calibration, retain $m_1$ blocks $A_b\subset\cT_A$ of $a_1$ consecutive snapshots and set $\mathcal A=[N]\times\bigcup_b A_b$, so $n_A=|\mathcal A|=Nm_1a_1$. Then
\[
V_b(\bp)=\frac1{a_1}\sum_{t\in A_b}\|Y_t-\widehat m_{\bp}(Z_t)\|_N^2,
\qquad
\widehat L_A(\bp)=\frac1{m_1}\sum_b V_b(\bp).
\]
This equals the mean calibration loss in Section~\ref{sent:lse}; no temperature rescaling is needed. Compute the Gibbs weights for each pair in $\Lambda$, select using data through $\cT_A$, and freeze the resulting $F$, including $S$, fitted candidates, weights, preprocessing, hyperparameters, and backbone.

For Stage~II, let $\mathcal J_R=\{t\in\cT_R:\{t-L,\ldots,t+H-1\}\subseteq\cT_R\}$ and retain prespecified blocks $\mathcal B_b\subset\mathcal J_R$ of $a_2$ consecutive forecast origins. Define
\[
U_t=(\widehat R_{t-L:t-1},\widehat R_{t:t+H-1}),\qquad
\ell(w;U_t)=\|\widehat R_{t:t+H-1}-f_w(\widehat R_{t-L:t-1})\|_N^2,
\]
\[
U=\bigl((U_t)_{t\in\mathcal B_b}\bigr)_{b=1}^{m_2},\qquad
\widehat L_2(w)=\frac1{m_2}\sum_{b=1}^{m_2}\frac1{a_2}\sum_{t\in\mathcal B_b}\ell(w;U_t),\qquad
\widehat w=\mathsf{Train}(U;F,\xi).
\]
Thus $\widehat L_2$ is~\eqref{eq:stage2_objective} with squared loss and $n_{\mathrm{res}}=Nm_2a_2$, restricted to the retained origins. The class $\Omega(F)$ is abstract; $\Omega_{\mathrm{LoRA}}$ is the implementation in Section~\ref{sec:transformer}.

\paragraph{Supports and gaps.}
The raw-panel support of an origin block is $\operatorname{supp}(\mathcal B_b)=\bigcup_{t\in\mathcal B_b}\{t-L,\ldots,t+H-1\}$, of length $a_2+L+H-1$. Gaps are the first index of the later support minus the last index of the earlier one. Let $d_{EA}$ be the gap from the end of $\cT_E$ to $A_1$, $s_1$ a lower bound on inter-calibration-block gaps, $d_{AR}$ the gap from the end of $\cT_A$ to $\operatorname{supp}(\mathcal B_1)$, and $s_2$ a lower bound on gaps between Stage-II supports. All gaps are positive. They index snapshots of the augmented raw panel $O^0=(Y^0,Z)$, not a process in which lagged covariates have been omitted. Target lags of at least $H$ ensure availability, but coverage also requires mixing for this augmented panel. When mixing is known only for a pre-augmentation process, support lengths and gaps must additionally account for the lags.

\paragraph{Independent-block law.}
Under $Q$, retain the actual marginal law of $F$ and replace the $m_2$ raw supports by independent stationary copies, independent of $F$. Apply the frozen scaling and residualization maps to each copy, yielding $U'$, and set $w'=\mathsf{Train}(U';F,\xi)$ with the same independent seed law. The transformed blocks are i.i.d. \emph{conditional on $F$}; unconditional independence of the scaled blocks is neither assumed nor needed. All information quantities involving $(U',w')$ are under $Q$. Total-variation contraction applies to the joint state and raw blocks before these transformations, so accounting for the fitted scale creates no additional coupling penalty.

\subsection{Algorithm}\label{app:alg}

Algorithm~\ref{alg:chrono} implements the chronological design above.

\begin{algorithm}[ht]
\caption{Chronological Hopformer (covered by Theorem~\ref{thm:e2e})}\label{alg:chrono}
\begin{algorithmic}[1]
\REQUIRE Raw panel; prespecified regions $\cT_E<\cT_A<\cT_R$, block lengths, and gaps; expert specifications; patterns $\cP$; grid $\Lambda$; context $L$; horizon $H$; backbone and adaptation configuration
\STATE Fit scaling state $S$ and other preprocessing on $\cT_E$ only; freeze these maps and use $Y_{i,t}=T_{S,i}(Y_{i,t}^0)$ thereafter
\STATE Fit $g_j$ and $\widehat\omega_{\bp}$ on $\mathcal E=[N]\times\cT_E$; construct and freeze the clipped candidates $\widehat m_{\bp}$ from Section~\ref{sent:lse}
\STATE Evaluate candidates on prespecified $A_b\subset\cT_A$, with first gap at least $d_{EA}$ and inter-block gaps at least $s_1$
\STATE Compute~\eqref{eq:spa} for each $(\lambda,\alpha)\in\Lambda$; select using data through $\cT_A$; freeze $\widehat m$ and the complete upstream state $F$
\STATE Form $\widehat R_t$ and $U_t$ on prespecified $\mathcal B_b\subset\mathcal J_R$, whose raw-panel supports lie in $\cT_R$, with first gap at least $d_{AR}$ and inter-support gaps at least $s_2$
\STATE Return $\widehat w=\mathsf{Train}(U;F,\xi)$ by approximately minimizing $\widehat L_2$ over $\Omega(F)$ using only retained Stage-II examples
\STATE Forecast by~\eqref{eq:hopformer_forecast}, with both stages fixed
\end{algorithmic}
\end{algorithm}

Examples crossing region boundaries, incomplete final blocks, and omitted gaps are excluded from the corresponding fitting, calibration, and Stage-II objectives. Selection of grid hyperparameters may use other observations available through $\cT_A$, but none from the residual-training region. The reported in-sample pipeline of Algorithm~\ref{alg:training} is a different estimator and is outside this theorem.

\subsection{Notation}\label{app:notation}

Table~\ref{tab:notation} collects the notation used in the model, algorithm, and proofs.

\begin{table}[htbp]
\centering
\small
\caption{Notation for the chronological analysis. Population quantities depend on the frozen scaling state, even where $S$ is suppressed.}
\label{tab:notation}
\begin{tabularx}{\linewidth}{@{}lX@{}}
\toprule
Symbol & Meaning \\
\midrule
$N,i,t,L,H,d$ & Number of series; series and time indices; context and horizon; covariate dimension \\
$Y_{i,t}^0,Z_{i,t},O_t^0$ & Raw target, underlying covariates, and augmented raw-panel snapshot \\
$S,T_{S,i},Y_{i,t}$ & Frozen scaling state, its time-homogeneous map, and transformed target \\
$\EE_\star$ & Independent stationary raw draw followed by frozen preprocessing \\
$\|V\|_N^2$ & $N^{-1}\sum_{i,r}V_{i,r}^2$: average over series, sum over window coordinates \\
$m_{\star,S},m_\star,R_t$ & Scale-dependent conditional mean, its abbreviated notation, and oracle residual \\
$g_j,\bp,S_{\bp},\widehat\omega_{\bp}$ & Shared expert; inclusion pattern; selected experts; within-pattern coefficients \\
$\widehat m_{\bp},\pi_\alpha,\widehat\rho,\widehat m$ & Clipped candidate; pattern prior; Gibbs weights; SPA aggregate \\
$\lambda,\Lambda$ & Inverse temperature and prespecified tuning grid \\
$\mathcal E,\mathcal A,n_A$ & Stage-I fitting set, calibration set, and calibration entry count \\
$\cT_E,\cT_A,\cT_R,F_E,F$ & Three time regions; frozen expert state; complete upstream state \\
$A_b,\mathcal B_b,a_j,m_j$ & Calibration and residual-origin blocks; block lengths and counts \\
$d_{EA},s_1,d_{AR},s_2$ & Upstream-to-block and inter-block gaps for the two stages \\
$\widehat R_t,U_t,U$ & Estimated residual, forecasting example, and retained block sample \\
$f_w,\Omega(F),\widehat w,\xi$ & Residual map, admissible class, learned parameter, and independent training seed \\
$Q,U',w'$ & Independent-raw-block law, transformed block sample, and $\mathsf{Train}(U';F,\xi)$ \\
$\delta_1,R_E,R_R,R_H$ & Trend error; end-to-end, oracle-residual, and fitted Hopformer risks \\
$C,D,K$ & Uniform scaled-target/candidate bound, residual-output bound, and Lipschitz constant \\
$B_1,B_2,a_{H,L}$ & $4C^2$, $H(2C+D)^2$, and $\sqrt H+K\sqrt L$ \\
$\tau_1,\tau_2,G_2,\varepsilon_o$ & Coupling penalties, information term, and expected optimization gap \\
$b_1^\Lambda,\bar\delta_1,\bar R_R$ & Selected-grid Stage-I certificate and expected-risk quantities in Appendix~\ref{app:compact} \\
\bottomrule
\end{tabularx}
\end{table}

\subsection{Two Elementary Facts}

\begin{lemma}[Projection identity]\label{lem:proj}
For every fixed scaling state and fixed bounded panel predictor $\mathbf h=(h_i)$, with $h_i$ depending only on $Z_{i,t}$,
\[
\EE_\star\|Y_t-\mathbf h(Z_t)\|_N^2=\EE_\star\|R_t\|_N^2+\delta_1(\mathbf h).
\]
\end{lemma}
\begin{proof}
Under the fresh-draw law with the scaling state fixed, $\EE_\star[R_{i,t}\mid Z_{i,t}]=0$ by the definition of $m_{\star,S,i}$. Expanding the square componentwise, the cross term vanishes by conditioning on $Z_{i,t}$. Average over series. This identity does not assert conditional orthogonality on the actual training trajectory given its fitted scale.
\end{proof}

\begin{lemma}[Raw-block coupling; \citealp{Yu1994mixing}]\label{lem:coupling}
Let $X_0$ be a measurable function of $O^0_{-\infty:t_0}$ and randomness independent of the panel. Let $\mathcal I_1<\cdots<\mathcal I_m$ be deterministic blocks with $\min\mathcal I_1-t_0\ge d$ and $\min\mathcal I_{b+1}-\max\mathcal I_b\ge s$. For $P=\mathcal L(X_0,O^0_{\mathcal I_1},\ldots,O^0_{\mathcal I_m})$ and $\widetilde P=P_{X_0}\otimes\bigotimes_bP_{O^0_{\mathcal I_b}}$,
\[
\|P-\widetilde P\|_{\mathrm{TV}}\le\beta(d)+(m-1)\beta(s).
\]
The same bound holds after any common measurable transformation of the joint state and blocks.
\end{lemma}
\begin{proof}
Write $\mathcal I_{1:m}=\bigcup_{b=1}^m\mathcal I_b$. By monotonicity of $\beta$ and contraction of total variation under measurable maps, $\|P-P_{X_0}\otimes P_{O^0_{\mathcal I_{1:m}}}\|_{\mathrm{TV}}\le\beta(d)$. An induction splits the raw blocks at total cost at most $(m-1)\beta(s)$. Tensoring with the common law $P_{X_0}$ preserves total variation, so the triangle inequality gives the claim. Further common measurable maps cannot increase the distance. Independent random seeds may equivalently be appended as a common product factor.
\end{proof}

\subsection{Proof of Proposition~\ref{prop:stage1}}\label{app:proof-stage1}

\emph{Step 1 (couple before scaling).} Let $\widetilde P=P_{F_E}\otimes\bigotimes_{b=1}^{m_1}P_{O^0_{A_b}}$. It preserves the actual expert-state law, including $S$, but makes the raw calibration blocks independent of that state and of one another. Conditional on $F_E$, apply its fixed scaling and candidate maps. The losses $V_b(\bp)$ are then i.i.d. in $[0,B_1]$, with mean $L_1(\widehat m_{\bp})=\EE_\star\|Y_t-\widehat m_{\bp}(Z_t)\|_N^2$. The mean and the oracle residual risk may depend on $S$, which is fixed in this conditional argument. Dependence within a block is unrestricted.

\emph{Step 2 (two PAC-Bayes inequalities).} Conditional on $F_E$, for every fixed $\bp$, Hoeffding's lemma gives
\[
\EE_{\widetilde P}\!\left[e^{\pm\lambda\{L_1(\widehat m_{\bp})-\widehat L_A(\bp)\}}\mid F_E\right]\le e^{\lambda^2B_1^2/(8m_1)}.
\]
Average over $\pi$, use the Donsker--Varadhan formula, and apply Markov's inequality. With conditional probability at least $1-\eta/2$ each, uniformly over distributions $\rho$ on $\cP$,
\[
\sum_{\bp}\rho(\bp)L_1(\widehat m_{\bp})\le\sum_{\bp}\rho(\bp)\widehat L_A(\bp)+\frac{\lambda B_1^2}{8m_1}+\frac{\mathrm{KL}(\rho\|\pi)+\log(2/\eta)}{\lambda},
\]
and the same inequality holds with $L_1$ and $\widehat L_A$ interchanged.

\emph{Step 3 (Gibbs variational step).} The posterior $\widehat\rho$ minimizes $\rho\mapsto\sum_{\bp}\rho(\bp)\widehat L_A(\bp)+\mathrm{KL}(\rho\|\pi)/\lambda$. Apply the first inequality at $\widehat\rho$, use this minimizing property, and apply the reversed inequality to point masses $\rho=\delta_{\bp}$ with $\mathrm{KL}(\delta_{\bp}\|\pi)=\log\{1/\pi(\bp)\}$. On the intersection event,
\[
\sum_{\bp}\widehat\rho(\bp)L_1(\widehat m_{\bp})\le\min_{\bp:\,\pi(\bp)>0}\left\{L_1(\widehat m_{\bp})+\frac{2[\log(1/\pi(\bp))+\log(2/\eta)]}{\lambda}\right\}+\frac{\lambda B_1^2}{4m_1}.
\]
This is the point-mass-comparator form of \citet[Thm.~4.2]{alquier2024user}.

\emph{Step 4 (convexity and projection).} For the realized weights, convexity gives
\[
L_1(\widehat m)\le\sum_{\bp}\widehat\rho(\bp)L_1(\widehat m_{\bp}).
\] Subtract the common, scale-dependent oracle residual risk $\EE_\star\|R_t\|_N^2$ and apply Lemma~\ref{lem:proj}. Conditional and hence unconditional $\widetilde P$-probability of~\eqref{eq:stage1} is at least $1-\eta$.

\emph{Step 5 (transfer and selection).} The event is a measurable function of the joint expert state and raw calibration blocks. Apply Lemma~\ref{lem:coupling} with $t_0=\max\cT_E$. Its probability changes by at most $\tau_1$, including after scaling. For a grid $\Lambda$, apply Steps 1--4 to every pair with failure budget $\eta/|\Lambda|$, and transfer the resulting \emph{single simultaneous event}. Thus the total failure probability is at most $\eta+\tau_1$, not $\eta+|\Lambda|\tau_1$. On this event every pair is certified, including a pair selected using other observations available through $\cT_A$. \hfill$\square$

\subsection{Proof of Lemma~\ref{lem:prop}}\label{app:proof-prop}

Fix $F$, including its scale, and $w\in\Omega(F)$. Let $\Delta_t=\widehat m(Z_t)-m_\star(Z_t)$, $e^R=R_{t:t+H-1}-f_w(R_{t-L:t-1})$, and $e^E=\widehat R_{t:t+H-1}-f_w(\widehat R_{t-L:t-1})$. Since $\widehat R_t=R_t-\Delta_t$,
\[
e^E=e^R-\Delta_{t:t+H-1}-\{f_w(\widehat R_{t-L:t-1})-f_w(R_{t-L:t-1})\}.
\]
The triangle inequality in $L^2(\PP_\star)$ bounds $\sqrt{R_E}$ by the norms of these three terms. The independent raw draw is stationary, and all frozen transformations are time-homogeneous, so
\[
\EE_\star\|\Delta_{t:t+H-1}\|_N^2=H\delta_1(\widehat m),\qquad
\EE_\star\|\Delta_{t-L:t-1}\|_N^2=L\delta_1(\widehat m).
\]
Both residual histories lie in $[-2C,2C]^{N\times L}$. Therefore Assumption~\ref{ass:reg}(b), followed by the preceding identity, gives
\[
\left(\EE_\star\|f_w(\widehat R_{t-L:t-1})-f_w(R_{t-L:t-1})\|_N^2\right)^{1/2}\le K\sqrt{L\delta_1(\widehat m)}.
\]
Combining the bounds proves~\eqref{eq:prop}. Only the frozen-scale fresh-draw process must be stationary here, not the actual data-scaled training trajectory. Squaring gives $R_E\le R_R+2a_{H,L}\sqrt{R_R\delta_1}+a_{H,L}^2\delta_1$; contemporaneous conditional-mean orthogonality does not generally cancel the cross term involving the perturbed history. \hfill$\square$

\subsection{Proof of Theorem~\ref{thm:e2e}}\label{app:proof-e2e}

\emph{Step 1 (auxiliary law).} Under $Q$, conditional on $F$, scaling and residualization are fixed time-homogeneous maps of independent stationary raw supports. For each fixed $w\in\Omega(F)$, the block losses satisfy
\[
V_b(w):=\frac1{a_2}\sum_{t\in\mathcal B_b}\ell(w;U_t),\qquad
\EE_Q[V_b(w)\mid F]=R_E(\widehat m,w),
\]
and are i.i.d. in $[0,B_2]$.

\emph{Step 2 (generalization).} A loss in $[0,B_2]$ is $(B_2/2)$-sub-Gaussian. Applying \citet[Thm.~1]{NIPS2017_info} conditionally on $F=f$ to the $m_2$ blocks gives
\[
\left|\EE_Q[R_E(\widehat m,w')-\widehat L_2(w')\mid F=f]\right|\le B_2\sqrt{\frac{I(U';w'\mid F=f)}{2m_2}}.
\]
Averaging over $F$ and applying Jensen's inequality yields $G_2$. Within-block dependence does not justify dividing the variance proxy by $a_2$, and no such reduction is used. The statement remains valid but uninformative when the mutual information is infinite.

\emph{Step 3 (approximate ERM).} Assumption~\ref{ass:reg}(d) and $\EE[\inf]\le\inf\EE$ imply
\begin{align*}
\EE_Q\widehat L_2(w')
&\le\EE_F\inf_{w\in\Omega(F)}\EE_Q[\widehat L_2(w)\mid F]+\varepsilon_o\\
&=\EE_F\inf_{w\in\Omega(F)}R_E(\widehat m,w)+\varepsilon_o.
\end{align*}
A comparator fixed conditional on $F$ has an unbiased block loss, so a second generalization penalty is unnecessary.

\emph{Step 4 (propagation).} Lemma~\ref{lem:prop}, nonemptiness of $\Omega(F)$, and continuity and monotonicity of $x\mapsto(\sqrt x+c)^2$ yield
\[
\inf_{w\in\Omega(F)}R_E(\widehat m,w)\le\left(\sqrt{R_R^\star(F)}+a_{H,L}\sqrt{\delta_1(\widehat m)}\right)^2.
\]
This step does not require an attained oracle minimizer. Combining Steps 2--3 with this inequality gives the right-hand side of~\eqref{eq:e2e} without its final coupling term under $Q$.

\emph{Step 5 (transfer).} The actual risk is a common measurable function
\[
R_H=\Psi\!\left(F,(O^0_{\operatorname{supp}(\mathcal B_b)})_{b=1}^{m_2},\xi\right)\in[0,B_2].
\] The function includes the fitted scale, the training algorithm, and fresh-draw evaluation. Apply Lemma~\ref{lem:coupling} at $t_0=\max\cT_A$ to the upstream state and \emph{raw} Stage-II blocks; appending the independent seed and applying $\Psi$ cannot increase total variation. Consequently,
\[
\left|\EE R_H-\EE_QR_E(\widehat m,w')\right|\le B_2\tau_2.
\]
The marginal law of $F$, and thus of the fitted-scale oracle and benchmark, is unchanged under $Q$. This proves~\eqref{eq:e2e} with no additional scaling-estimation term relative to that benchmark. \hfill$\square$

\subsection{Compact Expected-Risk Bound}\label{app:compact}

For a selected grid pair $(\widehat\lambda,\widehat\alpha)$, let
\[
b_1^\Lambda:=\min_{\bp:\,\pi_{\widehat\alpha}(\bp)>0}\left\{\delta_1(\widehat m_{\bp})+\frac{2[\log(1/\pi_{\widehat\alpha}(\bp))+\log(2|\Lambda|/\eta)]}{\widehat\lambda}\right\}+\frac{\widehat\lambda B_1^2}{4m_1}.
\]
A singleton grid recovers the fixed-pair certificate. Set
\[
\bar\delta_1:=\min\{B_1,\ \EE\min(B_1,b_1^\Lambda)+B_1\min(1,\eta+\tau_1)\},\qquad
\bar R_R:=\EE_F R_R^\star(F).
\]
\begin{corollary}[Explicit expected-risk certificate]\label{cor:e2e-compact}
Under the conditions of Theorem~\ref{thm:e2e},
\begin{equation}\label{eq:e2e-compact}
\EE R_H\le\left(\sqrt{\bar R_R}+a_{H,L}\sqrt{\bar\delta_1}\right)^2+G_2+\varepsilon_o+B_2\tau_2.
\end{equation}
\end{corollary}
\begin{proof}
The selected aggregate satisfies $\delta_1(\widehat m)\le\min(B_1,b_1^\Lambda)$ on the simultaneous Stage-I event and $\delta_1(\widehat m)\le B_1$ always, so $\EE\delta_1(\widehat m)\le\bar\delta_1$. Cauchy--Schwarz gives $\EE(\sqrt X+a\sqrt Y)^2\le(\sqrt{\EE X}+a\sqrt{\EE Y})^2$ for nonnegative $X,Y$. Apply this to~\eqref{eq:e2e} and use monotonicity in $\EE Y$.
\end{proof}

\subsection{Temperature and Rate Illustration}\label{app:rates}

The Stage-I certificate compares the aggregate with the best frozen candidate. For a fixed pair with $\lambda=2\sqrt{m_1}/B_1$, its complexity contribution is $B_1[\log(1/\pi(\bp))+\log(2/\eta)]/\sqrt{m_1}+B_1/(2\sqrt{m_1})$. In contrast, $\lambda=n_A/(4\sigma^2)$ with fixed $\sigma^2$ gives $\lambda B_1^2/(4m_1)=Na_1B_1^2/(16\sigma^2)$, which does not vanish. Grid selection certifies the chosen temperature but does not guarantee that it has the order needed for a vanishing bound.

Suppose $\beta(k)\le c_0e^{-\kappa k}$. Gaps of at least $\lceil2\kappa^{-1}\log T\rceil$ give $\tau_j\le c_0m_j/T^2\le c_0/T$. Regularly spaced retained blocks allow
\[
m_1\asymp\frac{|\cT_A|}{a_1+\kappa^{-1}\log T},\qquad
m_2\asymp\frac{|\cT_R|}{a_2+L+H+\kappa^{-1}\log T},
\]
provided the regions are long enough after the initial buffers. For an asymptotic illustration, hold $C,D,H,L,K$ fixed, take $\eta=m_1^{-1}$, and use $\lambda\asymp\sqrt{m_1}$ (uniformly for any selected grid temperatures). Suppose the selected prior admits a frozen comparator whose expected trend error is $\tilde O(m_1^{-1/2})$, while its expected log-prior penalty and $\log|\Lambda|$ grow at most polylogarithmically. Then $\bar\delta_1=\tilde O(m_1^{-1/2})$. Keeping $\eta$ fixed would not give this rate for the displayed expected certificate because its failure-event remainder contains $B_1\eta$.
For a positive oracle-residual benchmark, the propagated Stage-I contribution to the excess-risk \emph{upper bound} is at most $\tilde O(m_1^{-1/4})$; when $\bar R_R=0$, the cross term disappears and that contribution is $\tilde O(m_1^{-1/2})$. 

\subsection{Reported Metrics}\label{app:metrics}

Dividing~\eqref{eq:prop} by $\sqrt H$ gives the horizon-normalized root-MSE comparison
\[
\sqrt{R_E/H}\le\sqrt{R_R/H}+(1+K\sqrt{L/H})\sqrt{\delta_1}.
\]
For horizon-averaged MAE, the coordinatewise triangle inequality and Cauchy--Schwarz give $\mathrm{MAE}_E(w)\le\mathrm{MAE}_R(w)+(a_{H,L}/\sqrt H)\sqrt{\delta_1(\widehat m)}$. For MASE with series scales fixed during evaluation and bounded below by $s_{\min}>0$, divide the additive term by $s_{\min}$. These statements propagate trend error for a fixed $w$; they are not learning guarantees for MAE, MASE, or MAPE. They use the fitted target units. Under affine target scaling, inverse transformation multiplies each series' squared forecast error by $b_i(S)^2$; these factors must be retained or uniformly bounded to translate an averaged risk certificate into original units. No invertibility is asserted for general, potentially clipped scaling maps.

\subsection{The Information Term and LoRA}\label{app:info}

Conditional on $F$, suppose $w'$ takes at most $2^{q d_{\mathrm{tr}}}$ values, for example $d_{\mathrm{tr}}$ trainable parameters represented with $q$ bits each. Then $I(U';w'\mid F)\le\mathrm H(w'\mid F)\le qd_{\mathrm{tr}}\log2$ in nats. The trainable-parameter count is distinct from the covariate dimension $d$. For LoRA factors on matrices indexed by $J$,
\[
d_{\mathrm{tr}}=\sum_{\ell\in J}r_\ell(d_{\mathrm{in},\ell}+d_{\mathrm{out},\ell}),
\]
plus any other trainable parameters, such as biases or prediction heads. Hence $G_2\le\bar G_2:=B_2\sqrt{qd_{\mathrm{tr}}\log2/(2m_2)}$. For $B_2>0$, this \emph{counting certificate} is below $B_2$ if and only if $2m_2>qd_{\mathrm{tr}}\log2$; failure of that condition does not imply $G_2\ge B_2$.

This counting certificate directly leverages the discrete representation of machine precision without requiring post-hoc quantization. Crucially, since LoRA restricts adaptation to low-rank matrices where $r_\ell \ll \min\{d_{\mathrm{in},\ell}, d_{\mathrm{out},\ell}\}$, the effective parameter dimension satisfies $d_{\mathrm{tr}} \ll d_{\mathrm{total}}$, where $d_{\mathrm{total}}$ is the total parameter count of the backbone. Consequently, LoRA substantially tightens the mutual-information upper bound $\bar G_2$ relative to full fine-tuning, providing a principled theoretical justification for parameter-efficient adaptation within the chronological Stage-II framework. While the mutual information characterizes the estimation complexity under bounded storage, empirical forecasting performance further depends on the expressive capacity of the chosen low-rank subspace. 

\subsection{Scope and Theoretical Boundaries}\label{app:limits}

The guarantees in Section~\ref{sec:thm} cover the chronological estimator in Algorithm~\ref{alg:chrono} under Assumptions~\ref{ass:mix}--\ref{ass:reg}. They compare risk on independent evaluation windows with the frozen-candidate and oracle-residual benchmarks defined in the analysis. Arbitrary cross-sectional dependence is permitted, so the effective sample sizes are temporal block counts rather than the total number of panel entries. The implemented in-sample estimator and alternative cross-fitted or nonstationary procedures are not covered by the estimator and assumptions analyzed here.
\section{Synthetic Data}\label{sec:synthetic}
We use three synthetic constructions: Sales1, Sales2, and Synthetic Electricity. The descriptions below follow the recovered generators and distinguish their internal variables from the exported targets and permitted forecasting inputs. The evaluated benchmarks are the supplied, fixed CSV realizations; byte-identical regeneration from the recovered code has not been established. The original generation figures illustrate the constructions, not forecasting results under the revised evaluation protocol. Dataset dimensions and evaluation settings are in Table~\ref{tab:dataset-stats}.

\subsection*{Synthetic Sales Data -- Sales1}
Sales1 contains 200 daily store series from 2020-01-01 through 2022-12-31. Each store $i$ has its own amplitude $A_i$, cycle period $f_i$, baseline $B_i$, promotion effect $P_i$, temperature coefficient $T_i$, price coefficient $C_i$, promotion probability $p_i$, and noise scales $\sigma_{T_i}$ and $\sigma_{S_i}$. Table~\ref{tab:synthetic_params_compact} gives the nominal normal-draw parameters; $f_i$ is instead sampled uniformly from $\{7,14,30,90\}$ days. The latent daily sales and exported target are
\begin{align*}
\widetilde y_{i,t}
&=B_i+A_i\sin\left(\frac{2\pi t}{f_i}\right)
  +P_i\,\mathrm{promotion}_{i,t}
  +T_i\,\mathrm{temperature}_{i,t}
  +C_i\,\mathrm{price}_{i,t}+\varepsilon_{i,t},\\
y_{i,t}&=\operatorname{round}_2\!\left(\max\{0,\widetilde y_{i,t}\}\right),
\qquad \varepsilon_{i,t}\sim\mathcal N(0,\sigma_{S_i}^2).
\end{align*}
Here $t$ counts days from the start, and $\operatorname{round}_2$ denotes rounding to two decimal places. Promotion is Bernoulli with probability $p_i$. Writing $d_t$ for day of year, temperature is $20+15\sin(2\pi d_t/365)+\xi_{i,t}$, where $\xi_{i,t}\sim\mathcal N(0,\sigma_{T_i}^2)$. Price is fixed at 10, so this dataset does not test response to time-varying prices. The exported temperature is rounded after sales are computed.

\begin{table}[htbp]
\centering\small\setlength{\tabcolsep}{4pt}
\caption{Sales1 parameter distributions before clipping and rounding. The period column gives its discrete support, not a normal-distribution mean.}
\label{tab:synthetic_params_compact}
\begin{tabular}{lcccccccc}
\toprule
& $A$ & $f$ & $B$ & $P$ & $T$ & $C$ & $p$ & $\sigma_T,\sigma_S$ \\
\midrule
Mean & 80 & $\{7,14,30,90\}$ & 200 & 30 & 0.3 & $-12$ & 0.20 & 2.0, 10.0 \\
Std. dev. & 30 & -- & 80 & 10 & 0.1 & 5 & 0.05 & 0.5, 3.0 \\
\bottomrule
\end{tabular}
\end{table}

The generator clips draws to $A_i\geq10$, $B_i\geq50$, $P_i\geq0$, $C_i\leq0$, $0.05\leq p_i\leq0.40$, $\sigma_{T_i}\geq0.5$, and $\sigma_{S_i}\geq2$, then rounds the stored parameters. These are heterogeneous store-specific responses to a shared feature schema, not identical responses across stores. Future promotion, temperature, and price values are supplied under the synthetic \emph{oracle-exogenous} setting; in particular, future realized temperature is not an operational weather forecast.

Figure~\ref{fig:sales1_generation} retains the original parameter-distribution and sample-trajectory illustrations.
\begin{figure}[htbp]
\centering
\begin{subfigure}[b]{0.48\linewidth}
  \centering
  \includegraphics[width=\linewidth]{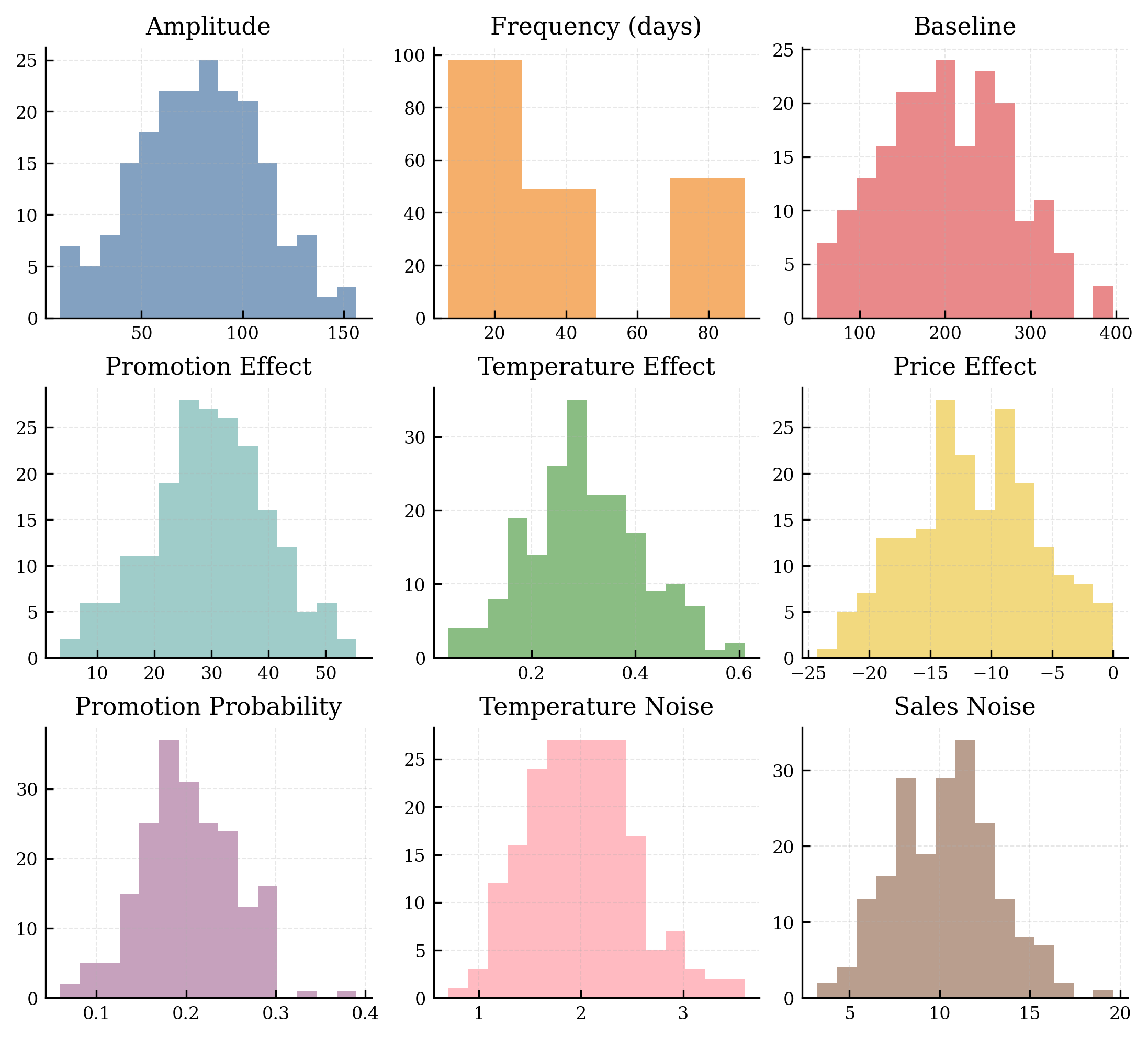}
  \caption{Store-specific parameter distributions. The original plot's ``Frequency (days)'' denotes cycle period.}
  \label{fig:store_param_dist}
\end{subfigure}\hfill
\begin{subfigure}[b]{0.48\linewidth}
  \centering
  \includegraphics[width=\linewidth]{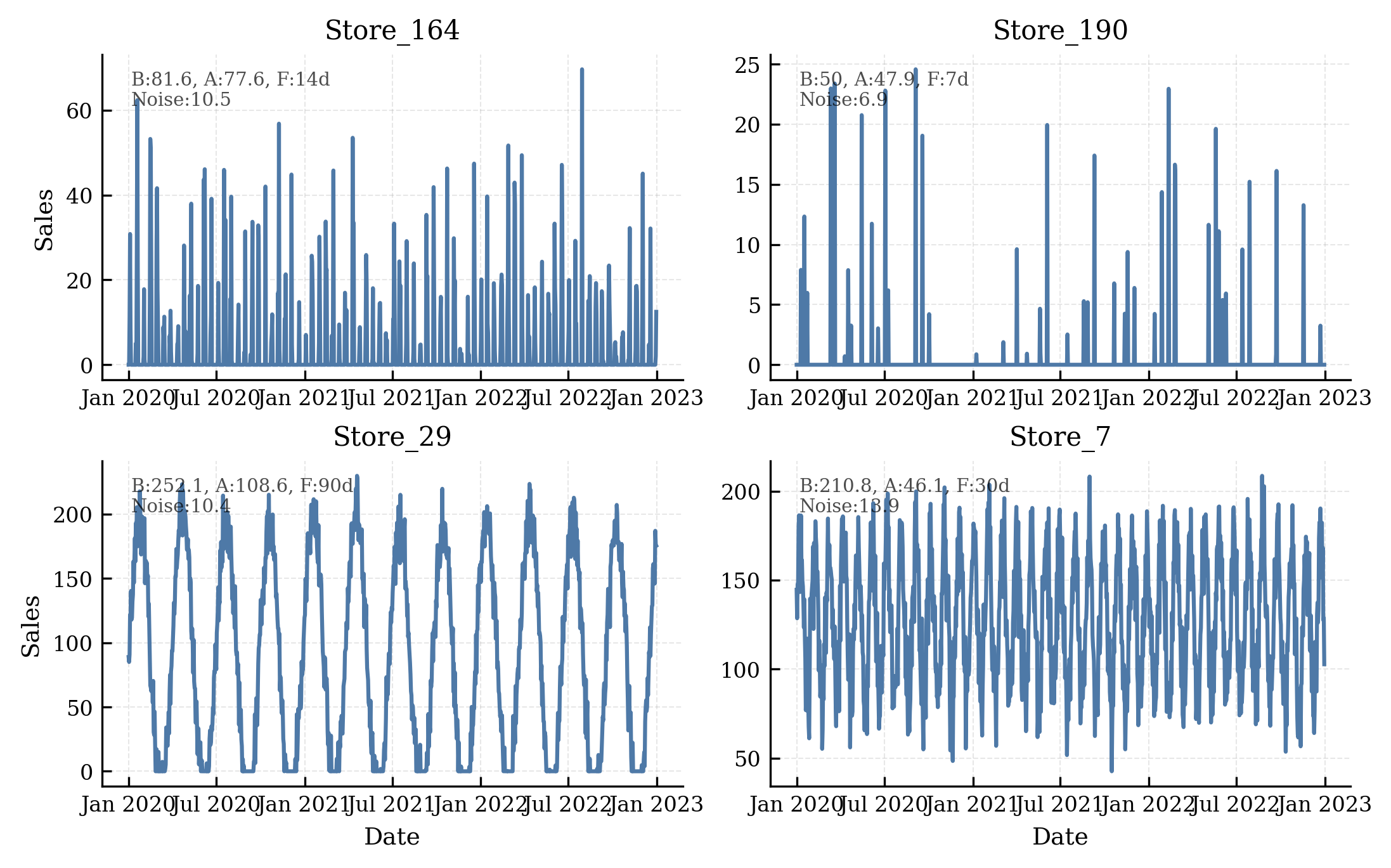}
  \caption{Sample sales trajectories from the original generation visualization.}
  \label{fig:store_sales_sample}
\end{subfigure}
\caption{Sales1 generation illustrations. Left: parameter heterogeneity. Right: example temporal trajectories. These are not forecasting comparisons.}
\label{fig:sales1_generation}
\end{figure}

\subsection*{Synthetic Sales Data -- Sales2}
Sales2 contains 200 daily store series over the same dates. It extends Sales1 with trend and annual seasonality, nonlinear temperature and price responses, promotion interactions, persistent marketing spend, and inventory-dependent sales. The nominal parameter distributions are in Table~\ref{tab:complex_params}. Let
\[
g_{i,t}=B_i+\alpha_i t+A_i\sin\left(\frac{2\pi t}{f_i}\right)
                 +S_i\sin\left(\frac{2\pi d_t}{365}\right).
\]
The pre-inventory covariate effect is
\begin{align*}
E^{\mathrm{cov}}_{i,t}
={}&\beta_{\mathrm{promo},i}\,\mathrm{promo}_{i,t}
                 (1-0.3\,\mathrm{comp}_{i,t})
 +\beta_{\mathrm{holiday},i}\,\mathrm{holiday}_t
 +\beta_{\mathrm{weekend},i}\,\mathrm{weekend}_t\\
&+\beta_{\mathrm{comp},i}\,\mathrm{comp}_{i,t}
 +\beta_{\mathrm{mkt},i}\sqrt{M_{i,t}}\\
&
 +\beta_{\mathrm{temp},i}\!\left[1-0.03(T_{i,t}-T_i^*)^2\right]
 +\beta_{\mathrm{price},i}\left(\frac{\mathrm{price}_{i,t}}{10}\right)^{\eta_i}.
\end{align*}
Promotion, competitor promotion, holiday, and weekend variables are binary. The temperature process is as in Sales1, while
$\mathrm{price}_{i,t}=10(1-0.1\,\mathrm{promo}_{i,t})[1+0.05\sin(t/30)]$.
Marketing spend $M_{i,t}$ starts at zero. With probability 0.03 on each day it is reset to a uniform draw on $[70,130]$; otherwise $M_{i,t}=\lambda_i M_{i,t-1}$. Its contribution to sales is proportional to $\sqrt{M_{i,t}}$, not to spend itself.

Let $I_{i,t}^{\mathrm{pre}}$ be inventory after any restocking but before current-day sales. The actual target construction is
\begin{align*}
q(I)&=\begin{cases}0.5+I/60,&I<30,\\1,&I\geq30,\end{cases}\\
s_{i,t}&=\max\{0,q(I_{i,t}^{\mathrm{pre}})(g_{i,t}+E^{\mathrm{cov}}_{i,t})+\varepsilon_{i,t}\},
\qquad \varepsilon_{i,t}\sim\mathcal N(0,\sigma_{S_i}^2),\\
y_{i,t}&=\operatorname{round}_2(s_{i,t}).
\end{align*}
Inventory starts at 100, receives up to 70 units every 14 days subject to a cap of 100, and is updated after sales as $I_{i,t}^{\mathrm{post}}=\max\{0,I_{i,t}^{\mathrm{pre}}-s_{i,t}/10\}$. The CSV stores this \emph{post-sales} inventory, rounded to two decimals. It therefore depends on the current target and is excluded from both historical and future forecasting inputs. The coefficient $\beta_{\mathrm{inv}}$ is sampled by the recovered code but not used in $q$; inventory is not an additive covariate effect.

\begin{table}[htbp]
\centering\scriptsize\setlength{\tabcolsep}{2.5pt}
\caption{Sales2 nominal normal-draw parameters before clipping and rounding. The period is sampled from the stated support. $\beta_{\mathrm{inv}}$ is retained for documentation but unused in the sales calculation.}
\label{tab:complex_params}
\resizebox{\linewidth}{!}{%
\begin{tabular}{l*{10}{c}}
\toprule
Symbol & $A$ & $B$ & $f$ & $\beta_{\mathrm{promo}}$ & $\beta_{\mathrm{weekend}}$ & $\beta_{\mathrm{holiday}}$ & $\beta_{\mathrm{temp}}$ & $T^*$ & $\beta_{\mathrm{price}}$ & $\eta$ \\
\midrule
Mean & 80 & 200 & $\{7,14,30,90\}$ & 30 & 15 & 40 & 0.3 & 22 & $-12$ & 1.5 \\
Std. dev. & 30 & 80 & -- & 10 & 5 & 15 & 0.1 & 3 & 5 & 0.3 \\
\midrule
Symbol & $\beta_{\mathrm{comp}}$ & $\beta_{\mathrm{mkt}}$ & $\lambda$ & $\beta_{\mathrm{inv}}$ & $S$ & $\alpha$ & $p$ & $p_{\mathrm{comp}}$ & $\sigma_T$ & $\sigma_S$ \\
\midrule
Mean & $-15$ & 0.4 & 0.85 & 0.2 & 30 & 0.01 & 0.2 & 0.15 & 2.0 & 10.0 \\
Std. dev. & 5 & 0.1 & 0.05 & 0.05 & 10 & 0.005 & 0.05 & 0.05 & 0.5 & 3.0 \\
\bottomrule
\end{tabular}}
\end{table}

Shared parameters use the Sales1 clipping rules. Additional constraints are $\beta_{\mathrm{holiday}},\beta_{\mathrm{mkt}},S\geq0$, $\beta_{\mathrm{comp}}\leq0$, $\eta\geq1$, $0.5\leq\lambda\leq1$, and $0.05\leq p_{\mathrm{comp}}\leq0.30$. Temperature, price, and marketing spend are rounded in the export after their unrounded values enter the calculation. The admitted future inputs are promotion, competitor promotion, temperature, price, holiday, marketing spend, and calendar fields under the oracle-exogenous convention; inventory remains unobserved by the forecasters.

Figure~\ref{fig:sales2_generation} retains the original illustrations of store heterogeneity and sample trajectories.
\begin{figure}[htbp]
\centering
\begin{subfigure}[b]{0.48\linewidth}
  \centering
  \includegraphics[width=\linewidth]{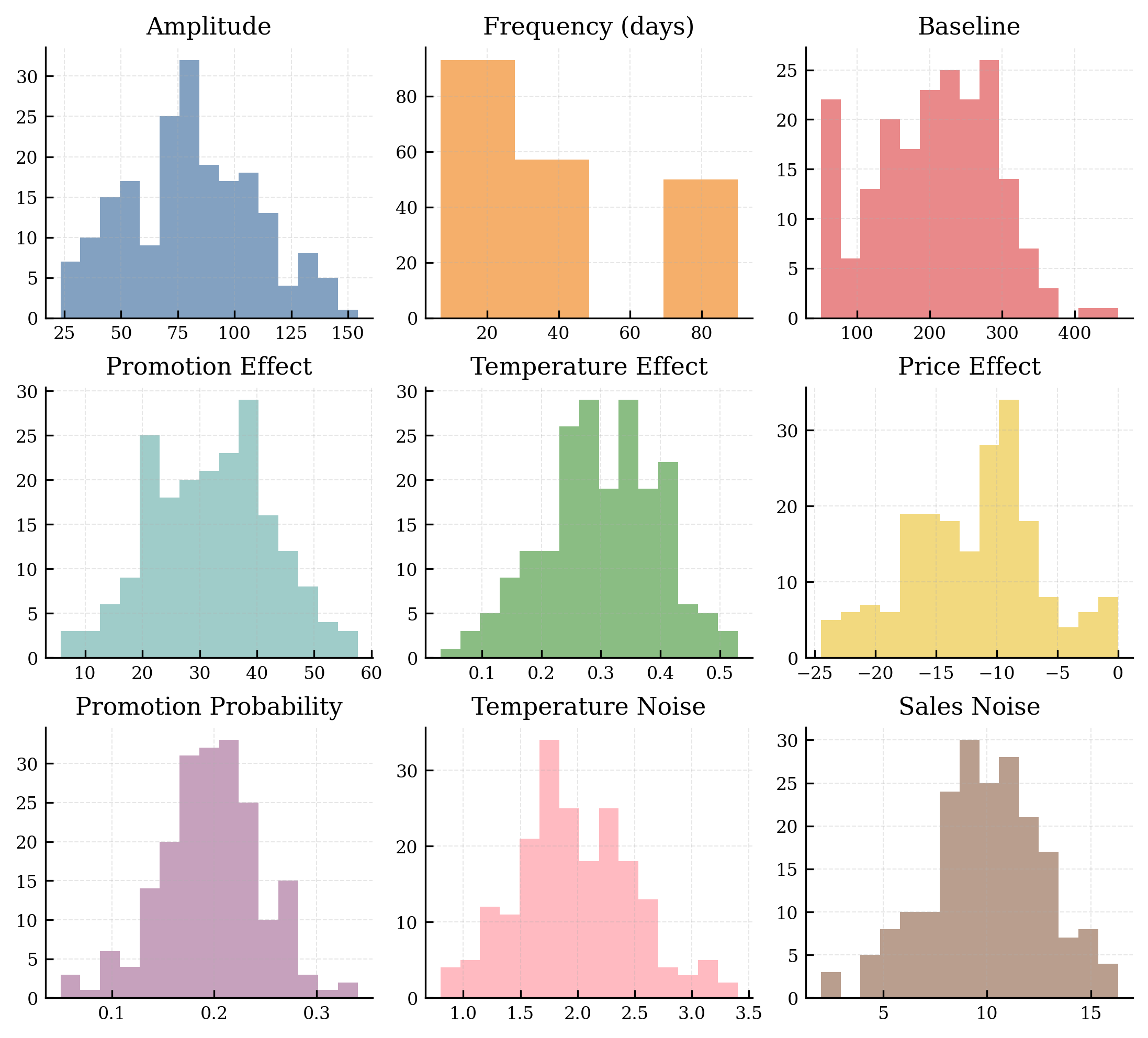}
  \caption{Store-specific parameter distributions from the original generation visualization.}
  \label{fig:store_param_dist_sale2}
\end{subfigure}\hfill
\begin{subfigure}[b]{0.48\linewidth}
  \centering
  \includegraphics[width=\linewidth]{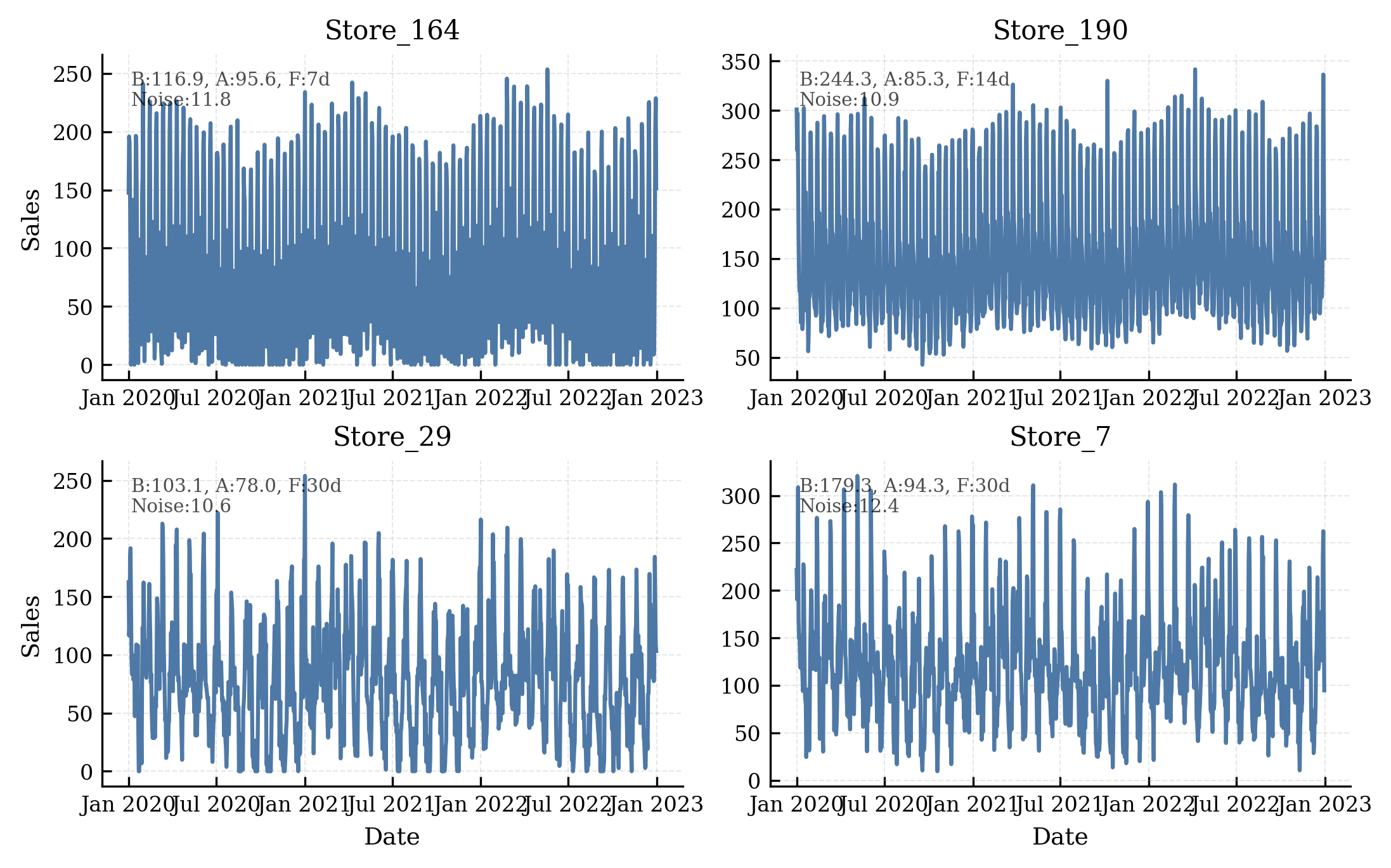}
  \caption{Sample sales trajectories from the original generation visualization.}
  \label{fig:store_sales_sample_sale2}
\end{subfigure}
\caption{Sales2 generation illustrations. Left: parameter heterogeneity. Right: example temporal trajectories. These do not compare forecasting methods.}
\label{fig:sales2_generation}
\end{figure}

\subsection*{Synthetic Electricity Load Generation}
Synthetic Electricity contains five hourly regional series over 730 complete days, from 2020-01-01 through 2021-12-30. We exclude the single trailing midnight observation per region beyond this interval. The recovered generator combines weather, calendar, and infrastructure effects with region-specific daily and weekly profiles. Its base component is multiplicative:
\[
b_{r,t}=B_r D_r(h_t)W_r(w_t)
 \left[1+A_{\mathrm{year},r}\sin\left(\frac{2\pi(d_t-15)}{365}\right)\right]
 (1+\alpha_r)^{\tau_t},
\]
where $h_t$, $w_t$, and $d_t$ are hour, weekday, and day of year, and $\tau_t$ is elapsed whole days divided by 365. Daily and weekly profiles are selected from five templates each, perturbed, and clipped to $[0.4,1.5]$ and $[0.6,1.2]$, respectively. The annual amplitude, trend, and other parameter ranges appear in Table~\ref{tab:region_param_ranges}.

Writing $T,H,V,S$ for approximate temperature, humidity, wind, and solar input, the component effects are
\begin{align*}
e^{\mathrm{cal}}_{r,t}&=b_{r,t}(\beta_{\mathrm{weekend},r}\,\mathrm{Weekend}_t
                  +\beta_{\mathrm{holiday},r}\,\mathrm{Holiday}_t
                  +\beta_{\mathrm{DST},r}\,\mathrm{DST}_t),\\
e^{\mathrm{weather}}_{r,t}&=b_{r,t}\!\left[f_{T,r}(T_{r,t})
            +\beta_{\mathrm{humid},r}(H_{r,t}-60)
            +\beta_{\mathrm{wind},r}V_{r,t}
            +\beta_{\mathrm{solar},r}\frac{S_{r,t}}{800}\right],\\
e^{\mathrm{outage}}_{r,t}&=b_{r,t}\beta_{\mathrm{outage},r}\,\mathrm{PlannedOutage}_{r,t},\\
f_{T,r}(T)&=\gamma_{\mathrm{cool},r}(T-T_r^C)_+^{1.5}
                  +\gamma_{\mathrm{heat},r}(T_r^H-T)_+^{1.2},
\end{align*}
where $(a)_+=\max(0,a)$. The four approximate weather inputs are constructed from supplied weather-factor series $z_1,\ldots,z_4$ as $T=20+10z_1$, $H=60+15z_2$, $V=5+5z_3$, and $S=\max(0,500+300z_4)$. The current forecasting schema uses these approximate inputs, not the component columns.

\paragraph{Exported target and auxiliary price.}
The code first constructs a noisy, nonnegative intermediate load,
\[
\widetilde L_{r,t}=\max\{0,b_{r,t}+e^{\mathrm{cal}}_{r,t}
          +e^{\mathrm{weather}}_{r,t}+e^{\mathrm{outage}}_{r,t}
          +\varepsilon_{r,t}\},
\quad \varepsilon_{r,t}\sim\mathcal N(0,(\sigma_{L,r}b_{r,t})^2).
\]
However, the exported \texttt{total\_load}, which is the evaluated target, is subsequently assigned
\[
y_{r,t}=b_{r,t}+e^{\mathrm{cal}}_{r,t}+e^{\mathrm{weather}}_{r,t}+e^{\mathrm{outage}}_{r,t}.
\]
Our export audit reproduces this sum exactly. We preserve the exported values, including negatives, rather than adding noise or clipping the target. Thus the benchmark does not test robustness to a final independent observation-noise term.

The auxiliary price depends on the intermediate load $\widetilde L_{r,t}$. The load factor and price are
\begin{align*}
\ell_{r,t}&=\frac{\widetilde L_{r,t}}{1.5B_r c_r},\\
P_{r,t}&=\max\!\left\{1,P_{0,r}[1+0.5(\ell_{r,t}-0.6)]
       +P_{0,r}k_r\left(\frac{(\ell_{r,t}-c_r)_+}{1-c_r}\right)^{\gamma_r}
       +a_{r,t}\right\},\\
a_{r,t}&=0.7a_{r,t-1}+0.3\nu_{r,t},\qquad
\nu_{r,t}\sim\mathcal N(0,(v_{P,r}P_{0,r})^2),\quad a_{r,0^-}=0.
\end{align*}
Here $c_r$ is the capacity parameter and $k_r$ the peak-price multiplier. We exclude price and all four construction components from both historical and future predictor sets. Future weather and planned-outage inputs are oracle-exogenous; they are not evidence of realistic forecast availability.

\begin{table}[htbp]
\centering\small\setlength{\tabcolsep}{4pt}
\caption{Nominal uniform sampling ranges in the recovered electricity generator. Temperature thresholds are adjusted to bracket $T^*$, so adjusted values need not lie in the nominal intervals. Daggers denote sampled fields unused in the load/price calculation shown above.}
\label{tab:region_param_ranges}
\begin{tabular}{llcc}
\toprule
Category & Symbol & Min & Max \\
\midrule
Base load and seasonality & $B_r$ (base load, MW) & 500 & 5000 \\
& $A_{\mathrm{year},r}$ (annual amplitude) & 0.15 & 0.35 \\
& $\alpha_r$ (annual growth) & $-0.02$ & 0.04 \\
\midrule
Temperature sensitivity & $T^*$ (threshold-ordering reference) & 18.0 & 23.0 \\
& $T^C$ (cooling threshold) & 22.0 & 27.0 \\
& $T^H$ (heating threshold) & 12.0 & 18.0 \\
& $\gamma_{\mathrm{cool}}$ & 0.02 & 0.06 \\
& $\gamma_{\mathrm{heat}}$ & 0.01 & 0.04 \\
\midrule
Weather sensitivity & $\beta_{\mathrm{humid}}$ & 0.001 & 0.005 \\
& $\beta_{\mathrm{wind}}$ & $-0.005$ & 0.002 \\
& $\beta_{\mathrm{solar}}$ & $-0.02$ & $-0.005$ \\
\midrule
Calendar effects & $\beta_{\mathrm{weekend}}$ & $-0.25$ & $-0.05$ \\
& $\beta_{\mathrm{holiday}}$ & $-0.30$ & $-0.10$ \\
& $\beta_{\mathrm{DST}}$ & $-0.05$ & 0.05 \\
\midrule
Infrastructure & $\beta_{\mathrm{outage}}$ & $-0.20$ & $-0.05$ \\
\midrule
Renewable fields & Substitution coefficient$^\dagger$ & 0.20 & 0.80 \\
& Volatility$^\dagger$ & 0.01 & 0.05 \\
\midrule
Price model & $P_{0,r}$ (base price) & 20 & 60 \\
& $k_r$ (peak multiplier) & 1.5 & 4.0 \\
& $v_{P,r}$ (momentum innovation scale) & 0.05 & 0.20 \\
\midrule
Capacity constraints & $c_r$ & 0.80 & 0.95 \\
& $\gamma_r$ (scarcity exponent) & 1.5 & 3.0 \\
\midrule
Noise fields & $\sigma_L$ (intermediate-load noise) & 0.01 & 0.05 \\
& \texttt{price\_noise}$^\dagger$ & 0.05 & 0.15 \\
\bottomrule
\end{tabular}
\end{table}

Figure~\ref{fig:region_components} shows the original summer and winter component plots for two regions. Price is shown only for illustration, not as a forecasting input. The unused renewable fields in Table~\ref{tab:region_param_ranges} do not contribute to these components.
\begin{figure}[htbp]
\centering
\begin{subfigure}[b]{0.48\linewidth}
  \centering
  \includegraphics[width=\linewidth]{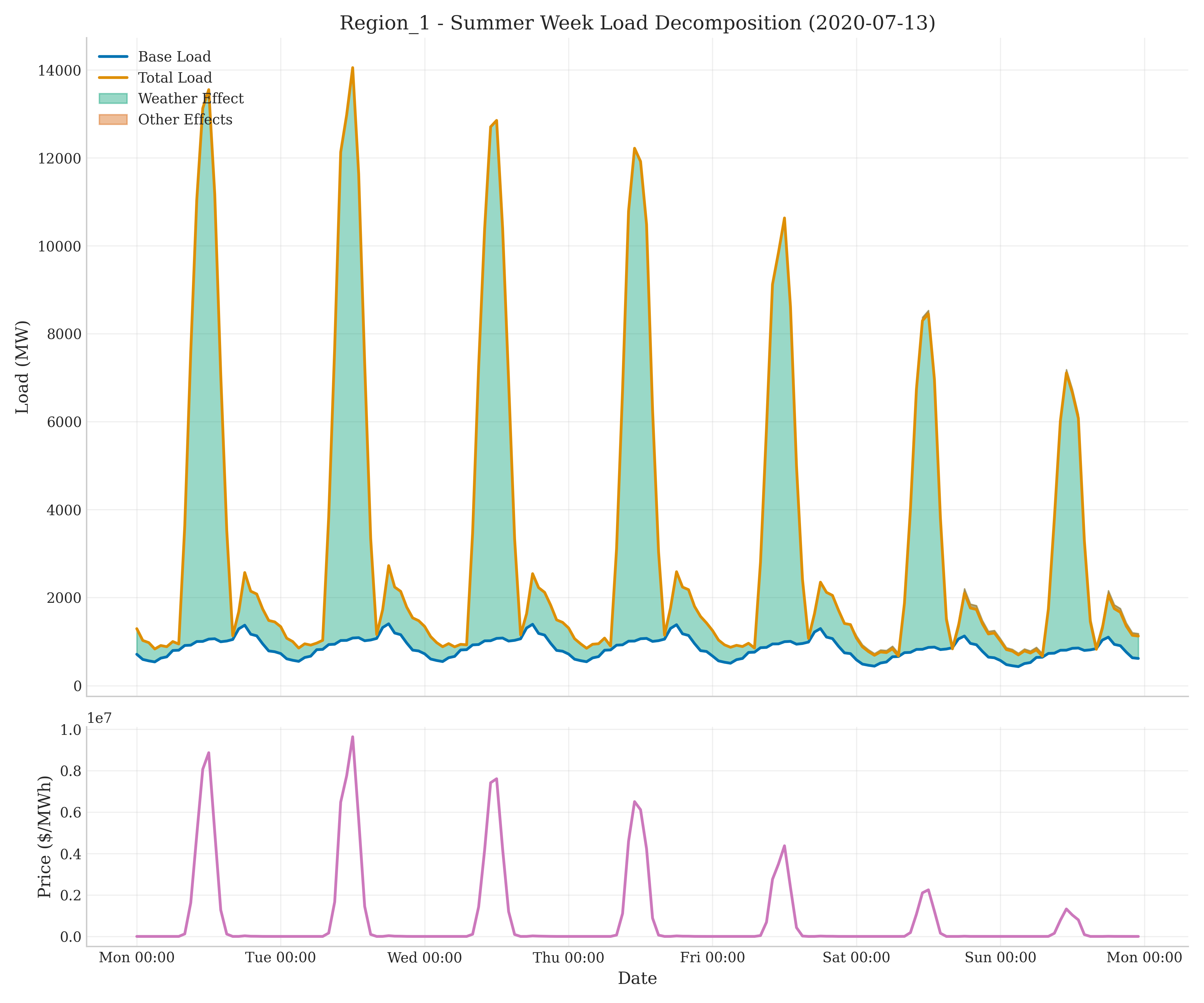}
  \caption{Region 1: summer.}\label{fig:region1_summer}
\end{subfigure}\hfill
\begin{subfigure}[b]{0.48\linewidth}
  \centering
  \includegraphics[width=\linewidth]{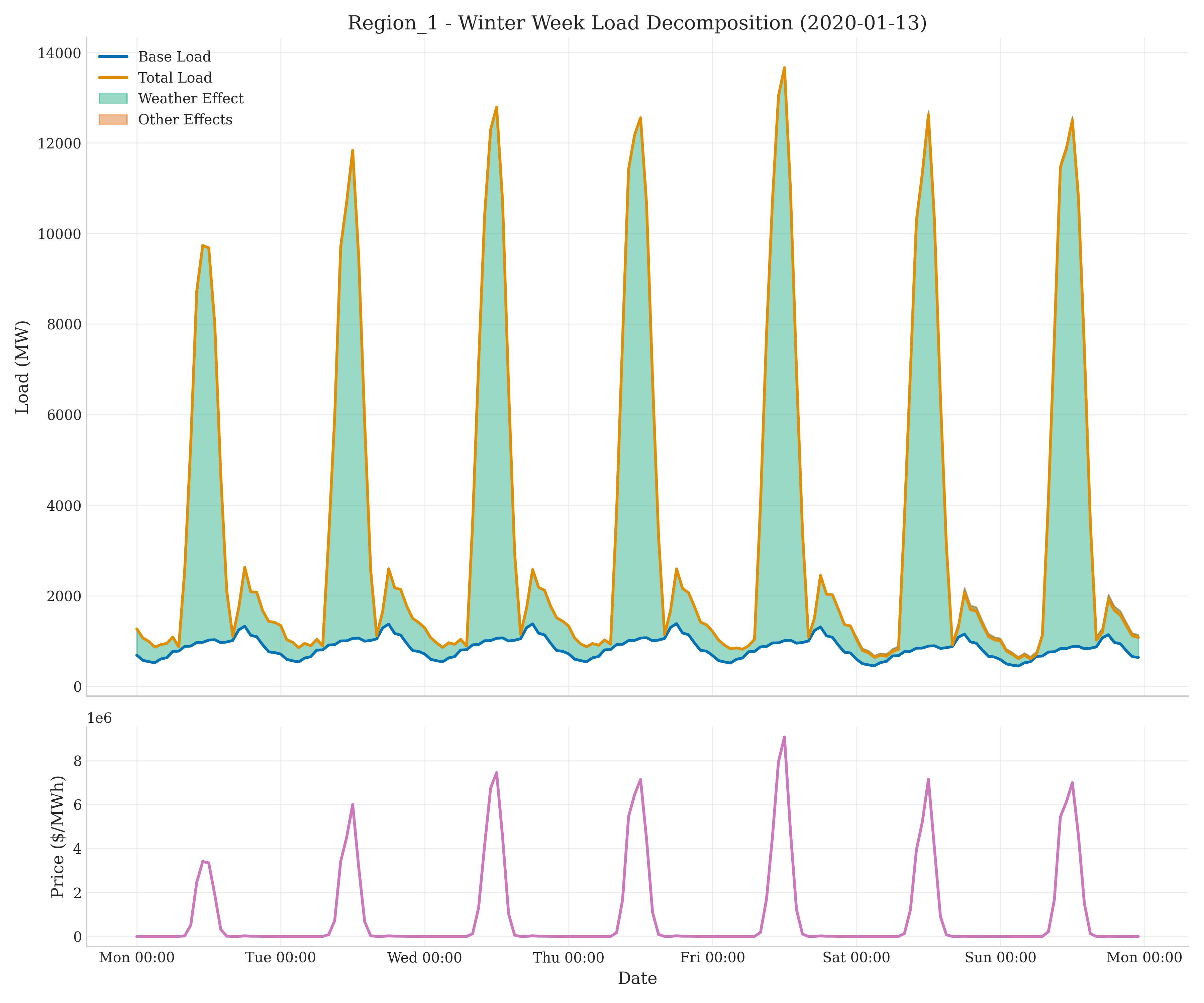}
  \caption{Region 1: winter.}\label{fig:region1_winter}
\end{subfigure}
\par\medskip
\begin{subfigure}[b]{0.48\linewidth}
  \centering
  \includegraphics[width=\linewidth]{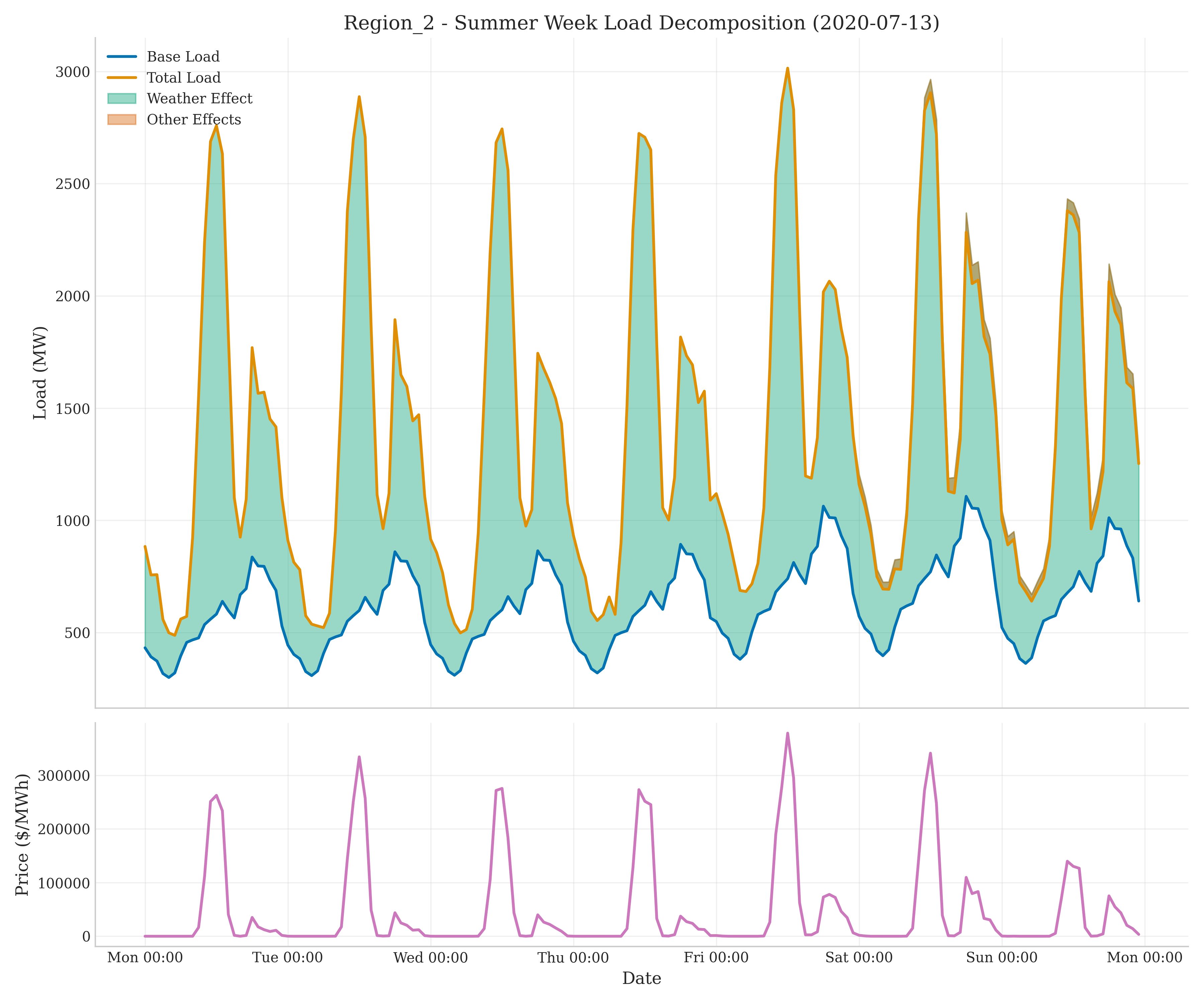}
  \caption{Region 2: summer.}\label{fig:region2_summer}
\end{subfigure}\hfill
\begin{subfigure}[b]{0.48\linewidth}
  \centering
  \includegraphics[width=\linewidth]{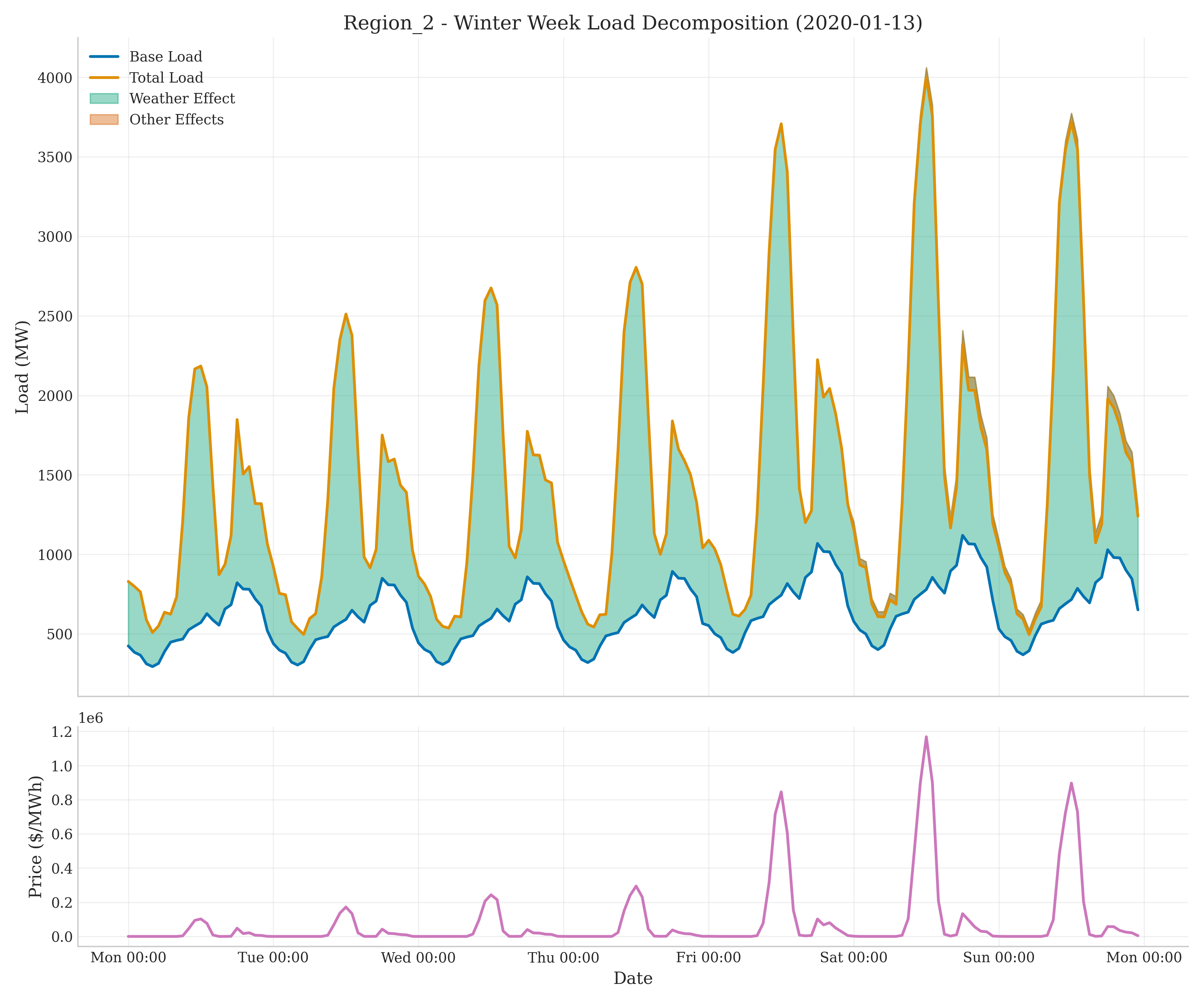}
  \caption{Region 2: winter.}\label{fig:region2_winter}
\end{subfigure}
\caption{Original electricity-generation illustrations: weekly load components and auxiliary prices. Region 1 is shown above Region 2; summer and winter appear in the left and right columns.}
\label{fig:region_components}
\end{figure}

Each synthetic benchmark is one fixed exported realization. Repeating optimization with seeds 2026, 2027, and 2028 does not create independent realizations of the data-generating process. Appendix~\ref{app:exp-data} specifies the forecasting information boundary, and Appendix~\ref{app:exp-protocol} gives the evaluation protocol.

\section{Datasets and Information Availability}\label{app:exp-data}
Appendix~\ref{sec:synthetic} describes the synthetic generators. The evaluations use their fixed exported realizations, not newly generated data. Sales2 post-sales inventory and Synthetic Electricity price and construction components are excluded from predictors; the Electricity target retains the exported component sum without a final added observation-noise term. Designated synthetic future inputs are oracle-exogenous. Here we document real-world data and the common input boundary; Table~\ref{tab:dataset-stats} gives dimensions and rolling tasks. No results from the earlier evaluation design are mixed into the present tables.

\subsection{Real-world panels}\label{sec:realdata}
\paragraph{EPF.}
We combine BE, DE, FR, NP, and PJM market prices on their common hourly span, 2013-01-01 through 2016-12-31. All targets and admitted inputs are observed; genuine zero and negative prices are retained. The two forecast slots contain generation and system-load forecasts for BE/FR, wind and Amprion-load forecasts for DE, grid-load and wind forecasts for NP, and system-load and COMED-load forecasts for PJM. These are harmonized market-specific slots, not two physically identical covariates across markets. We assume the published forecasts are available for the day-ahead horizon; static files cannot verify issue-time revisions. EPF tests real covariate-informed forecasting at a 24-hour horizon, but covers only five markets and one held-out year.

\paragraph{Rossmann-100.}
We use \texttt{Sales} in the labeled training file of the Rossmann Store Sales competition, with store and date keys. The full file has 1,115 stores and 1,017,209 rows. From fitting-period store IDs, we sample 100 stores without replacement using NumPy PCG64(2026), retaining their daily calendars from 2013-01-01 through 2015-07-31. The 48-day horizon follows the competition's test-span length, but our scores use earlier labeled rolling holdouts, not its unlabeled test set. This subset tests long-horizon retail forecasting at manageable cost; it is neither a full-panel result nor an unseen-store test.

\paragraph{Favorita-198.}
We use the store--family Store Sales competition, not the item-level Favorita task. Sampling from fitting-period store IDs with a separately reset PCG64(2026) selects stores 2, 9, 20, 26, 34, and 42. Each retains all 33 fitting-observed families, yielding 198 daily series from 2013-01-01 through 2017-08-15. The horizon is 16 days, matching the measured competition test dates. Four absent Christmas dates remain missing, not zero: the selected calendar has 792 missing item-days, including 594 in fitting. The dataset adds intermittent sales and family heterogeneity, but contains only six store clusters, not 198 independent stores.

The retail samples are fixed independently of model scores. Neither favorable results nor undefined MASE scales trigger resampling. Rossmann-100 has 0/594/2,718 missing item-days in fitting/validation/test; Favorita-198 has 594/0/198. These are missing calendar entries, not counts of repeated forecast leads. M5 and Illness are outside the present evaluation; no historical results for them are used.

\subsection{Permitted and consumed features}\label{app:exp-features}
Table~\ref{tab:exp-features} lists the future-known dynamic schema. Item identity is additionally permitted where supported; retail Stage I uses store and, for Favorita, family categories. Categorical identities remain strings in canonical data and are encoded using fitting data only.

\begin{table}[htbp]
\centering\small
\caption{Dynamic future-input schema. Synthetic entries use oracle information. Counts exclude static identity and one-hot expansion.}
\label{tab:exp-features}
\begin{tabularx}{\linewidth}{lrX}
\toprule Dataset & Count & Features \\
\midrule
EPF & 8 & Two market-specific forecast slots; year, month, day, weekday, weekend indicator, hour. \\
Sales1 & 8 & Promotion, temperature, price; year, month, day, weekday, weekend indicator. \\
Sales2 & 11 & Promotion, competitor promotion, temperature, price, holiday indicator, marketing spend; the five Sales1 calendar fields. \\
Syn. Electricity & 13 & Approximate temperature, humidity, wind, solar; hour, weekday, month, weekend, holiday, daylight-saving indicator, planned outage, year, day. \\
Both retail panels & 6 & Year, month, day of month, day of year, weekday, weekend indicator. \\
\bottomrule
\end{tabularx}
\end{table}

Rossmann promotions, openings, and customers, and Favorita promotions and transactions, are historical-only under the declared boundary. Retrospective CSV entries do not prove prior availability. Oil, holiday/school schedules without verified publication timing, undated store attributes, and unannounced-event annotations are excluded from the initial future schema. Operationally, completed-day records are assumed available at the next daily issuance. This is an explicit availability convention, not an audit of historical reporting latency.

The reported retail adapters use target history and their supported calendar/static inputs; they do \emph{not} consume the retained realized operational columns. HopFormer's retail Stage I uses only calendar and static identity, consistently in fitting and prediction. Missingness flags and elapsed-time indicators are retained in the bundle, but are not automatically consumed by native models. All models obey the same permitted information boundary; this does not imply identical feature consumption. No future target-observation mask is supplied as a predictor.

\subsection{Synthetic sales forecasting results}\label{app:exp-sales}
Table~\ref{tab:sales-results} gives every active model on Sales1 and Sales2. Moving these columns to the appendix was decided after observing results, including HopFormer's weaker performance; it changes neither evaluation scope nor inference. TSMixerX leads both Sales1 metrics and Sales2 MASE. NBEATSx leads Sales2 MAPE. These are fixed synthetic realizations with oracle inputs, not evidence about every realization of their generators.

\begin{table}[htbp]
\centering
\caption{Complete synthetic-sales comparison under the same protocol as Table~\ref{tab:combined_main_results}. MAPE is a decimal ratio. Bold indicates the lowest unrounded mean among these methods. Both datasets remain in all multiplicity families and summary conclusions.}
\label{tab:sales-results}
\small
\setlength{\tabcolsep}{3pt}\renewcommand{\arraystretch}{1.08}
\begin{tabular}{lrrrr}
\toprule
Model & \multicolumn{2}{c}{Sales1} & \multicolumn{2}{c}{Sales2} \\
\cmidrule(lr){2-3}\cmidrule(lr){4-5}
 & MASE & MAPE & MASE & MAPE \\
\midrule
HopFormer & 0.474 & 0.619 & 0.313 & 0.394 \\
\midrule
TimeXer & 0.439 & 0.579 & 0.208 & 0.295 \\
TimeMixer & 0.434 & 0.538 & 0.208 & 0.274 \\
TFT & 0.493 & 0.658 & 0.249 & 0.270 \\
TiDE & 0.499 & 0.594 & 0.213 & 0.297 \\
TSMixerX & \textbf{0.288} & \textbf{0.298} & \textbf{0.184} & 0.248 \\
NBEATSx & 0.333 & 0.412 & 0.211 & \textbf{0.238} \\
NHITS & 0.536 & 0.447 & 0.197 & 0.280 \\
PatchTST & 0.438 & 0.550 & 0.204 & 0.252 \\
iTransformer & 0.447 & 0.512 & 0.215 & 0.280 \\
DLinear & 0.450 & 0.548 & 0.218 & 0.286 \\
\bottomrule
\end{tabular}
\end{table}

\section{Evaluation and Baseline Configuration}\label{app:exp-protocol}
\subsection{Chronological splits and rolling tasks}
Table~\ref{tab:exp-splits} gives inclusive dates. Validation and test intervals denote unions of target timestamps. The three target intervals are disjoint, although horizons overlap between origins within validation or test. In the notation of Section~\ref{sec:method_overview}, $t$ is the first predicted timestamp and history ends at $t-1$; the implementation stores $t-1$ as the issuance origin.

\begin{table}[htbp]
\centering\scriptsize\setlength{\tabcolsep}{3pt}
\caption{Inclusive outer target intervals. Hourly endpoints include 23:00 of the listed final date. Windows are sequence-training examples per epoch, not independent observations.}
\label{tab:exp-splits}
\begin{tabular}{llllr}
\toprule Dataset & Fitting $F$ & Validation $V$ & Test & Windows \\
\midrule
EPF & 2013-01-01--2015-12-03 & 2015-12-04--2015-12-31 & 2016-01-01--2016-12-31 & 5,185 \\
Sales1/2 & 2020-01-01--2021-10-16 & 2021-10-17--2022-01-01 & 2022-01-02--2022-12-31 & 52,600 \\
Syn. Electricity & 2020-01-01--2020-12-02 & 2020-12-03--2020-12-30 & 2020-12-31--2021-12-30 & 1,535 \\
Rossmann-100 & 2013-01-01--2014-04-27 & 2014-04-28--2014-08-02 & 2014-08-03--2015-07-31 & 24,300 \\
Favorita-198 & 2013-01-01--2016-06-17 & 2016-06-18--2016-08-21 & 2016-08-22--2017-08-15 & 25,344 \\
\bottomrule
\end{tabular}
\end{table}

Starting at each interval's first target date, forecast starts advance by $\Delta$ in Table~\ref{tab:dataset-stats}; each predicts $H$ consecutive steps. Horizons and origin lists are identical across models within a dataset. They differ across datasets to reflect frequency, task horizon, and available history; the exact grids and split dates are our design, not a universal benchmark convention.

Sequence-training origins require the largest candidate history and an entirely observed $H$-step target inside $F$. This common support prevents shorter contexts from receiving extra training examples. EPF and Synthetic Electricity use daily training issuance; Sales1/2 and Rossmann use all eligible daily starts. Favorita has 852 eligible starts after missing-label filtering; we partition the sorted list into 128 consecutive bins with \texttt{numpy.array\_split}, sample one start per bin with a fresh PCG64(2026), and reuse these 128 starts for every series, context, seed, and epoch. Epoch shuffling changes order, not membership. This subsamples overlapping examples, not the daily calendar. Pointwise experts and classical methods retain their native fitting structures.

Model parameters and learned preprocessing remain fixed at test time, without refitting on $F\cup V$. Later origins may use targets observed since an earlier origin. Forecasts are issued jointly for all $H$ leads, never with actual future targets substituted inside the horizon. Prediction batches are grouped by origin. TimesFM's native in-context regression is an explicit exception to a blanket ban on fitting at inference: it uses only the available context, while neural weights remain frozen.

\subsection{Missingness, scaling, and scoring population}\label{app:exp-cleaning}
We retain the complete calendar, reject duplicate item/date keys, preserve genuine zeros, and never impute training labels or test truth. Missing historical numerical inputs are forward-filled causally; leading unknowns use a zero placeholder. Original observation flags remain separate. No backward filling or future-informed interpolation is used. The four non-retail panels are complete. In retail, pointwise training uses observed labels and sequence training requires fully observed target horizons.

The shared supervised-neural view uses fitting-only per-item target means and population SDs, and feature-wise covariate means/SDs pooled over fitting rows. Constant normalization scales are set to one, solely for numerical transformation. Native models may additionally apply causal normalization such as RevIN. Frozen foundation models use their native raw-input pathways; LightGBM uses raw tabular features. HopFormer and its matched Small controls use the $E$-only transform in Appendix~\ref{sec:pseudocode}. These are intentional model-specific numerical interfaces, not additional information.

Score inverse-transformed predictions in original units. Apply $\max(0,\widehat Y)$ to every retail model's final forecast, saving the unclipped forecast too; for HopFormer this occurs after trend--residual reconstruction. No final nonnegative clipping is applied on the other four panels.

MASE eligibility is fixed by observed seasonal pairs in all of $F$: it includes 5/5 EPF series, 199/200 Sales1 stores, 200/200 Sales2 stores, 5/5 Synthetic Electricity regions, 100/100 Rossmann stores, and 189/198 Favorita series. Sales1's Store\_73 has a constant fitting target. Undefined-scale series remain in prediction and retain any valid MAPE support. Eligibility depends on truth and fitting scales, never model outputs. Every expected item/origin/lead requires a finite prediction; missing model outputs fail the audit rather than being removed by an inner join.

\subsection{Validation, seeds, and optimization budgets}\label{app:exp-tuning}
For trained neural models, the complete seed-2026 search is the dataset's context grid crossed with learning rates $10^{-4}$ and $3\times10^{-4}$, with fixed architecture. Select the smallest validation MASE, breaking ties by smaller context and then preset index. Keep that selected seed-2026 fit, train the selected configuration with seeds 2027 and 2028, and select each seed's checkpoint on validation MASE. Do not retune context or learning rate per seed or test window. Frozen forecasters search context without updating weights. MOIRAI uses three predictive-sampling seeds; deterministic frozen methods and SeasonalNaive run once. LightGBM uses three training seeds.

The common neural controller uses batch size 32, no dropped final batch, gradient clipping at 1, at least 10 epochs, patience 8, and relative patience improvement $10^{-4}$. The absolute minimum validation-MASE checkpoint is tracked separately from the relative threshold that resets patience. Scratch models use Adam; Small-backbone LoRA uses AdamW with weight decay 0.01 and no warmup. Native training losses are retained: MSE for TimeXer, TimeMixer, PatchTST, iTransformer, DLinear, and TiDE; MAE for NHITS, TSMixerX, and NBEATSx; quantile losses for TFT and Chronos-Bolt adaptation. TFT and Bolt use their median quantile as the point forecast. Thus, common model selection does not mean identical training objectives.

The default maximum is 100 epochs, with continuation beyond an initial 50-epoch checkpoint when patience is unresolved. Sales1 NBEATSx and Synthetic Electricity TimeMixer, TiDE, TSMixerX, and iTransformer have approved 300-epoch complete-cell budgets. A documented amendment also allowed a new 300-epoch complete-cell campaign for specified retail baselines or Electricity TFT, conditional on a verified 100-epoch stopping hold. Extensions preserve all candidates, seeds, losses, and patience rules; exhausted unresolved fits are not admitted by simply relabeling them complete. HopFormer's residual LoRA ceiling remains 100. Expert and aggregation-solver amendments are listed in Appendix~\ref{app:exp-experts}. These are unequal, documented optimization budgets, not equal-compute comparisons.

\begin{table}[htbp]
\centering
\caption{Validation-selected context lengths, fixed across test origins and replicate seeds. P denotes a pointwise trend aggregator; S denotes the native seasonal-history rule. Dashes are unavailable additional-baseline evaluations, not incomplete active cells.}
\label{tab:exp-contexts}
\scriptsize
\setlength{\tabcolsep}{2.5pt}\renewcommand{\arraystretch}{1.08}
\begin{tabular}{lrrrrrr}
\toprule
Model & EPF & Sales1 & Sales2 & Syn. Elec. & Rossmann & Favorita \\
\midrule
HopFormer & 720 & 365 & 365 & 720 & 192 & 365 \\
Residual backbone frozen & 720 & 365 & 365 & 720 & 192 & 365 \\
Direct backbone + LoRA & 336 & 365 & 365 & 96 & 96 & 365 \\
Direct backbone frozen & 720 & 112 & 365 & 720 & 192 & 365 \\
SPA trend only & P & P & P & P & P & P \\
Equal-weight trend & P & P & P & P & P & P \\
Best-expert trend & P & P & P & P & P & P \\
ElasticNet trend & P & P & P & P & P & P \\
TimeXer & 96 & 112 & 365 & 168 & 96 & 365 \\
TimeMixer & 168 & 112 & 365 & 336 & 48 & 64 \\
TFT & 336 & 112 & 56 & 336 & 192 & 32 \\
TiDE & 336 & 56 & 365 & 96 & 48 & 32 \\
TSMixerX & 96 & 365 & 365 & 96 & 192 & 365 \\
NBEATSx & 168 & 112 & 365 & 336 & 96 & 32 \\
NHITS & 96 & 112 & 365 & 96 & 48 & 128 \\
PatchTST & 168 & 112 & 365 & 336 & 96 & 64 \\
iTransformer & 96 & 112 & 365 & 96 & 96 & 64 \\
DLinear & 96 & 365 & 365 & 336 & 96 & 64 \\
Chronos-2 & 720 & 365 & 365 & 720 & -- & -- \\
Chronos-Bolt Base & 336 & 365 & 365 & 720 & 192 & 365 \\
MOIRAI 1.1-R Base & 336 & -- & -- & -- & -- & -- \\
TimesFM 2.5 XReg & 720 & -- & -- & -- & -- & -- \\
LightGBM & 96 & -- & -- & -- & -- & -- \\
SeasonalNaive & S & -- & -- & -- & -- & -- \\
\bottomrule
\end{tabular}
\end{table}

\subsection{Baseline identities and input interfaces}\label{app:exp-baselines}
The active comparison contains ten neural baselines: exogenous architectures and temporal models with differing native covariate interfaces. Additional completed frozen-foundation, tabular, and seasonal results are retained separately below. Table~\ref{tab:exp-native} describes the measured adapters, not every capability of each original architecture. In particular, absence of future inputs in an adapter must not be interpreted as proving that such inputs are unhelpful.

\begin{table}[htbp]
\centering\small
\caption{Measured neural baseline configurations. Past/future covariates refer only to the admitted schema. All models use target history.}
\label{tab:exp-native}
\begin{tabularx}{\linewidth}{lX}
\toprule Model & Input use and fixed architecture \\
\midrule
TimeXer & Historical covariates, no future covariate injection; patch 8, width/feed-forward 512, 8 heads, 2 layers, dropout 0.1. \\
TimeMixer & Target-only; width/feed-forward 512, 2 layers, two factor-2 average-downsampling stages, channel-independent, moving average 25, dropout 0.1. \\
TFT & Past/future known covariates and item identity; hidden 160, 4 heads, stack 1, dropout 0.1; quantiles 0.1/0.5/0.9. \\
TiDE & Past/future known covariates and item embedding; encoder 1024/1024, decoder output 8, time encoder 64/4, final decoder 64, dropout 0.5; no layer norm. \\
TSMixerX & NeuralForecast extended TSMixer, not TimeMixer or PatchTSMixer; past/future known inputs and one-hot identity; 2 blocks, feed-forward 64, dropout 0, RevIN. \\
NBEATSx & Original identity and exogenous-WaveNet bases; known covariates and identity; hidden 256/256, 8 exogenous channels, softplus, zero dropout. \\
NHITS & Known covariates and identity; three identity stacks, one block each, hidden 512/512, max pools 2/2/1, frequency downsampling 4/2/1, linear interpolation. \\
PatchTST & Supervised target-only model; patch 16, stride 8, 3 layers, 16 heads, width 128, feed-forward 256, dropout 0.2; RevIN, no decomposition. \\
iTransformer & Historical covariates, no future injection; width/feed-forward 512, 3 layers, 8 heads, dropout 0.1. \\
DLinear & Target-only native trend/seasonal decomposition; shared scalar model, not separate fitted item parameters. \\
\bottomrule
\end{tabularx}
\end{table}

TimeXer, TimeMixer, PatchTST, iTransformer, DLinear, and NBEATSx use their source architectures. NHITS and TSMixerX use NeuralForecast implementations with the shared outer controller. TimeXer and TimeMixer receive deterministic left-repeat padding to multiples of 8 and 4, respectively, without additional history. TFT and TiDE retain their native architectures behind the common input and selection adapters. This is a budget-controlled implementation comparison, not a reproduction of every original paper's tuning budget.

\paragraph{Frozen pretrained baselines.}
Chronos-2 consumes admitted past/future known covariates, uses its median, and disables cross-learning. Chronos-Bolt Base is target-only and uses its native median; it is deliberately distinguished from the matched Small controls. MOIRAI 1.1-R Base uses patch size 32 and the median of 100 samples; its three seeds control predictive sampling, with the origin added to the seed. TimesFM 2.5 uses its 200M PyTorch model with native \texttt{xreg + timesfm}, ridge 0, per-input target normalization, and CPU regression with GPU neural inference. MOIRAI and TimesFM have admitted EPF results only. Pretraining-data overlap is unresolved.

\begin{table}[htbp]
\centering\scriptsize
\caption{Pinned Hugging Face checkpoint identities. Revisions, not moving default branches, identify the weights.}
\label{tab:exp-checkpoints}
\begin{tabular}{ll}
\toprule Repository & Revision \\
\midrule
\texttt{amazon/chronos-bolt-small} & \texttt{772f3d25d38aec6d914c8949dab4462e2d46f5d8} \\
\texttt{amazon/chronos-bolt-base} & \texttt{5d9f166d69f47aef3401367a7b842e78fe97b121} \\
\texttt{amazon/chronos-2} & \texttt{29ec3766d36d6f73f0696f85560a422f50e8498c} \\
\texttt{Salesforce/moirai-1.1-R-base} & \texttt{a0358780de4cdd7ac7963b12ae9985d225d55026} \\
\texttt{google/timesfm-2.5-200m-pytorch} & \texttt{1d952420fba87f3c6dee4f240de0f1a0fbc790e3} \\
\bottomrule
\end{tabular}
\end{table}

\paragraph{Tabular and statistical controls.}
EPF LightGBM fits 24 direct horizon-specific regressors on target/covariate histories, lead-specific known inputs, and market identity. It searches 31/63 leaves and the common contexts; learning rate is 0.05, minimum leaf size 20, feature/row fractions 0.8, maximum 1,000 rounds, and patience 50, selecting by validation MASE. SeasonalNaive repeats the latest observed 24-hour cycle on EPF. AutoARIMA was specified with season length 24, stepwise AICc selection and fixed-coefficient rolling state updates, but unresolved convergence prevents an admitted result. TimesNet was deferred for resource reasons. Neither method is represented by old or incomplete scores; ETS and other historical-paper entries are also not carried forward.

\subsection{Additional completed comparisons}\label{app:exp-additional}
The main display was narrowed after results had been inspected. To avoid discarding evidence, Table~\ref{tab:additional-results} retains all 14 completed cells from the omitted rows; their variability and adjusted comparisons remain in Appendix~\ref{app:exp-uncertainty}. Chronos-2 has lower EPF MASE and MAPE than HopFormer. On Synthetic Electricity, HopFormer improves on Chronos-2 by 47.2\% in MASE and 35.2\% in MAPE, computed from unrounded means. These are numerical comparisons, not a claim that HopFormer dominates the wider model set. Chronos-Bolt Base is distinct from the Small checkpoint used in the matched controls.

\begin{table}[htbp]
\centering
\caption{Additional completed external-baseline results retained after the displayed model set was narrowed. MAPE is a decimal ratio. Dashes mean no audited evaluation, not zero, equality, or a pending value to be imputed. These comparisons remain in the original 252-slot inference family.}
\label{tab:additional-results}
\fontsize{7}{8.5}\selectfont
\setlength{\tabcolsep}{1.7pt}\renewcommand{\arraystretch}{1.08}
\resizebox{\linewidth}{!}{%
\begin{tabular}{lrrrrrrrrrrrr}
\toprule
Model & \multicolumn{2}{c}{EPF} & \multicolumn{2}{c}{Sales1} & \multicolumn{2}{c}{Sales2} & \multicolumn{2}{c}{Syn. Electricity} & \multicolumn{2}{c}{Rossmann-100} & \multicolumn{2}{c}{Favorita-198} \\
\cmidrule(lr){2-3}\cmidrule(lr){4-5}\cmidrule(lr){6-7}\cmidrule(lr){8-9}\cmidrule(lr){10-11}\cmidrule(lr){12-13}
 & MASE & MAPE & MASE & MAPE & MASE & MAPE & MASE & MAPE & MASE & MAPE & MASE & MAPE \\
\midrule
Chronos-2 & 0.460 & 0.305 & 0.320 & 0.359 & 0.214 & 0.243 & 0.332 & 0.087 & -- & -- & -- & -- \\
Chronos-Bolt Base & 0.557 & 0.436 & 0.553 & 0.702 & 0.472 & 0.450 & 0.495 & 0.100 & 0.783 & 0.211 & 1.748 & 0.425 \\
MOIRAI 1.1-R Base & 0.656 & 0.505 & -- & -- & -- & -- & -- & -- & -- & -- & -- & -- \\
TimesFM 2.5 XReg & 0.603 & 0.417 & -- & -- & -- & -- & -- & -- & -- & -- & -- & -- \\
\midrule
LightGBM & 0.614 & 0.362 & -- & -- & -- & -- & -- & -- & -- & -- & -- & -- \\
SeasonalNaive & 0.832 & 0.612 & -- & -- & -- & -- & -- & -- & -- & -- & -- & -- \\
\bottomrule
\end{tabular}}
\end{table}

\paragraph{Robustness to context length and forecast horizon.}
Figure~\ref{fig:context_and_horizon} presents a separate zero-shot sensitivity analysis on the three synthetic datasets. The top row varies the context length $L\in\{32,64,96,128,256,512\}$ with $H=24$, while the bottom row varies the forecast horizon $H\in\{24,48,72,96,120\}$ with $L=256$. Across the context sweep, the curves generally favor HopFormer, with the clearest separation on Sales2; the Sales1 curves nearly overlap for $L\ge96$. Longer contexts generally reduce MASE, although the improvement is not monotone. Across the horizon sweep, HopFormer maintains lower MASE on Sales2 and its error changes relatively little as the horizon increases. Differences on Sales1 are smaller, whereas Electricity exhibits a crossover at $H=96$, where Chronos has lower error. These curves therefore indicate dataset-dependent sensitivity, rather than an advantage at every horizon.

\begin{figure*}[t]
  \centering

  \includegraphics[width=0.9\textwidth]{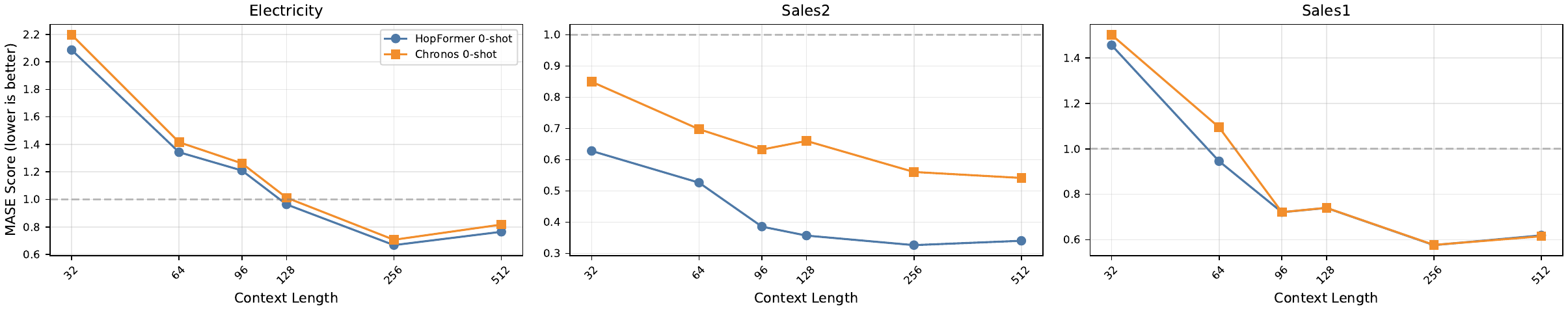}

  \vspace{0.8em}

  \includegraphics[width=0.9\textwidth]{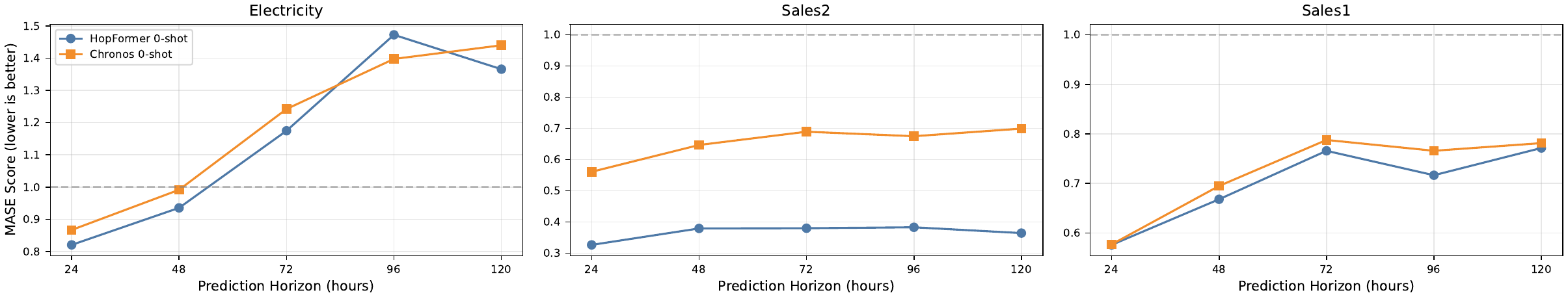}

  \caption{
    (Top) Effect of context length on model performance across datasets (prediction length = 24).
    (Bottom) Effect of prediction horizon on model performance across datasets (context length = 256).
  }
  \label{fig:context_and_horizon}
\end{figure*}

\section{Implemented HopFormer Pipeline}\label{sec:pseudocode}
\subsection{Stage I, normalization, and aggregation}\label{app:exp-experts}
Divide $F$ chronologically by complete days: $E$ contains the first $\lfloor0.70D\rfloor$ days, $A$ continues through day $\lfloor0.85D\rfloor$, and $B$ contains the remainder, where $D$ is the number of fitting days. Fit each item's mean $\mu_i$ and population SD $\sigma_i$ on observed $E$ targets; use $\sigma_i=1$ for a constant item. Work in $u_{i,t}=(Y_{i,t}-\mu_i)/\sigma_i$. This normalization safeguard does not replace an undefined MASE scale.

The eight experts in Table~\ref{tab:exp-experts} are pooled across items, with permitted covariates and identity. For iterative experts, use the last 10\% of $E$ as a chronological inner stopping holdout, with feature preprocessing fitted on the earlier part. Refit the expert and its feature preprocessing on all observed $E$ at the selected iteration count. Linear/MLP inputs are numerically standardized; categories are one-hot encoded except for native LightGBM/CatBoost categorical handling. Unseen categories use the reserved unknown/all-zero policy. The target transform remains $E$-only throughout.

\begin{table}[htbp]
\centering\small
\caption{Fixed expert pool. Settings describe the Stage-I experts, not the separate forecasting baseline or aggregation control with the same family name.}
\label{tab:exp-experts}
\begin{tabularx}{\linewidth}{lX}
\toprule Expert & Preset \\
\midrule
Ridge & Penalty 1, intercept, SVD solver. \\
ElasticNet & Penalty 0.001, $\ell_1$ ratio 0.5, intercept, cyclic solver, tolerance $10^{-4}$, maximum 10,000 iterations. \\
Random Forest & 300 trees, maximum depth 12, minimum leaf 10, all features, bootstrap sampling. \\
Extra Trees & Same tree/depth/leaf settings, randomized splits, no bootstrap sampling. \\
LightGBM & Squared loss, learning rate 0.05, 31 leaves, minimum leaf 20, $\ell_2$ penalty 1, patience 50. \\
XGBoost & Squared loss, histogram trees, learning rate 0.05, depth 6, row/feature fractions 0.8, $\ell_2$ penalty 1, patience 50. \\
CatBoost & RMSE, learning rate 0.05, depth 6, $\ell_2$ penalty 3, patience 50. \\
MLP & ReLU widths 128/64, no dropout, MSE, AdamW at 0.001, weight decay $10^{-4}$, batch 1,024, minimum 10/maximum 200 epochs, patience 20. \\
\bottomrule
\end{tabularx}
\end{table}

Boosting ceilings are normally 1,000 rounds. Approved stopping-budget amendments increase Sales1 XGBoost to 12,000 and CatBoost to 3,000, and Sales2/Synthetic Electricity XGBoost to 3,000, consistently across their seeds and shared-bank variants. These changes address unresolved development stopping, not selection of favorable test scores. An unresolved expert holds the bank; no expert is silently dropped. The separate nonnegative ElasticNet aggregation solver has a 1,000-iteration default and an approved 10,000-iteration ceiling on Synthetic Electricity, with tolerance unchanged at $10^{-4}$.

For every pattern $p\in\{0,1\}^8$, fit minimum-norm least-squares coefficients on the selected expert predictions in $E$, without an extra intercept or ridge term. The empty pattern predicts zero in normalized units. Set $C=\max(1,\max_E|u|)$ and clip \emph{each pattern} prediction to $[-C,C]$ before aggregation. Compute its mean squared calibration loss on observed $A$ targets and use
\[
\widehat\rho_{\lambda,\alpha}(p)\ \propto\
\frac{\exp\{-\alpha|p|-\lambda\widehat L_A(p)\}}{\binom{8}{|p|}},
\qquad
\widehat m(z)=\sum_p\widehat\rho_{\lambda,\alpha}(p)\widehat m_p(z).
\]
Pattern probabilities are nonnegative and sum to one; within-pattern coefficients need not. A single effective expert-weight vector is not equivalent when clipping is active. Empirical clipping bounds candidate predictions, not all future true targets.

Choose $\lambda\in\{1,10,100\}$ and $\alpha\in\{0,1\}$ on $B$ using pointwise macro MASE with $E$-only seasonal scales and original-unit reconstruction; retail also applies its common nonnegative final clip. Ties prefer larger $\alpha$, then smaller $\lambda$. This internal criterion is distinct from outer rolling validation. Seeds 2027/2028 reuse seed-2026 SPA hyperparameters and the selected single-expert identity, but refit their expert banks and aggregation coefficients. Within each seed, all contexts and compatible ablations share that bank. Stage I is never refitted on all $F$ or on $F\cup V$.

\subsection{Residual adaptation and reconstruction}\label{app:exp-residual}
With Stage I frozen, form residuals over eligible fitting histories and labels. Stage II uses all admissible $F$ sequence windows, including dates in $E$, $A$, and $B$; residualization is therefore in-sample. Chronos-Bolt Small receives only the $L$ residual-history values, not a separate covariate tensor. Its native context-dependent normalization remains causal. LoRA adapts attention query/key/value/output projections with rank 8, alpha 16 (multiplier $16/8$), dropout 0.05, and no bias adaptation, giving 589,824 trainable parameters. Backbone weights remain frozen. Training uses float32, deterministic PyTorch operations, TF32 disabled, native quantile loss, and masking of padded labels beyond $H$.

The median residual forecast is reconstructed once:
\[
\widehat Y_{i,t:t+H-1}
=\mu_i+\sigma_i\left\{
\widehat m(Z_{i,t:t+H-1})+
f_{\widehat w}\!\left(u_{i,t-L:t-1}-\widehat m(Z_{i,t-L:t-1})\right)
\right\}.
\]
The matched direct controls use the same checkpoint and $E$-only target transform, but zero trend. They isolate residualization more closely than the external Bolt Base baseline. Retail clipping occurs only after the displayed reconstruction.

Algorithms~\ref{alg:training} and~\ref{alg:prediction} describe this implemented estimator. It differs from Algorithm~\ref{alg:chrono}: residual fitting is not restricted to a disjoint, buffered region and uses native quantile rather than squared training loss. We do not treat the theory as a certificate for the empirical estimator or infer that residuals must have reduced variance.

\begin{algorithm}[htbp]
\caption{Implemented in-sample HopFormer fitting}\label{alg:training}
\begin{algorithmic}[1]
\REQUIRE Frozen fitting/validation split, masks, context grid, seeds, horizon, expert presets, and pinned Small checkpoint
\STATE Split $F$ into whole-day $E/A/B$ blocks; estimate item target transforms on observed $E$
\STATE At seed 2026, fit and check all eight experts on $E$ using the inner stopping holdout
\STATE Fit all pattern coefficients on $E$; calibrate clipped patterns on $A$; select SPA settings on $B$
\STATE Freeze Stage I; form fitting residual inputs and observed residual labels using that fitted state
\STATE Train every context/learning-rate candidate on common $F$ windows; select by rolling validation MASE
\STATE Repeat expert fitting and the selected residual configuration at seeds 2027/2028, keeping seed-2026 hyperparameter choices fixed
\STATE Verify stopping, checkpoint hashes, frozen backbone, complete validation search, and reloaded validation predictions for every required fit
\STATE Freeze the admitted model states and preprocessing before loading test targets for evaluation
\end{algorithmic}
\end{algorithm}

\begin{algorithm}[htbp]
\caption{Fixed-model rolling prediction}\label{alg:prediction}
\begin{algorithmic}[1]
\REQUIRE Fitted seed-specific states; origin $t-1$; permitted history; known horizon covariates
\STATE Build exactly $L$ causal historical inputs, preserving the original observation mask separately
\STATE Apply the fixed $E$-only target transform and subtract the Stage-I historical trend
\STATE Predict $H$ residual medians with the selected Small backbone and adapter
\STATE Add the Stage-I future trend and invert the target transform once
\STATE For retail only, save the unclipped forecast and clip the reconstructed forecast at zero
\STATE Save every item/origin/lead forecast and independently audit keys, truth, support, and metrics
\end{algorithmic}
\end{algorithm}

\section{Matched Controls and Aggregation Ablations}\label{app:exp-ablations}
Table~\ref{tab:exp-ablation-results} expands the main-text MASE table with both metrics on all six datasets. It is separate from the external-baseline comparison. All temporal controls use the same pinned Chronos-Bolt Small checkpoint, $E$-only normalization, common sequence-training support, and validation selection; each selects its own context and, when adapted, learning rate. Thus they compare matched modeling procedures after tuning, not identical contexts or identical compute.

\begin{itemize}
\item \textbf{HopFormer}: fitted SPA trend plus LoRA-adapted residual backbone.
\item \textbf{Residual backbone frozen}: the same fitted SPA trend plus the frozen residual backbone. Stage I is fitted, so the whole system is not zero-shot.
\item \textbf{Direct backbone + LoRA}: the backbone adapted directly on target histories, without SPA or residualization.
\item \textbf{Direct backbone frozen}: the frozen backbone applied to target histories, without SPA. Context is validation-selected; frozen weights do not imply absence of development-data use.
\item \textbf{SPA trend only}: fitted SPA without a residual forecaster.
\item \textbf{Equal-weight trend}: arithmetic mean of the same fitted experts, without a residual forecaster.
\item \textbf{Best-expert trend}: one expert selected on $B$ MASE; the seed-2026 identity is reused across replicate seeds.
\item \textbf{ElasticNet trend}: nonnegative ElasticNet on expert predictions and normalized targets in $E$, penalty 0.001, $\ell_1$ ratio 0.5, no intercept or coefficient renormalization. This is neither pure Lasso nor the ElasticNet covariate expert.
\end{itemize}

\begin{table}[htbp]
\centering
\caption{Complete matched-backbone and trend-aggregation results. MAPE is a decimal ratio. All temporal controls use Chronos-Bolt Small. No cell is missing. Bold marks the lowest mean within these rows, not significance.}
\label{tab:exp-ablation-results}
\fontsize{7}{8.5}\selectfont
\setlength{\tabcolsep}{1.7pt}\renewcommand{\arraystretch}{1.08}
\resizebox{\linewidth}{!}{%
\begin{tabular}{lrrrrrrrrrrrr}
\toprule
Model & \multicolumn{2}{c}{EPF} & \multicolumn{2}{c}{Sales1} & \multicolumn{2}{c}{Sales2} & \multicolumn{2}{c}{Syn. Electricity} & \multicolumn{2}{c}{Rossmann-100} & \multicolumn{2}{c}{Favorita-198} \\
\cmidrule(lr){2-3}\cmidrule(lr){4-5}\cmidrule(lr){6-7}\cmidrule(lr){8-9}\cmidrule(lr){10-11}\cmidrule(lr){12-13}
 & MASE & MAPE & MASE & MAPE & MASE & MAPE & MASE & MAPE & MASE & MAPE & MASE & MAPE \\
\midrule
HopFormer & \textbf{0.524} & 0.344 & 0.474 & 0.619 & 0.313 & 0.394 & \textbf{0.175} & \textbf{0.056} & 0.784 & 0.225 & 1.743 & 0.418 \\
Residual backbone frozen & 0.540 & \textbf{0.339} & 0.556 & 0.742 & 0.517 & 0.625 & 0.193 & 0.062 & 0.729 & 0.208 & 1.777 & 0.433 \\
Direct backbone + LoRA & 0.528 & 0.386 & \textbf{0.389} & \textbf{0.408} & \textbf{0.209} & \textbf{0.226} & 0.608 & 0.090 & 0.748 & 0.208 & \textbf{1.682} & \textbf{0.408} \\
Direct backbone frozen & 0.563 & 0.419 & 0.603 & 0.800 & 0.567 & 0.539 & 0.753 & 0.110 & 0.846 & 0.215 & 1.742 & 0.418 \\
\midrule
SPA trend only & 1.300 & 0.568 & 0.800 & 1.145 & 0.604 & 0.716 & 0.474 & 0.134 & \textbf{0.686} & \textbf{0.204} & 4.344 & 0.796 \\
Equal-weight trend & 1.229 & 0.599 & 1.061 & 1.118 & 0.592 & 0.787 & 1.709 & 0.626 & 0.758 & 0.205 & 4.204 & 0.777 \\
Best-expert trend & 1.347 & 0.623 & 0.778 & 1.173 & 0.607 & 0.779 & 0.586 & 0.093 & 0.736 & 0.227 & 4.156 & 0.644 \\
ElasticNet trend & 1.331 & 0.552 & 0.844 & 1.219 & 0.836 & 1.024 & 0.455 & 0.137 & 0.743 & 0.222 & 4.168 & 0.664 \\
\bottomrule
\end{tabular}}
\end{table}

\paragraph{Residualization and adaptation.}
Relative to direct Small-backbone adaptation, residualization improves both mean metrics on EPF and Synthetic Electricity, but worsens both on Sales1, Sales2, Rossmann, and Favorita. The stable MASE advantage on Synthetic Electricity coexists with stable disadvantages on Sales1, Sales2, and Favorita; the EPF gain does not meet the adjusted stability criterion. The Synthetic Electricity MAPE improvement has primary adjusted $p=0.0132$ but fails the longer-block checks. On Favorita, the direct frozen backbone is also slightly better in MASE ($1.742481$ versus $1.742775$), without significant evidence of a difference. It is not appropriate to describe the larger effect against the directly adapted backbone as a benefit of residualization.

Residual adaptation lowers mean MASE on EPF, Sales1, Sales2, Synthetic Electricity, and Favorita, but worsens Rossmann. Stable evidence supports its MASE gains on Sales1, Sales2, and Synthetic Electricity, and its MAPE gains on Sales2, Synthetic Electricity, and Favorita. EPF MAPE slightly worsens. These comparisons retain the fitted Stage-I bank within each seed and change the residual backbone's adaptation, with context selected separately for each procedure.

\paragraph{Aggregation.}
SPA-only has lower mean MASE than the best expert on four of six datasets, but this is not a universal or uniformly significant improvement. It stably beats equal weights on Sales1 MASE and both Synthetic Electricity metrics, and nonnegative ElasticNet on Sales1 MASE and both Sales2 metrics. Conversely, the best expert stably beats SPA on Sales1 MASE; on Favorita, equal weights beat SPA on MASE, while the best expert and ElasticNet beat it on MAPE. Rossmann's SPA-only MASE of $0.686$ is below every temporal method in Tables~\ref{tab:combined_main_results} and~\ref{tab:exp-ablation-results}, but the missing longest-block sensitivity prevents a stable-significance claim.

The expert-only variants share their fitted bank but preserve their respective aggregation procedures. SPA clips each pattern, whereas the alternatives do not apply that clipping. These are therefore comparisons of complete procedures, not sparsity alone. Full fine-tuning, alternative residual backbones, equal/best-expert trends followed by LoRA, horizon sweeps, and repeated synthetic realizations were not evaluated. In particular, no experiment establishes parity between LoRA and full fine-tuning. All six planned ablation contrasts and their sensitivity-adjusted values appear in Appendix~\ref{app:exp-inference}.

\section{Metrics, Variability, and Inferential Scope}\label{app:exp-uncertainty}
\subsection{Metric definitions and aggregation}\label{app:eval-metrics}
For each series, the fixed MASE scale $q_i$ is the mean absolute seasonal difference over observed pairs $(Y_{i,u},Y_{i,u-m})$ entirely in $F$. The seasonal lag is 24 for hourly and 7 for daily data. Empty-pair or zero scales make MASE undefined, without epsilon substitution. MAPE uses $|Y-\widehat Y|/|Y|$ on observed, exactly nonzero targets, including negative targets with an absolute denominator. Near-zero targets are not filtered. \textbf{All MAPE means and SDs in this paper are decimal ratios, not percentages.}

Within each item/origin, average valid lead losses. Then macro-average eligible items at that origin, average origins, and average seed losses. An all-zero item/horizon can have defined MASE but undefined MAPE. Each metric uses its own predetermined support, shared across models. Repeated target timestamps at different leads represent different forecast tasks and remain in the evaluation. Forecasts from different seeds are not averaged before scoring. No pooled cross-dataset mean is used for the main claims.

\subsection{Window and seed variability}
Let $\ell_{s,w}$ be the item-macro loss for seed $s$ at origin $w$. The reported mean is $\operatorname{mean}_w\operatorname{mean}_s\ell_{s,w}$. We separately report the sample SD across origins of seed-averaged losses and the sample SD across seeds of origin-averaged losses:
\[
\mathrm{SD}_{\rm window}=\operatorname{sd}_{w}\!\left(\operatorname{mean}_{s}\ell_{s,w}\right),
\qquad
\mathrm{SD}_{\rm seed}=\operatorname{sd}_{s}\!\left(\operatorname{mean}_{w}\ell_{s,w}\right).
\]
Both SDs use one degree-of-freedom correction. Tables~\ref{tab:exp-window-sd} and~\ref{tab:exp-seed-sd} report these distinct quantities. Window SD measures changing forecast difficulty, not a standard error or confidence interval. A single deterministic run has undefined seed SD, not zero. Training seeds do not measure uncertainty over new datasets or synthetic realizations.

\begin{table}[htbp]
\centering
\caption{Sample SD across rolling origins after averaging seed losses. MAPE SD is in decimal-ratio units. These are descriptive SDs, not standard errors. Dashes denote unavailable additional-baseline cells only.}
\label{tab:exp-window-sd}
\fontsize{7}{8.5}\selectfont
\setlength{\tabcolsep}{1.7pt}\renewcommand{\arraystretch}{1.08}
\resizebox{\linewidth}{!}{%
\begin{tabular}{lrrrrrrrrrrrr}
\toprule
Model & \multicolumn{2}{c}{EPF} & \multicolumn{2}{c}{Sales1} & \multicolumn{2}{c}{Sales2} & \multicolumn{2}{c}{Syn. Electricity} & \multicolumn{2}{c}{Rossmann-100} & \multicolumn{2}{c}{Favorita-198} \\
\cmidrule(lr){2-3}\cmidrule(lr){4-5}\cmidrule(lr){6-7}\cmidrule(lr){8-9}\cmidrule(lr){10-11}\cmidrule(lr){12-13}
 & MASE & MAPE & MASE & MAPE & MASE & MAPE & MASE & MAPE & MASE & MAPE & MASE & MAPE \\
\midrule
HopFormer & 0.411 & 1.621 & 0.024 & 0.279 & 0.033 & 0.262 & 0.070 & 0.040 & 0.207 & 0.050 & 0.476 & 0.038 \\
Residual backbone frozen & 0.401 & 1.600 & 0.056 & 0.298 & 0.068 & 0.298 & 0.077 & 0.045 & 0.217 & 0.040 & 0.430 & 0.027 \\
Direct backbone + LoRA & 0.409 & 1.978 & 0.022 & 0.156 & 0.005 & 0.067 & 0.559 & 0.051 & 0.233 & 0.050 & 0.441 & 0.024 \\
Direct backbone frozen & 0.421 & 2.130 & 0.065 & 0.428 & 0.043 & 0.203 & 0.709 & 0.054 & 0.260 & 0.053 & 0.420 & 0.023 \\
\midrule
SPA trend only & 0.631 & 1.894 & 0.102 & 0.389 & 0.096 & 0.374 & 0.204 & 0.111 & 0.150 & 0.027 & 0.440 & 0.033 \\
Equal-weight trend & 0.691 & 2.101 & 0.132 & 0.460 & 0.059 & 0.363 & 0.301 & 0.344 & 0.146 & 0.020 & 0.419 & 0.078 \\
Best-expert trend & 0.651 & 2.051 & 0.095 & 0.417 & 0.074 & 0.443 & 0.364 & 0.032 & 0.125 & 0.024 & 0.559 & 0.043 \\
ElasticNet trend & 0.627 & 1.840 & 0.088 & 0.431 & 0.206 & 0.514 & 0.223 & 0.123 & 0.167 & 0.040 & 0.564 & 0.052 \\
\midrule
TimeXer & 0.423 & 1.885 & 0.018 & 0.277 & 0.006 & 0.201 & 0.915 & 0.241 & 0.197 & 0.040 & 0.430 & 0.047 \\
TimeMixer & 0.456 & 1.895 & 0.021 & 0.230 & 0.006 & 0.152 & 0.234 & 0.035 & 0.250 & 0.058 & 0.446 & 0.067 \\
TFT & 0.541 & 1.470 & 0.035 & 0.213 & 0.023 & 0.124 & 0.466 & 0.057 & 0.220 & 0.047 & 0.492 & 0.026 \\
TiDE & 0.502 & 1.908 & 0.020 & 0.235 & 0.014 & 0.207 & 1.272 & 0.391 & 0.158 & 0.051 & 0.456 & 0.041 \\
TSMixerX & 0.440 & 2.045 & 0.018 & 0.079 & 0.005 & 0.133 & 0.234 & 0.040 & 0.231 & 0.059 & 0.459 & 0.025 \\
NBEATSx & 0.469 & 1.891 & 0.017 & 0.159 & 0.007 & 0.152 & 0.431 & 0.050 & 0.344 & 0.093 & 0.455 & 0.042 \\
NHITS & 0.471 & 1.502 & 0.027 & 0.155 & 0.013 & 0.145 & 0.250 & 0.040 & 0.166 & 0.047 & 0.512 & 0.032 \\
PatchTST & 0.451 & 1.995 & 0.021 & 0.245 & 0.004 & 0.141 & 0.288 & 0.062 & 0.236 & 0.048 & 0.442 & 0.063 \\
iTransformer & 0.443 & 2.176 & 0.022 & 0.190 & 0.007 & 0.200 & 0.297 & 0.044 & 0.239 & 0.047 & 0.435 & 0.058 \\
DLinear & 0.528 & 2.587 & 0.020 & 0.232 & 0.005 & 0.233 & 0.244 & 0.036 & 0.175 & 0.030 & 0.440 & 0.074 \\
\midrule
Chronos-2 & 0.351 & 1.331 & 0.022 & 0.135 & 0.012 & 0.126 & 0.277 & 0.067 & -- & -- & -- & -- \\
Chronos-Bolt Base & 0.404 & 2.276 & 0.072 & 0.334 & 0.037 & 0.160 & 0.400 & 0.056 & 0.274 & 0.055 & 0.429 & 0.026 \\
MOIRAI 1.1-R Base & 0.452 & 2.439 & -- & -- & -- & -- & -- & -- & -- & -- & -- & -- \\
TimesFM 2.5 XReg & 0.447 & 2.198 & -- & -- & -- & -- & -- & -- & -- & -- & -- & -- \\
\midrule
LightGBM & 0.464 & 1.515 & -- & -- & -- & -- & -- & -- & -- & -- & -- & -- \\
SeasonalNaive & 0.581 & 2.835 & -- & -- & -- & -- & -- & -- & -- & -- & -- & -- \\
\bottomrule
\end{tabular}}
\end{table}

\begin{table}[htbp]
\centering
\caption{Sample SD across seed-specific mean errors. MAPE SD is in decimal-ratio units. n/a denotes a deterministic single run, not zero SD; dashes denote unavailable additional-baseline cells. Values rounded to 0.000 need not be exactly zero.}
\label{tab:exp-seed-sd}
\fontsize{7}{8.5}\selectfont
\setlength{\tabcolsep}{1.7pt}\renewcommand{\arraystretch}{1.08}
\resizebox{\linewidth}{!}{%
\begin{tabular}{lrrrrrrrrrrrr}
\toprule
Model & \multicolumn{2}{c}{EPF} & \multicolumn{2}{c}{Sales1} & \multicolumn{2}{c}{Sales2} & \multicolumn{2}{c}{Syn. Electricity} & \multicolumn{2}{c}{Rossmann-100} & \multicolumn{2}{c}{Favorita-198} \\
\cmidrule(lr){2-3}\cmidrule(lr){4-5}\cmidrule(lr){6-7}\cmidrule(lr){8-9}\cmidrule(lr){10-11}\cmidrule(lr){12-13}
 & MASE & MAPE & MASE & MAPE & MASE & MAPE & MASE & MAPE & MASE & MAPE & MASE & MAPE \\
\midrule
HopFormer & 0.001 & 0.005 & 0.004 & 0.015 & 0.049 & 0.029 & 0.013 & 0.007 & 0.016 & 0.005 & 0.063 & 0.018 \\
Residual backbone frozen & 0.000 & 0.003 & 0.000 & 0.007 & 0.021 & 0.041 & 0.009 & 0.006 & 0.004 & 0.002 & 0.059 & 0.018 \\
Direct backbone + LoRA & 0.001 & 0.002 & 0.001 & 0.004 & 0.000 & 0.004 & 0.048 & 0.008 & 0.019 & 0.006 & 0.006 & 0.011 \\
Direct backbone frozen & n/a & n/a & n/a & n/a & n/a & n/a & n/a & n/a & n/a & n/a & n/a & n/a \\
\midrule
SPA trend only & 0.019 & 0.015 & 0.004 & 0.022 & 0.001 & 0.056 & 0.105 & 0.019 & 0.005 & 0.001 & 0.167 & 0.122 \\
Equal-weight trend & 0.013 & 0.029 & 0.028 & 0.098 & 0.007 & 0.066 & 0.030 & 0.008 & 0.008 & 0.002 & 0.027 & 0.028 \\
Best-expert trend & 0.047 & 0.079 & 0.000 & 0.000 & 0.000 & 0.000 & 0.004 & 0.000 & 0.001 & 0.000 & 0.007 & 0.003 \\
ElasticNet trend & 0.003 & 0.001 & 0.003 & 0.007 & 0.114 & 0.302 & 0.098 & 0.025 & 0.017 & 0.009 & 0.009 & 0.001 \\
\midrule
TimeXer & 0.010 & 0.009 & 0.000 & 0.004 & 0.001 & 0.010 & 0.094 & 0.037 & 0.027 & 0.009 & 0.030 & 0.013 \\
TimeMixer & 0.013 & 0.022 & 0.002 & 0.009 & 0.001 & 0.003 & 0.046 & 0.015 & 0.008 & 0.003 & 0.015 & 0.012 \\
TFT & 0.055 & 0.009 & 0.014 & 0.034 & 0.003 & 0.008 & 0.166 & 0.092 & 0.059 & 0.014 & 0.057 & 0.010 \\
TiDE & 0.017 & 0.011 & 0.090 & 0.298 & 0.018 & 0.022 & 3.261 & 0.649 & 0.167 & 0.054 & 0.002 & 0.011 \\
TSMixerX & 0.007 & 0.029 & 0.005 & 0.002 & 0.002 & 0.007 & 0.071 & 0.021 & 0.255 & 0.087 & 0.025 & 0.004 \\
NBEATSx & 0.005 & 0.022 & 0.010 & 0.016 & 0.034 & 0.015 & 0.328 & 0.047 & 0.708 & 0.227 & 0.032 & 0.004 \\
NHITS & 0.055 & 0.018 & 0.052 & 0.010 & 0.003 & 0.013 & 0.074 & 0.018 & 0.020 & 0.009 & 0.028 & 0.007 \\
PatchTST & 0.011 & 0.005 & 0.002 & 0.002 & 0.000 & 0.006 & 0.027 & 0.015 & 0.004 & 0.001 & 0.008 & 0.003 \\
iTransformer & 0.010 & 0.032 & 0.005 & 0.015 & 0.002 & 0.012 & 0.090 & 0.049 & 0.006 & 0.002 & 0.011 & 0.001 \\
DLinear & 0.008 & 0.008 & 0.000 & 0.004 & 0.000 & 0.001 & 0.036 & 0.006 & 0.002 & 0.001 & 0.003 & 0.001 \\
\midrule
Chronos-2 & n/a & n/a & n/a & n/a & n/a & n/a & n/a & n/a & -- & -- & -- & -- \\
Chronos-Bolt Base & n/a & n/a & n/a & n/a & n/a & n/a & n/a & n/a & n/a & n/a & n/a & n/a \\
MOIRAI 1.1-R Base & 0.001 & 0.001 & -- & -- & -- & -- & -- & -- & -- & -- & -- & -- \\
TimesFM 2.5 XReg & n/a & n/a & -- & -- & -- & -- & -- & -- & -- & -- & -- & -- \\
\midrule
LightGBM & 0.002 & 0.008 & -- & -- & -- & -- & -- & -- & -- & -- & -- & -- \\
SeasonalNaive & n/a & n/a & -- & -- & -- & -- & -- & -- & -- & -- & -- & -- \\
\bottomrule
\end{tabular}}
\end{table}

\subsection{Paired time-block inference}\label{app:exp-inference}
We form paired origin losses after averaging seeds, keeping the entire cross-section together during resampling. For left and right methods, $d_w=\operatorname{mean}_s\ell^{\rm left}_{s,w}-\operatorname{mean}_s\ell^{\rm right}_{s,w}$; negative differences favor the left method. Seeds and series are not treated as independent temporal replicates. The circular moving-block bootstrap samples contiguous origin blocks, wraps at the period boundary, and truncates the concatenated sample to the original number of origins. It uses 10,000 draws and RNG seed 2026. Table~\ref{tab:inference-blocks} gives the fixed block choices.

\begin{table}[htbp]
\centering\small
\caption{Time-block inference settings in origins. Blocks must cover at least $\lceil H/\Delta\rceil$ origins; the numerical guard also requires $K\geq\max(20,4b)$. These guards do not guarantee inferential power or stationarity.}
\label{tab:inference-blocks}
\begin{tabular}{lrrrrr}\toprule
Dataset & $K$ & $\lceil H/\Delta\rceil$ & Primary & Sensitivity 1 & Sensitivity 2 \\\midrule
EPF & 366 & 1 & 7 & 14 & 28 \\
Sales1 / Sales2 & 49 & 4 & 8 & 4 & 12 \\
Synthetic Electricity & 365 & 1 & 7 & 14 & 28 \\
Rossmann-100 & 46 & 7 & 8 & 7 & 12 (not estimable) \\
Favorita-198 & 50 & 3 & 8 & 3 & 12 \\
\bottomrule\end{tabular}
\end{table}

The primary-block 95\% interval is the percentile interval of resampled mean differences. These intervals are pointwise, not simultaneous familywise intervals. With observed difference $\bar d$ and resampled means $\bar d^{*(r)}$, the two-sided centered-bootstrap approximation is
\[
p=\frac{1+\sum_{r=1}^{10000}\mathrm{I}\{|\bar d^{*(r)}-\bar d|\geq|\bar d|\}}{10001},
\]
where $\mathrm{I}$ is the indicator function. Its minimum is $1/10001$, never zero. Constant degenerate differences receive no bootstrap inference. An interval excluding zero need not imply significance after multiplicity correction.

\paragraph{Fixed families and sensitivity.}
External comparisons retain 252 slots: 18 planned baselines, seven datasets including M5, and two metrics. The baseline list is TimeXer, TimeMixer, TFT, TiDE, TSMixerX, NBEATSx, NHITS, PatchTST, iTransformer, TimesNet, DLinear, Chronos-2, Chronos-Bolt Base, MOIRAI, TimesFM XReg, LightGBM, SeasonalNaive, and AutoARIMA. The evaluated retail subsets occupy their original dataset slots. M5 and other unavailable evaluations remain NA, with a value of one used only inside Holm adjustment; this is not an observed null result. The completed analysis contains 148 available primary-block tests, including the retained additional baselines. Restricting the displayed main table does not shrink the correction family.

The separate 72-slot ablation family consists of six datasets, two metrics, and six contrasts: SPA-only versus equal-weight, best-expert, and nonnegative-ElasticNet trends; and HopFormer versus a frozen residual backbone, a direct frozen backbone, and a directly LoRA-adapted backbone. The earlier four-dataset family was extended to six datasets before the retail contrasts were computed; all current contrasts use the resulting 72-slot family. All 72 primary-block comparisons are available.

For each family we apply Holm correction separately across the entire fixed family in the primary and two declared sensitivity scenarios. Let $p_0$ be the primary adjusted value and $p_{\max}$ the maximum of the three adjusted values. The latter is a conservative summary of the sensitivity rule, not a newly calibrated combined $p$-value. Stable evidence requires all three tests to be estimable and $p_{\max}<0.05$, with direction retained. Rossmann fails the longest-block guard ($46<48$), so its $p_{\max}$ is NA even when $p_0<0.05$. NA is not evidence of equivalence. Tables~\ref{tab:primary-p-mase}--\ref{tab:ablation-p-mape} report the numerical adjustments, including unfavorable findings.

\begin{table}[htbp]
\centering
\caption{External comparisons (MASE): left is HopFormer, right is the named baseline. Each cell gives $p_0$ above $p_{\max}$. $\downarrow$ denotes stable evidence for lower loss on the left; $\uparrow$ favors the right. An unmarked value does not establish equivalence. Rossmann has no $p_{\max}$; -- means no available test. Values use the complete fixed family, not just the displayed rows.}
\label{tab:primary-p-mase}
\fontsize{7}{8.5}\selectfont
\setlength{\tabcolsep}{3pt}\renewcommand{\arraystretch}{1.08}
\begin{tabular}{lcccccc}
\toprule
Comparison & EPF & Sales1 & Sales2 & Syn. Elec. & Rossmann & Favorita \\
\midrule
TimeXer & \shortstack{$0.0252$\\$0.0684$} & \shortstack{$0.0252$\\$0.0252^{\uparrow}$} & \shortstack{$0.0252$\\$0.0252^{\uparrow}$} & \shortstack{$0.0252$\\$0.0840$} & \shortstack{$1.0000$\\$\mathrm{NA}$} & \shortstack{$0.1580$\\$0.2370$} \\[2pt]
TimeMixer & \shortstack{$0.0252$\\$0.0684$} & \shortstack{$0.0252$\\$0.0252^{\uparrow}$} & \shortstack{$0.0252$\\$0.0252^{\uparrow}$} & \shortstack{$0.0252$\\$0.0252^{\downarrow}$} & \shortstack{$1.0000$\\$\mathrm{NA}$} & \shortstack{$0.0252$\\$0.0252^{\downarrow}$} \\[2pt]
TFT & \shortstack{$0.0252$\\$0.0252^{\downarrow}$} & \shortstack{$0.0252$\\$0.1155$} & \shortstack{$0.0252$\\$0.0252^{\uparrow}$} & \shortstack{$0.0252$\\$0.0252^{\downarrow}$} & \shortstack{$1.0000$\\$\mathrm{NA}$} & \shortstack{$0.0252$\\$0.0252^{\downarrow}$} \\[2pt]
TiDE & \shortstack{$0.0252$\\$0.0352^{\downarrow}$} & \shortstack{$0.0252$\\$0.0252^{\downarrow}$} & \shortstack{$0.0252$\\$0.0252^{\uparrow}$} & \shortstack{$0.0252$\\$0.0252^{\downarrow}$} & \shortstack{$1.0000$\\$\mathrm{NA}$} & \shortstack{$0.0252$\\$0.0252^{\downarrow}$} \\[2pt]
TSMixerX & \shortstack{$0.0495$\\$0.4805$} & \shortstack{$0.0252$\\$0.0252^{\uparrow}$} & \shortstack{$0.0252$\\$0.0252^{\uparrow}$} & \shortstack{$0.0252$\\$0.0252^{\downarrow}$} & \shortstack{$0.0252$\\$\mathrm{NA}$} & \shortstack{$0.0960$\\$0.2790$} \\[2pt]
NBEATSx & \shortstack{$0.0252$\\$0.1620$} & \shortstack{$0.0252$\\$0.0252^{\uparrow}$} & \shortstack{$0.0252$\\$0.0252^{\uparrow}$} & \shortstack{$0.0252$\\$0.0252^{\downarrow}$} & \shortstack{$0.3565$\\$\mathrm{NA}$} & \shortstack{$0.0252$\\$0.0252^{\downarrow}$} \\[2pt]
NHITS & \shortstack{$0.0252$\\$0.0252^{\downarrow}$} & \shortstack{$0.0252$\\$0.0252^{\downarrow}$} & \shortstack{$0.0252$\\$0.0252^{\uparrow}$} & \shortstack{$0.0252$\\$0.0252^{\downarrow}$} & \shortstack{$1.0000$\\$\mathrm{NA}$} & \shortstack{$0.0252$\\$0.0252^{\downarrow}$} \\[2pt]
PatchTST & \shortstack{$0.0252$\\$0.1467$} & \shortstack{$0.0252$\\$0.0252^{\uparrow}$} & \shortstack{$0.0252$\\$0.0252^{\uparrow}$} & \shortstack{$0.0252$\\$0.0252^{\downarrow}$} & \shortstack{$1.0000$\\$\mathrm{NA}$} & \shortstack{$0.0252$\\$0.0252^{\downarrow}$} \\[2pt]
iTransformer & \shortstack{$0.0252$\\$0.0252^{\downarrow}$} & \shortstack{$0.0332$\\$0.1620$} & \shortstack{$0.0252$\\$0.0252^{\uparrow}$} & \shortstack{$0.0252$\\$0.0252^{\downarrow}$} & \shortstack{$1.0000$\\$\mathrm{NA}$} & \shortstack{$0.0252$\\$0.0252^{\downarrow}$} \\[2pt]
DLinear & \shortstack{$0.0252$\\$0.0252^{\downarrow}$} & \shortstack{$0.0810$\\$0.1920$} & \shortstack{$0.0252$\\$0.0252^{\uparrow}$} & \shortstack{$0.0252$\\$0.0252^{\downarrow}$} & \shortstack{$1.0000$\\$\mathrm{NA}$} & \shortstack{$0.0252$\\$0.0252^{\downarrow}$} \\[2pt]
\midrule
Chronos-2 & \shortstack{$0.0252$\\$0.0252^{\uparrow}$} & \shortstack{$0.0252$\\$0.0252^{\uparrow}$} & \shortstack{$0.0252$\\$0.0252^{\uparrow}$} & \shortstack{$0.0252$\\$1.0000$} & -- & -- \\[2pt]
Chronos-Bolt Base & \shortstack{$1.0000$\\$1.0000$} & \shortstack{$0.0252$\\$0.0252^{\downarrow}$} & \shortstack{$0.0252$\\$0.0252^{\downarrow}$} & \shortstack{$0.0252$\\$0.0519$} & \shortstack{$1.0000$\\$\mathrm{NA}$} & \shortstack{$1.0000$\\$1.0000$} \\[2pt]
MOIRAI 1.1-R Base & \shortstack{$0.0252$\\$0.0252^{\downarrow}$} & -- & -- & -- & -- & -- \\[2pt]
TimesFM 2.5 XReg & \shortstack{$0.0252$\\$0.3160$} & -- & -- & -- & -- & -- \\[2pt]
LightGBM & \shortstack{$0.0252$\\$0.0252^{\downarrow}$} & -- & -- & -- & -- & -- \\[2pt]
SeasonalNaive & \shortstack{$0.0252$\\$0.0252^{\downarrow}$} & -- & -- & -- & -- & -- \\[2pt]
\bottomrule
\end{tabular}
\end{table}

\begin{table}[htbp]
\centering
\caption{External comparisons (decimal MAPE): left is HopFormer, right is the named baseline. Each cell gives $p_0$ above $p_{\max}$. $\downarrow$ denotes stable evidence for lower loss on the left; $\uparrow$ favors the right. An unmarked value does not establish equivalence. Rossmann has no $p_{\max}$; -- means no available test. Values use the complete fixed family, not just the displayed rows.}
\label{tab:primary-p-mape}
\fontsize{7}{8.5}\selectfont
\setlength{\tabcolsep}{3pt}\renewcommand{\arraystretch}{1.08}
\begin{tabular}{lcccccc}
\toprule
Comparison & EPF & Sales1 & Sales2 & Syn. Elec. & Rossmann & Favorita \\
\midrule
TimeXer & \shortstack{$0.1272$\\$0.1760$} & \shortstack{$1.0000$\\$1.0000$} & \shortstack{$0.0252$\\$0.0252^{\uparrow}$} & \shortstack{$0.0252$\\$0.0840$} & \shortstack{$1.0000$\\$\mathrm{NA}$} & \shortstack{$0.0252$\\$0.0252^{\downarrow}$} \\[2pt]
TimeMixer & \shortstack{$0.5512$\\$0.5512$} & \shortstack{$1.0000$\\$1.0000$} & \shortstack{$0.0252$\\$0.0252^{\uparrow}$} & \shortstack{$0.0252$\\$0.0252^{\downarrow}$} & \shortstack{$1.0000$\\$\mathrm{NA}$} & \shortstack{$0.0252$\\$0.0252^{\downarrow}$} \\[2pt]
TFT & \shortstack{$1.0000$\\$1.0000$} & \shortstack{$1.0000$\\$1.0000$} & \shortstack{$0.0252$\\$0.0684$} & \shortstack{$0.0252$\\$0.0252^{\downarrow}$} & \shortstack{$1.0000$\\$\mathrm{NA}$} & \shortstack{$1.0000$\\$1.0000$} \\[2pt]
TiDE & \shortstack{$0.3120$\\$0.5031$} & \shortstack{$1.0000$\\$1.0000$} & \shortstack{$1.0000$\\$1.0000$} & \shortstack{$0.0252$\\$0.0252^{\downarrow}$} & \shortstack{$0.5771$\\$\mathrm{NA}$} & \shortstack{$0.0252$\\$0.0252^{\downarrow}$} \\[2pt]
TSMixerX & \shortstack{$1.0000$\\$1.0000$} & \shortstack{$0.0252$\\$0.0252^{\uparrow}$} & \shortstack{$0.0810$\\$0.1620$} & \shortstack{$0.0252$\\$0.0252^{\downarrow}$} & \shortstack{$0.0252$\\$\mathrm{NA}$} & \shortstack{$1.0000$\\$1.0000$} \\[2pt]
NBEATSx & \shortstack{$0.4284$\\$0.4284$} & \shortstack{$0.0252$\\$0.0346^{\uparrow}$} & \shortstack{$0.0252$\\$0.0252^{\uparrow}$} & \shortstack{$0.0252$\\$0.0252^{\downarrow}$} & \shortstack{$0.3565$\\$\mathrm{NA}$} & \shortstack{$1.0000$\\$1.0000$} \\[2pt]
NHITS & \shortstack{$1.0000$\\$1.0000$} & \shortstack{$0.0495$\\$0.1312$} & \shortstack{$0.0252$\\$0.0845$} & \shortstack{$0.0252$\\$0.0252^{\downarrow}$} & \shortstack{$1.0000$\\$\mathrm{NA}$} & \shortstack{$0.5015$\\$0.8849$} \\[2pt]
PatchTST & \shortstack{$0.5015$\\$0.5015$} & \shortstack{$1.0000$\\$1.0000$} & \shortstack{$0.0252$\\$0.0252^{\uparrow}$} & \shortstack{$0.0252$\\$0.0252^{\downarrow}$} & \shortstack{$1.0000$\\$\mathrm{NA}$} & \shortstack{$0.0252$\\$0.0252^{\downarrow}$} \\[2pt]
iTransformer & \shortstack{$0.7445$\\$0.7445$} & \shortstack{$1.0000$\\$1.0000$} & \shortstack{$0.0252$\\$0.0252^{\uparrow}$} & \shortstack{$0.0252$\\$0.0252^{\downarrow}$} & \shortstack{$1.0000$\\$\mathrm{NA}$} & \shortstack{$0.0252$\\$0.0252^{\downarrow}$} \\[2pt]
DLinear & \shortstack{$0.0652$\\$0.0996$} & \shortstack{$1.0000$\\$1.0000$} & \shortstack{$0.0252$\\$0.0252^{\uparrow}$} & \shortstack{$0.0252$\\$0.0252^{\downarrow}$} & \shortstack{$1.0000$\\$\mathrm{NA}$} & \shortstack{$0.0252$\\$0.0252^{\downarrow}$} \\[2pt]
\midrule
Chronos-2 & \shortstack{$1.0000$\\$1.0000$} & \shortstack{$0.0252$\\$0.0252^{\uparrow}$} & \shortstack{$0.6614$\\$0.6614$} & \shortstack{$1.0000$\\$1.0000$} & -- & -- \\[2pt]
Chronos-Bolt Base & \shortstack{$1.0000$\\$1.0000$} & \shortstack{$1.0000$\\$1.0000$} & \shortstack{$1.0000$\\$1.0000$} & \shortstack{$0.0252$\\$1.0000$} & \shortstack{$1.0000$\\$\mathrm{NA}$} & \shortstack{$1.0000$\\$1.0000$} \\[2pt]
MOIRAI 1.1-R Base & \shortstack{$0.1727$\\$0.1727$} & -- & -- & -- & -- & -- \\[2pt]
TimesFM 2.5 XReg & \shortstack{$1.0000$\\$1.0000$} & -- & -- & -- & -- & -- \\[2pt]
LightGBM & \shortstack{$1.0000$\\$1.0000$} & -- & -- & -- & -- & -- \\[2pt]
SeasonalNaive & \shortstack{$0.5099$\\$0.5291$} & -- & -- & -- & -- & -- \\[2pt]
\bottomrule
\end{tabular}
\end{table}

\begin{table}[htbp]
\centering
\caption{Ablation comparisons (MASE): left/right follow the row label. Each cell gives $p_0$ above $p_{\max}$. $\downarrow$ denotes stable evidence for lower loss on the left; $\uparrow$ favors the right. An unmarked value does not establish equivalence. Rossmann has no $p_{\max}$; -- means no available test. Values use the complete fixed family, not just the displayed rows.}
\label{tab:ablation-p-mase}
\fontsize{7}{8.5}\selectfont
\setlength{\tabcolsep}{3pt}\renewcommand{\arraystretch}{1.08}
\begin{tabular}{lcccccc}
\toprule
Comparison & EPF & Sales1 & Sales2 & Syn. Elec. & Rossmann & Favorita \\
\midrule
SPA vs equal & \shortstack{$0.6399$\\$1.0000$} & \shortstack{$0.0072$\\$0.0072^{\downarrow}$} & \shortstack{$1.0000$\\$1.0000$} & \shortstack{$0.0072$\\$0.0072^{\downarrow}$} & \shortstack{$0.0072$\\$\mathrm{NA}$} & \shortstack{$0.0072$\\$0.0072^{\uparrow}$} \\[2pt]
SPA vs best & \shortstack{$0.0532$\\$1.0000$} & \shortstack{$0.0072$\\$0.0072^{\uparrow}$} & \shortstack{$1.0000$\\$1.0000$} & \shortstack{$0.2968$\\$1.0000$} & \shortstack{$0.0072$\\$\mathrm{NA}$} & \shortstack{$0.0351$\\$0.0782$} \\[2pt]
SPA vs ElasticNet & \shortstack{$0.0092$\\$0.3145$} & \shortstack{$0.0072$\\$0.0072^{\downarrow}$} & \shortstack{$0.0072$\\$0.0072^{\downarrow}$} & \shortstack{$0.3792$\\$1.0000$} & \shortstack{$0.0168$\\$\mathrm{NA}$} & \shortstack{$0.0980$\\$0.2000$} \\[2pt]
HopFormer vs frozen residual & \shortstack{$0.3792$\\$0.3792$} & \shortstack{$0.0072$\\$0.0072^{\downarrow}$} & \shortstack{$0.0072$\\$0.0072^{\downarrow}$} & \shortstack{$0.0072$\\$0.0072^{\downarrow}$} & \shortstack{$0.0168$\\$\mathrm{NA}$} & \shortstack{$0.2666$\\$0.3996$} \\[2pt]
HopFormer vs direct frozen & \shortstack{$0.2968$\\$0.2968$} & \shortstack{$0.0072$\\$0.0072^{\downarrow}$} & \shortstack{$0.0072$\\$0.0072^{\downarrow}$} & \shortstack{$0.0072$\\$0.0150^{\downarrow}$} & \shortstack{$0.2666$\\$\mathrm{NA}$} & \shortstack{$1.0000$\\$1.0000$} \\[2pt]
HopFormer vs direct LoRA & \shortstack{$1.0000$\\$1.0000$} & \shortstack{$0.0072$\\$0.0072^{\uparrow}$} & \shortstack{$0.0072$\\$0.0072^{\uparrow}$} & \shortstack{$0.0072$\\$0.0329^{\downarrow}$} & \shortstack{$1.0000$\\$\mathrm{NA}$} & \shortstack{$0.0132$\\$0.0132^{\uparrow}$} \\[2pt]
\bottomrule
\end{tabular}
\end{table}

\begin{table}[htbp]
\centering
\caption{Ablation comparisons (decimal MAPE): left/right follow the row label. Each cell gives $p_0$ above $p_{\max}$. $\downarrow$ denotes stable evidence for lower loss on the left; $\uparrow$ favors the right. An unmarked value does not establish equivalence. Rossmann has no $p_{\max}$; -- means no available test. Values use the complete fixed family, not just the displayed rows.}
\label{tab:ablation-p-mape}
\fontsize{7}{8.5}\selectfont
\setlength{\tabcolsep}{3pt}\renewcommand{\arraystretch}{1.08}
\begin{tabular}{lcccccc}
\toprule
Comparison & EPF & Sales1 & Sales2 & Syn. Elec. & Rossmann & Favorita \\
\midrule
SPA vs equal & \shortstack{$1.0000$\\$1.0000$} & \shortstack{$1.0000$\\$1.0000$} & \shortstack{$0.8279$\\$1.0000$} & \shortstack{$0.0072$\\$0.0072^{\downarrow}$} & \shortstack{$1.0000$\\$\mathrm{NA}$} & \shortstack{$1.0000$\\$1.0000$} \\[2pt]
SPA vs best & \shortstack{$0.0720$\\$0.2496$} & \shortstack{$0.8279$\\$1.0000$} & \shortstack{$0.3302$\\$0.3570$} & \shortstack{$0.0240$\\$1.0000$} & \shortstack{$0.0072$\\$\mathrm{NA}$} & \shortstack{$0.0072$\\$0.0072^{\uparrow}$} \\[2pt]
SPA vs ElasticNet & \shortstack{$1.0000$\\$1.0000$} & \shortstack{$0.1696$\\$0.3666$} & \shortstack{$0.0092$\\$0.0294^{\downarrow}$} & \shortstack{$1.0000$\\$1.0000$} & \shortstack{$0.1485$\\$\mathrm{NA}$} & \shortstack{$0.0072$\\$0.0072^{\uparrow}$} \\[2pt]
HopFormer vs frozen residual & \shortstack{$1.0000$\\$1.0000$} & \shortstack{$0.1190$\\$0.1452$} & \shortstack{$0.0072$\\$0.0072^{\downarrow}$} & \shortstack{$0.0072$\\$0.0106^{\downarrow}$} & \shortstack{$0.2666$\\$\mathrm{NA}$} & \shortstack{$0.0072$\\$0.0072^{\downarrow}$} \\[2pt]
HopFormer vs direct frozen & \shortstack{$0.3302$\\$0.3302$} & \shortstack{$0.7979$\\$1.0000$} & \shortstack{$0.4893$\\$0.9995$} & \shortstack{$0.0072$\\$0.1505$} & \shortstack{$1.0000$\\$\mathrm{NA}$} & \shortstack{$1.0000$\\$1.0000$} \\[2pt]
HopFormer vs direct LoRA & \shortstack{$0.8279$\\$0.8279$} & \shortstack{$0.0532$\\$0.0532$} & \shortstack{$0.0072$\\$0.0088^{\uparrow}$} & \shortstack{$0.0132$\\$1.0000$} & \shortstack{$0.4444$\\$\mathrm{NA}$} & \shortstack{$0.8279$\\$1.0000$} \\[2pt]
\bottomrule
\end{tabular}
\end{table}

\paragraph{Summary and matched effects.}
Table~\ref{tab:inference-counts} counts stable external-baseline contrasts, not independent datasets, effect sizes, or a pooled success rate. Table~\ref{tab:matched-intervals} gives all six MASE contrasts against direct adaptation of the same Small backbone; favorable and unfavorable differences are shown with identical conventions. MAPE effects can be read from Table~\ref{tab:exp-ablation-results} and their adjusted tests from Table~\ref{tab:ablation-p-mape}.

\begin{table}[htbp]
\centering
\caption{Stable evidence against the ten active neural baselines, with two metric comparisons per baseline. Neither includes insufficient sensitivity support; it does not mean equivalence.}
\label{tab:inference-counts}
\small
\setlength{\tabcolsep}{2.5pt}\renewcommand{\arraystretch}{1.08}
\begin{tabular}{lrrrr}
\toprule
Dataset & Available & HopFormer lower & Baseline lower & Neither \\
\midrule
EPF & 20/20 & 5 & 0 & 15 \\
Sales1 & 20/20 & 2 & 7 & 11 \\
Sales2 & 20/20 & 0 & 16 & 4 \\
Syn. Electricity & 20/20 & 18 & 0 & 2 \\
Rossmann-100 & 20/20 & 0 & 0 & 20 \\
Favorita-198 & 20/20 & 14 & 0 & 6 \\
\bottomrule
\end{tabular}
\end{table}

\begin{table}[htbp]
\centering
\caption{HopFormer minus direct backbone + LoRA: all six matched MASE effects. Negative differences favor HopFormer. Intervals are primary-block, pointwise 95\% intervals; $p_0$ and $p_{\max}$ use the 72-slot Holm family. $\downarrow$ favors HopFormer stably; $\uparrow$ favors direct adaptation. Rossmann has no estimable longest-block sensitivity.}
\label{tab:matched-intervals}
\scriptsize
\setlength{\tabcolsep}{2.5pt}\renewcommand{\arraystretch}{1.08}
\begin{tabular}{lrrrrr}
\toprule
Dataset & Difference & 95\% interval & $p_0$ & $p_{\max}$ & Stable direction \\
\midrule
EPF & -0.004 & $[-0.025,+0.019]$ & 1.0000 & 1.0000 & -- \\
Sales1 & +0.086 & $[+0.073,+0.098]$ & 0.0072 & 0.0072 & $\uparrow$ \\
Sales2 & +0.104 & $[+0.086,+0.121]$ & 0.0072 & 0.0072 & $\uparrow$ \\
Syn. Electricity & -0.433 & $[-0.565,-0.305]$ & 0.0072 & 0.0329 & $\downarrow$ \\
Rossmann-100 & +0.036 & $[-0.011,+0.084]$ & 1.0000 & NA & -- \\
Favorita-198 & +0.061 & $[+0.036,+0.093]$ & 0.0132 & 0.0132 & $\uparrow$ \\
\bottomrule
\end{tabular}
\end{table}

\paragraph{Interpretation limits.}
These are retrospective, conditional comparisons, not a fully blinded confirmatory study. Moving-block inference assumes sufficiently stable, weakly dependent loss differences; block-length sensitivity and minimum-origin guards do not prove those properties. Circular wrapping joins the end and start of the evaluated year. Seasonal changes, few effective blocks, and near-zero prices under EPF MAPE can affect inference. Saved lag-one correlations and first-/second-half effect diagnostics help expose such behavior but do not establish stationarity. For example, the Synthetic Electricity HopFormer--TimeMixer MASE difference changes from $-0.360$ in the first half to $-0.049$ in the second, despite remaining favorable under every declared block setting. The results condition on the fitted seeds and fixed panels; they do not quantify generalization to new synthetic realizations, new stores, or all possible training runs. No $p$-value is inferred from either descriptive SD.

\section{Execution Interface and Reproducibility}\label{sec:code_usage}
The reported implementation uses explicit immutable data bundles and gated runners, not the earlier AutoGluon wrapper example. \texttt{run\_hopformer\_validation.py} exposes \texttt{stage-one}, \texttt{trial}, and \texttt{select}; inputs identify the dataset bundle, seed, variant, context, preset, Stage-I state, native source, and pinned checkpoint. \texttt{summarize\_hopformer\_validation.py} verifies the full search. \texttt{run\_hopformer\_final.py} separates admission, validation replay, and test forecasting; \texttt{run\_hopformer\_ablations.py} applies corresponding gates to controls. Development loaders exclude the test extension. Diagnostic or smoke-test outputs are not final results.

Each admitted cell retains source/dependency/checkpoint hashes, dataset and split metadata, candidate and selected-fit records, stopping curves, seed identities, forecast files, and item/origin metric panels. Reloaded predictions must reproduce validation output before test access; audits check complete forecast keys, exact raw truth, original-unit reconstruction, clipping, and independently recomputed metrics. Verified artifact reuse is recorded explicitly; interrupted fits are not treated as complete. Heterogeneous GPU execution and separate software environments are captured per run rather than presented as identical hardware or equal runtime. The final independent collection verified 122 model/dataset cells and rescored 54,821,040 forecast rows, including the additional historical rows retained above. All 66 active model/dataset cells and 42 separate control cells passed. Raw-forecast keys, targets, and metric-specific support agree across methods. These internal audit artifacts do not by themselves establish public availability of a reproducible release. A technical NBEATSx--Sales2 recovery reran the full unchanged context/learning-rate grid after cross-hardware validation replay failed; all selected seeds subsequently reproduced their validation predictions exactly before test evaluation. Failed artifacts were preserved, tolerances were not relaxed, and recovery was not selected by test performance.

\section{Scope and Limitations}\label{app:exp-scope}
The evidence concerns fixed realizations and seen-series rolling forecasting. The real-world covariate sets are modest and the strict retail schema excludes future realized operations; these experiments do not establish asymptotic scalability in covariate dimension. EPF's forecast-slot semantics differ by market, and its availability assumption is not verified by archived forecast vintages. Synthetic oracle inputs and the Electricity export construction make that task more favorable than operational weather forecasting. Retail subsets and sparse temporal replication limit population-level conclusions.

Model-specific losses, preprocessing, future-input support, and optimization ceilings are disclosed rather than treated as identical. Pretraining overlap cannot be ruled out. Test scores have been inspected during development, so these results are not a fully blinded confirmatory study. The active matrix is complete. Dashes occur only in retained additional-baseline rows and do not imply rankings. The main-table model subset and appendix placement of Sales1/Sales2 were chosen after result inspection; all six datasets and the original external-comparison family remain in the reported analysis. No result here establishes parity with full fine-tuning, superiority across all datasets, unseen-store transfer, or a guaranteed reduction in residual variance. The theoretical estimator and its assumptions remain those of Section~\ref{sec:thm} and Appendix~\ref{app:theory}, not an empirical guarantee for this in-sample pipeline.

\end{document}